\documentclass{article}
\usepackage[utf8]{inputenc} 
\usepackage[T1]{fontenc}    

\usepackage[a4paper, total={6in, 9in}]{geometry}

\usepackage{graphicx}
\usepackage{hyperref}       
\usepackage{url}            
\usepackage{booktabs}       
\usepackage{amsfonts}       
\usepackage{nicefrac}       
\usepackage{microtype}      
\usepackage[round]{natbib}
\usepackage{amsmath,amsthm,amssymb,amsfonts}
\usepackage[dvipsnames,table]{xcolor}         
\usepackage{xspace}
\usepackage{dsfont}
\usepackage{multirow}
\usepackage{makecell}
\usepackage{nicefrac}
\usepackage{enumitem}
\usepackage{fontawesome5}
\definecolor{pearThree}{HTML}{E74C3C}
\definecolor{pearcomp}{HTML}{B97E29}
\definecolor{pearDark}{HTML}{2980B9}
\definecolor{pearDarker}{HTML}{1D2DEC}
\hypersetup{
	colorlinks,
	citecolor=pearDark,
	linkcolor=pearThree,
	breaklinks=true,
	urlcolor=black}
\usepackage[ruled]{algorithm2e}
\usepackage{multicol}

\newtheorem{theorem}{Theorem}
\newtheorem{lemma}[theorem]{Lemma}
\newtheorem{corollary}[theorem]{Corollary}
\newtheorem{claim}[theorem]{Claim}

\newtheorem{definition}[theorem]{Definition}

\newtheorem{assumption}[theorem]{Assumption}

\theoremstyle{definition}
\newtheorem{remark}{Remark}

\usepackage{array}
\newcommand{\PreserveBackslash}[1]{\let\temp=\\#1\let\\=\temp}
\newcolumntype{C}[1]{>{\PreserveBackslash\centering}p{#1}}
\newcolumntype{R}[1]{>{\PreserveBackslash\raggedleft}p{#1}}
\newcolumntype{L}[1]{>{\PreserveBackslash\raggedright}p{#1}}

\usepackage{amsmath,amsfonts,bm}

\def\eqref#1{equation~\ref{#1}}

\def\1{\bm{1}}

\DeclareMathAlphabet{\mathsfit}{\encodingdefault}{\sfdefault}{m}{sl}
\SetMathAlphabet{\mathsfit}{bold}{\encodingdefault}{\sfdefault}{bx}{n}

\DeclareMathOperator*{\argmax}{arg\,max}

\DeclareMathOperator{\Tr}{Tr}

\providecommand{\norm}[1]{\lVert#1\rVert}
\providecommand{\bignorm}[1]{\Big\lVert#1\Big\rVert}
\providecommand{\abs}[1]{\lvert#1\rvert}

\newcommand{\lsviucb}[1]{{\small\textsc{LSVI-UCB}}\xspace}
\newcommand{\ucrlvtr}{{\small\textsc{UCRL-VTR}}\xspace}
\newcommand{\ucrlvtrplus}{{\small\textsc{UCRL-VTR+}}\xspace}
\newcommand{\wt}[1]{\widetilde{#1}}

\newcommand{\wh}[1]{\widehat{#1}}

\newcommand{\transp}{\mathsf{T}}

\newcommand{\comment}[1]{}

\usepackage[colorinlistoftodos, textwidth=30mm, disable]{todonotes}

\usepackage{authblk}

\title{Differentially Private Exploration in\\ Reinforcement Learning with Linear Representation}

\author[1]{Paul Luyo\footnote{Work done while at Facebook AI Research.}}
\affil[1]{Ecole Polytechnique}
\author[2,3]{Evrard Garcelon}
\author[2]{Alessandro Lazaric}
\author[2]{Matteo Pirotta}
\affil[2]{Facebook AI Research}
\affil[3]{CREST, ENSAE}
\date{}

\usepackage[toc,page,header]{appendix}
\usepackage{minitoc}

\begin{document}
\doparttoc 
\faketableofcontents 

\maketitle

\begin{abstract}
This paper studies privacy-preserving exploration in Markov Decision Processes (MDPs) with linear representation. We first consider the setting of linear-mixture MDPs~\citep{ayoub2020vtricml} (a.k.a.\ model-based setting) and provide an unified framework for analyzing joint and local differential private (DP) exploration. Through this framework, we prove a $\widetilde{O}(K^{3/4}/\sqrt{\epsilon})$ regret bound for $(\epsilon,\delta)$-local DP exploration and a $\widetilde{O}(\sqrt{K/\epsilon})$ regret bound for $(\epsilon,\delta)$-joint DP. We further study privacy-preserving exploration in linear MDPs~\citep{jin2020lsviucb} (a.k.a.\ model-free setting) where we provide a $\wt{O}\left(K^{\frac{3}{5}}/\epsilon^{\frac{2}{5}}\right)$ regret bound for $(\epsilon,\delta)$-joint DP, with a novel algorithm based on low-switching. Finally, we provide insights into the issues of designing local DP algorithms in this model-free setting.
\end{abstract}
\vspace{-0.05in}
\section{INTRODUCTION}\label{sec:introduction}
\vspace{-0.05in}

Privacy-preserving techniques are fundamental for deploying machine learning solutions to applications involving sensitive data.
Differential privacy~\citep[][DP]{dwork2014algorithmic} is the standard technique for building algorithms with privacy guarantees on individual data.
While the standard framework applies to a batch setting where the training data is available before hand~\citep[e.g.,][]{erlingsson2014rappor, dwork2014algorithmic, abadi2016deep, abowd2018us}, only recently the literature has focused on studying privacy-preserving techniques for sequential decision-making problems~\citep[e.g.,][]{shariff2018nips,zheng2020locally,vietri2020privaterl,garcelon2020local}.

In this paper, we contribute to the study of DP in online reinforcement learning (RL). For online RL on sensitive data, we assume that a single user enters the system at each episode $k$ and interacts with the system for a fix-horizon $H$. Through the interaction, the user reveals states and rewards that contain sensitive information. In the non-private case the goal of the learning agent is to select actions maximizing the cumulative reward of the users. The agent thus needs to balance exploration to get information about users and exploitation to maximize the reward. In privacy-preserving RL, we additionally aim to provide privacy guarantees on users' data (i.e., states, rewards or any function of these values).
The literature around privacy-preserving exploration in RL has focused on joint and local DP~\citep{vietri2020privaterl,garcelon2020local}. In the JDP setting, users are willing to reveal sensitive data to the agent but they want to prevent third-parties from inferring sensitive information by observing the behavior of the learning agent. On the other hand, in the local model, users are more wary about privacy and additionally require the learning agent to observe only privatized information. 
These works come with provably efficient guarantees, but their algorithms and analysis are limited to tabular settings. This limitation makes it difficult to extract insights about the design and analysis of more practical algorithms.\footnote{While preparing this draft, we noticed a very recent concurrent work that studied \emph{only} JDP in linear-mixture MDPs~\citep{Liao2021locally}. We recover the same bound when $C_w = \sqrt{d}$, see Cor.~\ref{cor:JDP-UCRL-VTR}. Our scope is broader, we study both JDP and LDP in linear-mixture MDP and also JDP in linear MDPs. 
}

In this paper, we take a first step in this direction, by providing theoretical-grounded insights on the design of private RL algorithms using parametric representations. We show that, in this setting, designing efficient and private algorithms is not simple and it poses new challenges w.r.t.\ the tabular setting that require novel solutions. More precisely, we consider (continuous) sequential decision-making problems with linear parametric representations, and provide novel algorithm with provably efficient guarantees for privacy-preserving exploration.
For the linear-mixture representation~\citep{ayoub2020vtricml}, we provide a unified framework which allows to study both joint and local DP. Through this framework, we prove a $\widetilde{O}(K^{3/4}/\sqrt{\epsilon})$ regret bound over $K$ episodes for $(\epsilon,\delta)$-local DP exploration and a $\widetilde{O}(\sqrt{K/\epsilon})$ regret bound for $(\epsilon,\delta)$-joint DP. The key feature enabling these results is the intrinsic model-based nature of this setting, which, similarly to the tabular setting, allows to build a private estimate of the model. In the case of linear model~\citep{jin2020lsviucb}, the model-free nature poses additional challenges compared to the tabular or linear-mixture scenario. Thanks to the use of batching, we provide a novel $(\epsilon,\delta)$-JDP algorithm whose regret bound is $\wt{O}\left(K^{\frac{3}{5}}/\epsilon^{\frac{2}{5}}\right)$. Notably, batching is not used to improve the computational efficiency but it is critical for achieving the improved $K^{\frac{3}{5}}$ rate. Indeed, without batching, we are only able to prove a $\wt{O}(K^{3/4}/\sqrt{\epsilon})$ regret bound, see Sec.~\ref{sec:linear.mdp}. On the other hand, we were not able to derive an LDP exploration result for this setting and we conjecture it may require a substantially different approach. 
As far as we know, these are the first model-free private algorithm for exploration in the literature able to scale beyond the tabular setting. We summarize all the results in Tab.~\ref{tab:summary.results}.

\begin{table*}[t]
    \centering
    \footnotesize
    \begin{tabular}{ccccc}
        \hline
        Algorithm & Setting & JDP & LDP & Regret\\
        \hline
        \citet{garcelon2020local} & Tab. & $\epsilon,0$ & $\epsilon,0$ & $\wt{O}(H^{\frac{3}{2}}S\sqrt{AK} + \frac{H^3S^2A\sqrt{K}}{\epsilon})$\\
        \citet{vietri2020privaterl} & Tab. & $\epsilon,0$ & N/A & $\wt{O}\left(H^2\sqrt{SAK} + \frac{SAH^3 + S^2AH^3}{\epsilon}\right)$\\
        \hline
        \hline
        Our - Cor.~\ref{cor:JDP-UCRL-VTR} & LM & $\epsilon,\delta$ & N/A & $\wt{O}\left(C_w d^{\frac{3}{4}}H^{\frac{9}{4}}\sqrt{K}/\epsilon^{\frac{1}{2}} + dH^2\sqrt{K}\right)$\\
        Our - Cor.~\ref{cor:LDP-UCRL-VTR} & LM & $\epsilon,\delta$ & $\epsilon,\delta$ & $\wt{O}\left(C_w d^{\frac{3}{4}}H^{\frac{5}{2}}K^{\frac{3}{4}} /\epsilon^{\frac{1}{2}} + dH^2\sqrt{K}\right)$\\
        Our - Thm.~\ref{th:regret_JDP-LSVI-UCB-RarelySwitch} & L & $\epsilon,\delta$ & N/A & $\wt{O}\left(d^{\frac{8}{5}}H^{\frac{11}{5}}K^{\frac{3}{5}}/\epsilon^{\frac{2}{5}}+ d^{\frac{3}{2}}H^2\sqrt{K}\right)$\\
        \hline
    \end{tabular}
    \caption{Summary of privacy/regret values for exploration in RL (Tab. = tabular, LM=linear mixture, L=linear). Refer to Rem.~\ref{rem:laplace.mixture} and~\ref{rem:puredp.linear} for a discussion around pure differential privacy ($\delta = 0$).}
    \label{tab:summary.results}
\end{table*}

\section{PRELIMINARIES}\label{sec:preliminaries}

We consider a time-inhomogeneous finite-horizon Markov decision process (MDP) $M = \big(\mathcal{S}, \mathcal{A}, H,$ $\{r_h\}_{h=1}^H,\{p_h\}_{h=1}^H\big)$ where $\mathcal{S}$ is the state space and $\mathcal{A}$ is the action space, $H$ is the length of the episode, $\{r_h\}$ and $\{p_h\}$ are the reward function and state-transition probability measure.
We assume that $r_h : \mathcal{S} \times \mathcal{A} \to [0,1]$ is the deterministic reward function at step $h$.\footnote{While we consider deterministic rewards, our work can be easily generalized to unknown stochastic rewards.}
We assume that $\mathcal{S}$ is a measurable space with a possibly infinite number of elements and $\mathcal{A}$ is a finite set.
A deterministic policy $\pi = (\pi_1,\ldots,\pi_H) \in \Pi$ is a sequence of decision rules $\pi_h : \mathcal{S} \to \mathcal{A}$.
For every $h \in [H] := \{1, \ldots, H\}$ and $(s,a) \in \mathcal{S}\times\mathcal{A}$, we define the value functions of a policy $\pi$ as $V^\pi_h(s,a) = Q_h^{\pi}(s,\pi_h(s))$ and 
$Q_h^\pi(s,a) = r_h(s,a)+\mathbb{E}_{\pi}\left[ \sum_{i=h+1}^H r_i(s_i,a_i) \right]$,
where the expectation is over probability measures induced by the policy and the MDP over state-action sequences of length $H-h$.
Under certain regularity conditions~\citep[e.g.,][]{bertsekas2004stochastic}, there always exists an optimal deterministic policy $\pi^\star$, such that $V^{\pi^\star}_h(s):=V^\star_h(s)= \sup_{\pi} V^\pi_h(s)$ and $Q^{\pi^\star}_h(s,a):=Q^\star_h(s,a)= \sup_{\pi} Q^\pi_h(s,a)$.

When the state space is large or continuous, it is common to resort to a parametric representation of the MDP.
In this paper, we consider MDPs having linear structure~\citep{jin2020lsviucb} or linear-mixture structure~\citep{ayoub2020vtricml}, two common assumptions which enables efficient learning through model-free and model-based algorithms.
\begin{assumption}[Linear MDP]\label{asm:lowrank}
    Let $\phi : \mathcal{S}\times\mathcal{A} \to \mathbb{R}^d$ be a feature map.
    An MDP has linear structure if 
    \begin{align*}
        \forall s,a,h,s', \quad r_h(s,a) &= \phi(s,a)^\transp \theta_h,\\
         p_h(s'|s,a) &= \phi(s,a)^\transp \mu_h(s')
    \end{align*}
    where $\mu_h : \mathcal{S} \to \mathbb{R}^d$. 
    We assume $\|\theta_h\|_2 \leq \sqrt{d}$, $\|\int_{\mathcal{S}} \mu_h(s')v(s')\mathrm{d}s'\|_2\leq \sqrt{d}\|v\|_{\infty}$ and $\|\phi(s,a)\|_2\leq 1$, for any $s,a,h$ and function $v : \mathcal{S} \to \mathbb{R}$.
\end{assumption}
This definition implies that the value function of any policy is linear in the features, i.e., for any policy $\pi \in \Pi$, $\exists w_h^\pi \in \mathbb{R}^d$ such that $Q^\pi_h(s,a) = \phi(s,a)^\transp w_h^\pi$.
In the context of regret minimization, \citet{jin2020lsviucb} proposed a model-free algorithm that achieves a $\wt{O}(\sqrt{d^3 H^4K})$ regret bound.

\begin{assumption}[Linear-Mixture MDP]\label{asm:linear.mixture}
    Let $\phi : \mathcal{S}\times\mathcal{A}\times \mathcal{S} \to \mathbb{R}^d$ be a feature map. An MDP has linear-mixture structure if 
    \begin{align*}
        \forall s,a,s',h \quad p_h(s'|s,a) = \phi(s'|s,a)^\transp w_h
    \end{align*}
    where $w_h \in \mathbb{R}^d$ and $\|w_h\|_2 \leq C_{w}$.
    Furthermore, for any bounded function $v : \mathcal{S} \to [0,1]$, we have that $\|\phi_v(s,a)\| \leq 1$ where $\phi_v(s,a) = \int_{s' \in \mathcal{S}} \phi(s'|s,a) v(s') \mathrm{d}s'$ is computable.
\end{assumption}
In contrast to Asm.~\ref{asm:lowrank}, this structure allows to directly estimate the transition model (i.e., $(w_h)_h$) through a linear regression problem.
This property was used by \citet{ayoub2020vtricml} to design a model-based algorithm whose regret bound is $\wt{O}(dH^{2}\sqrt{K})$. This bound was refined to $\wt{O}(dH^{3/2}\sqrt{K})$ in~\citep{zhou2021vtrplus}, through the use of variance-aware concentrations.

\subsection{Privacy-Preserving Exploration}
We consider the standard online interaction protocol where users interact with a learning algorithm (also called \emph{server} in the privacy literature). As common, we assume $r$ is known and $p$ is unknown. At each episode $k$, a new user $\mathfrak{u}_k$\footnote{Similarly to~\citep{vietri2020privaterl,garcelon2020local}, a user here is represented by a tree of depth $H$, where each node contains the state, reward and the features associated to this state.} arrives and their personal information is encoded in state $s_{1k}$. The learning algorithm provides a policy $\pi_k$ that is executed by the user, and a trajectory $\tau_k = (s_{hk}, a_{hk}, r_{hk})_{h\in[H]}$ is collected and sent to the algorithm. We evaluate the performance of a learning algorithm by its cumulative regret $R(K) = \sum_{k=1}^K V^\star(s_{1k}) - V^{\pi_k}(s_{1k})$.
In addition, we want to guarantee that the algorithm preserves users' privacy while minimizing the regret.

In RL, two notions of differential privacy~\citep{dwork2006calibrating} has been considered.
Joint DP~\citep[e.g.][]{vietri2020privaterl} is the analogous of central DP normally considered in machine learning.
Given two sequences of $K$ users $\mathfrak{U}_K = (\mathfrak{u}_1, \dots, \mathfrak{u}_K)$ and $\mathfrak{U}'_K = (\mathfrak{u}'_1, \dots, \mathfrak{u}'_K)$, these sequences are said to be $k_0$-neighbors if $\mathfrak{u}_i = \mathfrak{u}'_i$ for all $i \neq k_0$ and $\mathfrak{u}_{k_0} \neq \mathfrak{u}'_{k_0}$.
For any $k_0 \in [K]$, let $\mathcal{A}_{-k_0}(\mathfrak{U}^K)$ be the set of possible sequence of actions (excluding the one from user $k_0$) produced by an algorithm $\mathcal{A}$.
\begin{definition}[Joint DP]
    An algorithm $\mathcal{A}$ is $(\epsilon,\delta)$-JDP under continual observation if for any $k_0 \in [K]$, any pair of $k_0$-neighboring user's sequences $\mathfrak{U}_K$, $\mathfrak{U}'_K$, we have:
    \[
        \mathbb{P}\left(\mathcal{A}_{-k_0}(\mathfrak{U}_K) \in E\right) \leq e^\epsilon\mathbb{P}\left(\mathcal{A}_{-k_0}(\mathfrak{U}'_K) \in E\right) + \delta
    \]
    where $E \subseteq A^{H \times [K-1]} \times \Pi$ is a set of sequence of actions and a policy.
\end{definition}
In terms of privacy model, \emph{users trust the RL algorithm with their raw data} (i.e., states, rewards, features, etc.) and the adversary can only observe the output of the algorithm.
The privacy guarantee ensures that when a user changes, the actions computed by the algorithm for the other $K-1$ users stay the same, hence the adversary can not infer the sequence of states, actions, features and rewards associated to the changed user.

\citet{garcelon2020local} studied the stronger \emph{local differential privacy} (LDP) notion. In contrast to JDP, users do not trust the algorithm with their sensitive data. Instead, the algorithm has access to user information only through samples that have been privatized, i.e., \emph{it cannot observe directly states, actions, features and rewards}. 
The appeal of this local model is that privatization can be done on the user's side using a randomizer $\mathcal{M}$, making the protocol more secure.
\begin{definition}[Local DP]
    For any $\epsilon \geq 0$ and $\delta \geq 0$, a privacy preserving-mechanism $\mathcal{M}$ is said to be $(\epsilon, \delta)$-LDP iff for all users $\mathfrak{u}, \mathfrak{u}' \in \mathfrak{U}$, trajectories $(\tau_{\mathfrak{u}}, \tau_{\mathfrak{u}'}) \in \mathcal{T}_{\mathfrak{u}} \times \mathcal{T}_{\mathfrak{u}'}$ and all $E \subset \{\mathcal{M}(\mathcal{T}_\mathfrak{u})|\mathfrak{u} \in \mathfrak{U}\}$:
    \[
        \mathbb{P}(\mathcal{M}(\tau_{\mathfrak{u}}) \in E) \leq e^{\epsilon}\mathbb{P}(\mathcal{M}(\tau_{\mathfrak{u}'}) \in E) + \delta
    \]
    where $\mathcal{T}_{\mathfrak{u}}$ is the space of trajectories associated to the user $\mathfrak{u}$. 
\end{definition}

\section{PRIVACY IN LINEAR-MIXTURE MDPS}\label{sec:linear.mixture}

\begin{algorithm}[t]
\small
    \caption{Privacy-Preserving UCRL-VTR}
    \label{alg:Perturbed-UCRL-VTR}
    
    \KwIn{episodes $K$, horizon $H$, ambient dimension $d$, privacy parameters $\epsilon,\delta$, failure probability $p$, bound $C_w$}
    Initialize $\lambda = H^2$ and, $\forall h \in [H]$,  $\wt{\Lambda}_{1,h} = \lambda\cdot I_{d\times d}$, $\wt{u}_{1,h} = 0_d$, $\wt{w}_{1,h} = 0_d$\\
    Initialize $\beta_k$ satisfying the condition in Thm.~\ref{th:regret_Private_UCRL-VTR}\\
    \For{$k=1, \dots, K$}{
        \tcp{\small User's side}
        Receive $(\wt{\Lambda}_{k,h})_{h \in [H]}$ and $(\wt{w}_{k,h})_{h \in [H]}$\\
        \For{$h = 1, \dots, H$}{
            Observe $s^k_h$\\
            Choose action $a^k_h = \argmax_{a \in A} \wt{Q}_{k,h}(s^k_h,a)$\\
            $\wt{Q}_{k,h}(\cdot,\cdot) = \min\Big\{H,r_h(\cdot,\cdot) + \langle\phi_{\wt{V}_{k,h+1}}(\cdot, \cdot),\wt{w}_{k,h}\rangle + \beta_k\bignorm{\phi_{\wt{V}_{k,h+1}}(\cdot,\cdot)}_{\wt{\Lambda}^{-1}_{k,h}}\Big\}$
            \\
            $\wt{V}_{k,h}(\cdot) = \max_{a \in A} \wt{Q}_{k,h}(\cdot,a)$
        }
        For any $h \in [H]$, let $X_{k,h} = \phi_{\wt{V}_{k,h+1}}(s^k_h,a^k_h)$, $y_{k,h} = \wt{V}_{k,h+1}(s^k_{h+1})$\\
        Send $\left\{X_{k,h}X^\intercal_{k,h} + B^1_{k,h}, X_{k,h}y_{k,h} + g^1_{k,h}\right\}_{h \in [H]}$ \\
        \tcp{\small Server's side (i.e., algorithm)}
        Update design matrix and target $\forall h \in [H]$\\
        \For{$h = 1, \dots, H$}{
            $B_{k+1,h} = \left(\sum_{1 \leq i \leq k} B^1_{i,h}\right) + B^2_{k,h}$\\
            $g_{k+1,h} = \left(\sum_{1 \leq i \leq k} g^1_{i,h}\right) + g^2_{k,h}$\\
            $\wt{\Lambda}_{k+1,h} = \lambda I_{d\times d} + \sum_{i=1}^k X_{i,h}X^\intercal_{i,h} + B_{k+1,h}$ \\
            $\wt{u}_{k+1,h} = \sum_{i=1}^{k} X_{i,h}y_{i,h} + g_{k+1,h}$\\
            $\wt{w}_{k+1,h} = \wt{\Lambda}^{-1}_{k+1,h}\cdot \wt{u}_{k+1,h}$
        }
    }
\end{algorithm}

In the non-private setting, efficient algorithms for exploration leverage the structure in Asm.~\ref{asm:linear.mixture} to build an estimate of the dynamics.
At each episode $k$, they compute an estimate $w_{k,h}$ of the unknown transitions and a confidence ellipsoid
\[
\wh{\mathcal{C}}_{k,h} = \{w | (w - w_{k,h})^\transp \Lambda_{k,h}(w - w_{k,h}) \leq \beta_{k}\}.
\]
Then, following the optimism in the face of uncertainty principle, they construct an optimistic estimate of the optimal Q-function
\[
    Q_{k,h}(\cdot,\cdot) = \min \Big\{H, r_h(\cdot,\cdot) + \max_{w \in \wh{\mathcal{C}}_{k,h}} \langle w, \phi_{V_{k,h+1}}(\cdot,\cdot) \rangle \Big\},
\]
and $V_{k,h}(\cdot) = \max_a Q_{k,h}(\cdot, a)$.
The main difference resides in the way these confidence intervals are built. For example, \ucrlvtr~\citep{ayoub2020vtricml} uses a standard self-normalized concentration inequality, while \ucrlvtrplus~\citep{zhou2021vtrplus} a variance-aware concentration bound.
For ease of presentation, we focus on \ucrlvtr{} in this section, the reader may refer to App.~\ref{app:vtrplus.framework} for the variant based on \ucrlvtrplus.

Guaranteeing privacy presents the additional challenge of privatizing these estimates without significantly  deteriorating the regret.
By noticing that $w_{k,h}$ is computed through ridge regression~\citep[e.g., ][Eq. 4]{ayoub2020vtricml}, we can take inspiration from the linear bandit literature~\citep{shariff2018nips,zheng2020locally} to build private estimates.
Denote by $\wt{w}_{k,h} = (\wt{\Lambda}_{k,h})^{-1} \wt{u}_{k,h}$ the model estimate, where\footnote{In all the paper, we use the notation $\,\wt{\cdot}\,$ to denote private (i.e., perturbed) quantities.}
\begin{align}
    \wt{\Lambda}_{k,h} &= \underbrace{\lambda I_{d \times d} + \sum_{i=1}^{k-1} X_{i,h} X_{i,h}^\transp}_{:=\Lambda_{k,h}} + B_{k,h}\\
    \wt{u}_{k,h} &= \sum_{i=1}^{k-1} X_{i,h} y_{i,h} + g_{k,h}
\end{align}
and $X_{k,h} = \phi_{\wt{V}_{k,h+1}}(s_h^k,a_h^k)$.
First, note that setting $B_{k,h} =0$ and $g_{k,h} =0$ yields the non-private \ucrlvtr.\footnote{To be precise, we should set $B^1_{k,h} = B^2_{k,h} = 0_{d\times d}$ and $g^1_{k,h} = g^2_{k,h} = 0_d$.}
On the other hand, by making $B_{k,h}$ and $g_{k,h}$ carefully designed noise terms, we can guarantee either JDP or LDP privacy for the sequence $(\wt{w}_{k,h}, \wt{\Lambda}_{k,h})_{k,h}$ observed by the users.

Before proceeding, note that $B_{k,h}$ and $g_{k,h}$ can be further decompose as follows: $B_{k,h} = B_{k,h}^2 + \sum_{i=1}^{k-1} B_{i,h}^1$ and $g_{k,h} = g_{k,h}^2 + \sum_{i=1}^{k-1} g_{i,h}^1$.
While the terms $B^1_{i,h}$ and $g^1_{i,h}$ regulates the LDP level, the terms $B^2_{i,h}$ and $g^2_{i,h}$ are responsible for the JDP guarantees. We can already notice that, as expected, the impact of LDP is much larger than the one of JDP since it is cumulated over episodes.
For example, in order to ensure JDP, we set $B^1_{k,h} = 0_{d\times d}$ and $g^1_{k,h} = 0_d$, whereas $B^2_{k,h}$ and $g^2_{k,h}$ are perturbations associated to the tree-based method (see App.~\ref{app:ucrlvtr}). 
On the other way, the LDP algorithm follows by adding gaussian noises in the information sent to the server (i.e., $B^1_{k,h}, g^1_{k,h}$) and setting $g^2_{k,h} = 0_d$ and $B^2_{k,h}$ as a deterministic matrix (see Cor.~\ref{cor:LDP-UCRL-VTR}).
This shows the flexibility of the proposed unified framework.

We are now ready to present the privacy/regret results for the unified algorithm in Alg.~\ref{alg:Perturbed-UCRL-VTR}. 
\begin{corollary}[JDP-UCRL-VTR]\label{cor:JDP-UCRL-VTR}
    Fix any privacy level $\epsilon, \delta \in (0,1)$. Set $B^1_{k,h} = 0_{d\times d}$, $g^1_{k,h} = 0_d$ and $B^2_{k,h}$, $g^2_{k,h}$ as the noises associated to the $k$-th prefix of trees containing in each node symmetric matrices/vectors with all of its entries $\mathcal{N}(0,\sigma^2_B)$, where $\sigma_B = \frac{32H^2}{\epsilon}\sqrt{2HK_0\log\left(\frac{8H}{\delta}\right)\log\left(\frac{4}{\delta}\right)\log\left(\frac{16HK_0}{\delta}\right)}$ and $K_0 = \lceil\log_2(K) + 1\rceil$. 
    Let $\Upsilon^J_{\frac{p}{6KH}} = \sigma_B\sqrt{K_0}\left(4\sqrt{d} + 2\log\left(\frac{6KH}{p}\right)\right)$ and choose $\beta_k$ as:
\begin{equation*}\label{eq:beta_JDP-UCRL-VTR}
\beta_k = 3(C_w + 1)\sqrt{\lambda + \Upsilon^J_{\frac{p}{6KH}}} + \sqrt{2H^2\log\left(\frac{3H\cdot\left(1+KH\right)^{\frac{d}{2}}}{p}\right)} 
\end{equation*}
Then, for any $p \in (0,1)$, Alg.~\ref{alg:Perturbed-UCRL-VTR} is $(\epsilon,\delta)$-JDP and with probability at least $1-p$, its regret is bounded by:
\[
    R(K) = \widetilde{O}\left((C_w + 1)\cdot\frac{d^{\frac{3}{4}}H^{\frac{9}{4}}\sqrt{K}}{\sqrt{\epsilon}} + dH^2 \sqrt{K}\right) 
\]
where $\widetilde{O}(\cdot)$ hides polylog$\left(\frac{1}{p}, \frac{1}{\delta}, H, K\right)$ factors.
\end{corollary}
In the following corollary, we show that by simply changing the noise perturbations we can achieve LDP guarantees.
Note that the algorithm does not need to know the features $\phi(s'|s,a)$ and thus we can preserve \emph{local} privacy over those features. On the other hand, we assume the users have access to the features for planning and computing the information $X_{h,k}$ and $y_{k,h}$ for the algorithm.
\begin{corollary}[LDP-UCRL-VTR]\label{cor:LDP-UCRL-VTR} 
    Fix any privacy level $\epsilon, \delta \in (0,1)$. 
    Set $B^2_{k,h} = 2\Upsilon^L_{\frac{p}{6KH}}\cdot I_{d\times d}$, $g^2_{k,h} = 0_d$ and $B^1_{k,h} \sim \mathcal{N}(0, \sigma_B^2 1_{d\times d})$, $g^1_{k,h} \sim \mathcal{N}(0, \sigma_B^2 I_{d\times d})$, where $\sigma_B = \frac{4H^3}{\epsilon}\sqrt{2\log\left(\frac{4H}{\delta}\right)}$ and $\Upsilon^L_{\frac{p}{6KH}} = \sigma_B\sqrt{K}\left(4\sqrt{d} + 2\log\left(\frac{6KH}{p}\right)\right)$.
    Then, choosing $\beta_k$ as follows:
    \begin{equation*}\label{eq:beta_LDP-UCRL-VTR}
    \beta_k = 3(C_w+1)\sqrt{\lambda + \Upsilon^L_{\frac{p}{6KH}}} + \sqrt{2H^2\log\left(\frac{3H\cdot\left(1+KH \right)^{\frac{d}{2}}}{p}\right)}
    \end{equation*}
    Alg.~\ref{alg:Perturbed-UCRL-VTR} is $(\epsilon,\delta)$-LDP and, for any $p \in (0,1)$, with probability at least $1-p$, its regret is bounded as follows:
    \[
      R(K) = \widetilde{O}\left((C_w + 1)\cdot\frac{d^{\frac{3}{4}}H^{\frac{5}{2}}K^{\frac{3}{4}}}{\sqrt{\epsilon}} + dH^2 \sqrt{K}\right).
    \]
\end{corollary}
These corollaries show a clear separation between tabular and linear-mixture settings. Indeed, while we obtain a $\sqrt{T}$ regret bound for JDP as in the tabular setting, the cost of privacy is not anymore additive but rather multiplicative. We pay an ever higher price in LDP where we are not able to recover a $\sqrt{T}$ regret bound.
On the other hand, these results are aligned with the one available for linear contextual bandits~\citep{shariff2018nips,zheng2020locally} with JDP and LDP.
The higher dependences in $d$ and $H$ compared to the non-private setting are due to the need of enlarging the confidence intervals to deal with the added noise. See the proof sketch for more details.

\begin{remark}[Pure DP]\label{rem:laplace.mixture}
Both corollaries exploit two properties of gaussian distributions: \emph{i)} they are bounded with high probability; \emph{ii)} gaussian noise is enough to ensure \emph{approximate} differential privacy by properly tuning the variance of the distribution. Note that these properties also hold for other distributions (e.g., Laplace noise). 
For example, replacing the gaussian mechanisms by the Laplace mechanisms provides an algorithm that can achieve pure differential privacy ($\delta = 0$) at the cost of worsening the dependency on $d$ and $H$ in the regret bounds of both corollaries. Indeed, every application of the advanced composition theorem would be replaced by the simple composition and hence, the variance of the noise would increase by a factor $\sqrt{H}$. Similarly, $\Upsilon_{\frac{p}{6KH}}$ increases by a factor $\sqrt{d}$. These changes translate into an additional factor $(dH)^{\frac{1}{4}}$ in Cor.~\ref{cor:JDP-UCRL-VTR} and $d^{\frac{1}{4}}$ in Cor.~\ref{cor:LDP-UCRL-VTR} (see App.~\ref{app:ucrlvtr}).
\end{remark}

Finally, we refer to the reader to App.~\ref{app:vtrplus.framework} for a detailed comparison with the variant based on \ucrlvtrplus. Notably, similarly to the tabular setting~\citep{garcelon2020local}, the use of variance-aware concentrations does not bring any advantage (in the main term) since the noise used for privacy dominates the regret bound.

\subsection{Proof Sketch of Corollary~\ref{cor:JDP-UCRL-VTR} and~\ref{cor:LDP-UCRL-VTR}}
In this section, we provide a sketch of the proof. We start providing an intuition about the privacy analysis.

\textbf{JDP privacy.}
The tree-based method (see App.~\ref{app:tree.method}) ensures that the released matrix $B_{k+1,h}$ is a sum of at most $K_0 \approx \log_2(K)$ matrices with gaussian entries $\mathcal{N}\left(0,\sigma^2_B\right)$.
Due to the choice of $\sigma_B$ and remarking that the sensitivity of $X_{i,h}X^\intercal_{i,h}$ is at most $2H^2$ (since the Frobenius norm $\bignorm{X_{i,h}X^\intercal_{i,h}}_F \leq H^2$), 
$(\wt{\Lambda}_{k,h})_k$ is indeed $\Big(\epsilon\Big(2\sqrt{8H\log\left(\frac{4}{\delta}\right)}\Big)^{-1}, \frac{\delta}{4H}\Big)$-DP for any $h \in [H]$.
The same result can be shown for the sequences $(\wt{u}_{k,h})_k$.
Then, by advanced composition (App.~\ref{app:advanced.comp}),  $\left(\wt{\Lambda}_{k,h}\right)_{k,h}$ and $\left(\wt{u}_{k,h}\right)_{k,h}$ are $\left(\frac{\epsilon}{2}, \frac{\delta}{2}\right)$-DP. 
By simple composition (App.~\ref{app:simple.comp}) and the post-processing property (App.~\ref{app:post.processing}), then we prove that the sequence $\left(\wt{\Lambda}_{k,h}, \wt{u}_{k,h}\right)_{(k,h)}$ is $(\epsilon,\delta)$-DP and by the post-processing property, $(\wt{\Lambda}_{k,h}, \wt{w}_{k,h})$ follows.
As a consequence, Alg.~\ref{alg:Perturbed-UCRL-VTR} is $(\epsilon,\delta)$-JDP due to the Billboard lemma (App.~\ref{app:billboard.lemma}) since the $Q$-functions $Q^P_{k,h}(\cdot, \cdot)$ are a function of $\wt{\Lambda}_{k,h}$, $\wt{w}_{k,h}$ and user's private data.

\textbf{LDP privacy.}
It suffices to show that $\left(X_{k,h}X^\intercal_{k,h} + B^1_{k,h}\right)_h$ and $\left(X_{k,h}y_{k,h} + g^1_{k,h}\right)_h$ are each $\left(\frac{\epsilon}{2},\frac{\delta}{2}\right)$-LDP.
Due to simple composition of independent mechanisms, it is enough to prove that each $(k,h)$-term is $\left(\frac{\epsilon}{2H}, \frac{\delta}{2H}\right)$-LDP.
This claim follows from the choice of $\sigma_B$, the sensitivity of each of the terms $X_{i,h}X^\intercal_{i,h}$, $X_{i,h}y_{i,h}$ ($2H^2$ in both cases) and the design of the gaussian mechanisms.

\textbf{Regret analysis.}
By leveraging the boundness of the noise, we can provide an unified regret analysis of Alg.~\ref{alg:Perturbed-UCRL-VTR}. Cor.~\ref{cor:JDP-UCRL-VTR} and Cor.~\ref{cor:LDP-UCRL-VTR} follows from a specific instantiation of the parameters in the following theorem.
\begin{theorem} \label{th:regret_Private_UCRL-VTR}
    Assume that the following hypothesis hold with probability at least $1-\frac{p}{3}$:
    \begin{enumerate}[noitemsep,topsep=0pt,parsep=0pt,partopsep=0pt,leftmargin=.4cm]
        \item There exist two sequences of non-negative real numbers $(\Upsilon^k_{low})_k, (\Upsilon^k_{high})_k$ such that all the eigenvalues of $B_{k,h}$ belong to $[\Upsilon^k_{low}, \Upsilon^k_{high}]$, for all $ h \in [H]$.
        \item There exists a sequence of non-negative real numbers $(C_k)_k$ such that $\bignorm{g_{k,h}}_2 \leq C_k$ for all $h \in [H]$.
        \item $(\beta_k)_k$ is a non-decreasing\footnote{Any sequence $(\beta_k)_k$ can be converted into the non-decreasing sequence $(\max\left\{\beta_1,\dots,\beta_k\right\})_k$} sequence satisfying the following inequality:
    \end{enumerate}
        \begin{equation*}\label{eq:beta_ineq_UCRL-VTR}
        \beta_k \geq \frac{(\lambda+\Upsilon^k_{high})C_w + C_k}{\sqrt{\lambda + \Upsilon^k_{low}}} + \sqrt{2H^2\log\left(\frac{3H\left(1+KH\right)^{\frac{d}{2}}}{p}\right)}
        \end{equation*}
    Then, for any $p \in (0,1)$, with probability at least $1-p$, the regret of Alg.~\ref{alg:Perturbed-UCRL-VTR} is bounded as follows:
    \[
        R(K) = \widetilde{O}\left(\left(Hd^{\frac{1}{2}}\beta_K + H^{\frac{3}{2}}\right)\cdot K^{\frac{1}{2}}\right)
    \]
\end{theorem}
This theorem shows that the regret bound is directly proportional to the width of the confidence interval.
Indeed, in JDP, we obtain the same width as in the non-private case, recovering the standard $\sqrt{K}$ rate. 
On the other hand, LDP requires to enlarge the confidence intervals by a factor $K^{1/4}$ to balance the noise for privacy, leading to a higher regret bound.
The corollaries follow by noticing that the perturbations described before are bounded with high probability. We have that $\Upsilon^k_{low} = \Upsilon^J_{\frac{p}{6KH}}$ (resp. $\Upsilon^L_{\frac{p}{6KH}}$), $\Upsilon^k_{high} = 3\Upsilon^J_{\frac{p}{6KH}}$ (resp. $3\Upsilon^L_{\frac{p}{6KH}}$), while $C_k = \sigma_B\sqrt{\underline{K}}\left(\sqrt{d} + 2\sqrt{\log\left(\frac{6KH}{p}\right)}\right)$ where $\underline{K} \approx \log_2(K)$ for JDP and $\underline{K} \approx K$ for LDP. Furthermore, $\beta_k \approx \sqrt{\log(K)}$ for JDP and $\beta_k \approx K^{1/4}$ for LDP satisfy the third condition.

\section{PRIVACY IN LINEAR MDPS}\label{sec:linear.mdp}
\begin{algorithm*}[t]
\small
    \caption{Privacy-Preserving LSVI-UCB through Batching and JDP}
    \label{alg:JDP-LowRank-RareSwitch.correct}    
    \KwIn{episodes $K$, horizon $H$, ambient dimension $d$, privacy parameters $\epsilon,\delta$ and failure probability $p$.}
    Initialize $K_0, B, \sigma_u, \sigma_\Lambda, \Upsilon^J_{\frac{p}{6KH}}, \lambda, c_K$ and $\beta$ as in App.~\ref{app:parameters.jdp.lsviucb.correct}\\
    Initialize $\wt{\Lambda}_{0,h} = \lambda I_{d\times d}$, $\wt{w}_{0,h} = 0_d$ $\forall h \in [H]$, batch counter $b=0$, $(\eta^i_h)_{(i,h) \in [B] \times [H]} \in \mathbb{R}^d$ s.t.\ $\eta^i_h \sim \mathcal{N}(0, \sigma^2_u\cdot I_{d \times d})$.\\
    Build $H$ trees $(H_h)_{h \in [H]}$ with $2B$ nodes, each node
    initialized with a $d \times d$ symmetric matrix $Z_\Lambda = \frac{Z'_\Lambda + Z'^T_\Lambda}{2}$, where $Z'_\Lambda$ is a matrix of Gaussian noises $\mathcal{N}(0, \sigma^2_\Lambda)$. \\
    For each $0 \leq i < B$, $k_i = i\cdot\lceil\frac{K}{B}\rceil + 1$. \\
    \For{$k=1, \dots, K$}{
        \tcp{\small user's side}
        Receive $(\widetilde{\Lambda}_{b,h})_{h \in [H]}$, $(\widetilde{w}_{b,h})_{h \in [H]}$\\
        \For{$h=1,\dots, H$}{
            Observe $D^k_h = \{\phi(s_{k,h},a)|a \in \mathcal{A}\}$\\
            Take action $a_{k,h} = \argmax_{a \in \mathcal{A}} Q_{k,h}(s_{k,h},a) = \Pi_{[0,H]}\left[ \wt{w}^\intercal_{b,h}\cdot\phi(s_{k,h},a) + \beta\norm{\phi(s_{k,h},a)}_{\widetilde{\Lambda}^{-1}_{b,h}}\right]$
        }
        Send $(D^k_h)_{h \in [H]}$ and $X_{u_k} = \{(s_{k,h},a_{k,h},r_{k,h})_{h \in [H]}\}$\\
        \tcp{\small server's side (algorithm)}
        \If{$k = k_{b+1}-1$}{
            $\wt{w}_{b+1,H+1} = 0$, $Q_{k+1,H+1}(\cdot, \cdot) = 0$\\
            \For{$h = H, \dots, 1$}{
                $\Lambda_{k+1,h} = \lambda\cdot I_{d\times d} + \sum_{i=1}^{k} \phi(s_{i,h},a_{i,h})\phi(s_{i,h},a_{i,h})^\transp$\\
                $\wt{\Lambda}_{b+1,h} = \Lambda_{k+1,h} + \left( c_K + \Upsilon^J_{\frac{p}{6KH}}\right)I_{d \times d} + H^{b+1}_h$ \tcp{$H^{b+1}_h$: $b+1$-th prefix of tree $H_h$}
                $u_{k+1,h} = \sum_{i=1}^k \phi(s_{i,h},a_{i,h})\left[ r_h(s_{i,h},a_{i,h}) + V_{k+1,h+1}(s_{i,h+1})\right]$\\
                $\wt{u}_{b+1,h} = u_{k+1,h} + \eta^{b+1}_h$\\
                $\wt{w}_{b+1,h} = \wt{\Lambda}^{-1}_{b+1,h}\cdot\wt{u}_{b+1,h}$\\
                $Q_{k+1,h}(\cdot,\cdot) = \Pi_{[0,H]} \left[\phi(\cdot,\cdot)^\transp\cdot\wt{w}_{b+1,h} + \beta\norm{\phi(\cdot,\cdot)}_{\wt{\Lambda}^{-1}_{b+1,h}}\right]$, $V_{k+1,h}(\cdot) = \max_{a \in \mathcal{A}} Q_{k+1,h}(\cdot, a)$
            }
            $b = b + 1$ \tcp{New batch, increase counter}
        }
    }
\end{algorithm*}

In this section, we provide the first privacy-preserving algorithm for exploration in linear MDPs.
While we are able to design an efficient algorithm with JDP guarantees, the design of an LDP algorithm remains an open question. 
We anticipate that the model-free structure of algorithms plays a critical role. In fact, the algorithms need to directly use features to perform planning, thus making unclear how to perform efficiently this step in a local model (see Rem.~\ref{rem:ldp.linearmdps}).

We based our algorithm on \lsviucb~\citep{jin2020lsviucb}, an efficient optimistic algorithm for exploration with $\wt{O}(\sqrt{d^3H^4K})$ minimax regret bound. However, simply adapting \lsviucb{} to JDP leads to a sub-optimal $\wt{O}(K^{3/4}/\sqrt{\epsilon})$ regret bound since to guarantee privacy on all the $K$ outputs of the algorithm, the variance of the noise scales with $\sqrt{K}$ and hence the regret increases by a $K^{1/4}$ factor (see Rem.~\ref{rem:nonbatched.lsvi.ucb} and ~\ref{app:rem.nonbatched.lsvi.ucb}).
We thus rely on batching techniques to decrease the number of policy switching and thus the number of elements that we need to privatize. Batching techniques have been investigate in the non-private literature for improving the computational efficiency of algorithms~\citep[e.g.,][]{gao2021adaptivity,wang2021adaptivity}. In our case, batching is critical for improving the minimax rate to $K^{\frac{3}{5}}$.
To the best of our knowledge, Alg.~\ref{alg:JDP-LowRank-RareSwitch.correct} is the first model-free exploration algorithm with JDP guarantees. 

Similarly to~\citep{wang2021adaptivity}, we employ a static batching scheme.
Let $b_k$ denote the batch at episode $k$ and $k_{b}$ the episode when batch $b$ started. Then, our algorithm updates the estimated $Q$-function every $\Big\lceil\frac{K}{B}\Big\rceil$ episodes. The new private estimates of $(w_h^{\pi^\star})_h$ are then computed through a \emph{perturbed} optimistic least-square value iteration.
Formally, we denote by $\wt{w}_{b_k,h} =  \wt{\Lambda}_{b_k,h}^{-1} \wt{u}_{b_k,h}$ the estimated value with 
\begin{align*}
    \widetilde{\Lambda}_{b_k,h} &= \Lambda_{k_{b_k}, h} + \left(c_K + \Upsilon^J_{\frac{p}{6KH}}\right)\cdot I_{d \times d} + H^{b_k}_h\\
    \wt{u}_{b_k,h} &= \eta_h^{b_k} + \sum_{i=1}^{k_{b_k}-1} \phi(s_{i,h},a_{i,h})\left( r_{i,h} + V^{k_{b_k}}_{h+1}(s_{i,h+1}) \right)
\end{align*}
$\widetilde{\Lambda}_{b_k,h}$ is the design matrix perturbed through the tree-based method, $\wt{u}_{b_k,h}$ is the response vector perturbed through Gaussian mechanism, $c_K$ is a regularizer that depends on the noise injected for privacy and $\Lambda_{k,h} = \lambda I_{d \times d} + \sum_{i=1}^{k-1} \phi(s_{i,h},a_{i,h})\phi(s_{i,h},a_{i,h})^\transp$.
When a new user $u_k$ arrives, the algorithm provides the terms $(\wt{w}_{b_k,h}, \wt{\Lambda}_{b_k,h})_h$ so that the user can play the optimistic greedy action.
Then, we can guarantee that the sequence of values is central-DP and the resulting algorithm is JDP.

The following theorem states the privacy/regret properties of Alg.~\ref{alg:JDP-LowRank-RareSwitch.correct}.
\begin{theorem}\label{th:regret_JDP-LSVI-UCB-RarelySwitch}
    Fix any privacy level $\epsilon, \delta \in (0,1)$. For any $p \in (0,1)$, Alg.~\ref{alg:JDP-LowRank-RareSwitch.correct} is $(\epsilon,\delta)$-JDP and, with probability at least $1-p$, its regret is bounded as follows:
    \[
        R(K) = \wt{O}\left(\sqrt{d^3H^4K} + \frac{d^{\frac{8}{5}}H^{\frac{11}{5}}K^{\frac{3}{5}}}{\epsilon^{\frac{2}{5}}}\right)
    \]
\end{theorem}
This theorem shows that it is possible to obtain a non-trivial $K^{\frac{3}{5}}/\epsilon^{\frac{2}{5}}$ regret bound for JDP exploration, leaving a $(K\epsilon)^{\frac{1}{10}}$ gap w.r.t.\ the linear-mixture setting. However, the model-free nature of this setting calls for a different technique for achieving such improved result result. 
We further  notice that, compared to the non-private case, in the worst case, we have a looser dependence w.r.t.\ $H$ and $d$ induced by the noise injected for privacy and the static batching schedule.
We conclude this section with remarks.

\begin{remark}[Pure DP]\label{rem:puredp.linear}
    Similarly to the linear-mixture setting, it is possible to obtain pure joint DP ($\delta=0$) with polynomially worse regret bounds (see App.~\ref{app:pureJDP.linear.setting}).
\end{remark}

\begin{remark}[Non-Batched Algorithm]\label{rem:nonbatched.lsvi.ucb}
As already mentioned, following a similar approach, it is possible to design an $(\epsilon,\delta)$-JDP algorithm based on LSVI-UCB.
The regret bound would be $\wt{O}(K^{3/4}/\sqrt{\epsilon})$, a $\wt{O}(K^{\frac{3}{20}}/\epsilon^{\frac{1}{10}})$-factor worse than the one of Alg.~\ref{alg:JDP-LowRank-RareSwitch.correct}. The intuition is that batching leads to less than $B = O\left((K\epsilon)^{\frac{2}{5}}/(d^{\frac{3}{5}}H^{\frac{1}{5}})\right)$ policy switches, thus requiring privacy composition over a polynomially smaller number of objects. Without batching, the number of elements would be $K$ thus requiring to increase the noise variance $\sigma_u$ by a factor scaling with $K^{\frac{3}{5}}/\epsilon^{\frac{2}{5}}$ compared to the one used in Alg.~\ref{alg:JDP-LowRank-RareSwitch.correct}. This change would result in the additional multiplicative factor to the regret.
\end{remark}
\begin{remark}[Data-Dependent Batches]\label{rem:datadepbatch.lsvi.ucb}
An alternative approach for low-switching bandit is to leverage dynamic batching, e.g., using a determinant~\citep[e.g.,][]{abbasi-yadkori2011,gao2021adaptivity} or a trace-based condition~\citep[e.g.,][]{Calandriello2020gaus,garcelon2021helinearbandit}. While these schemes will lead to a logarithmic number of batches and to a potential improvement in the regret bound, it is unclear how to carry out the privacy analysis. The main issue is that the starting time of each batch becomes a random variable that depends on the specific user sequence. Thus changing a user at episode $k_0$ may affect the time of \emph{all} the subsequent baches, making the current technique for privacy analysis not applicable. It may be possible to protect batch times using the \emph{sparse vector technique}~\citep[e.g.][]{HardtR10} for checking the condition. However, it is unclear to us at the moment how to adapt their privacy analysis to our setting. We leave this as an open question.
\end{remark}
\begin{remark}[Local Differential Privacy]\label{rem:ldp.linearmdps}
A key feature we leverage in the design of Alg.~\ref{alg:JDP-LowRank-RareSwitch.correct} is that in the JDP setting the algorithm can directly access user's data (features, states, rewards, etc.). Indeed the JDP model assumes the attacker can only observe the output of the algorithm and should not be able to infer user's information from it.
On the other hand, LDP requires information to be privatized locally, preventing the algorithm to have access to clean features. In linear-mixture models, we overcome this limitation by leveraging the possibility of estimating the model and performing planning on the user's side. Indeed, all the information required to estimate the model can be computed locally and privatized before being sent to the algorithm. However, in the linear MDP setting, planning is performed on the server side using a model-free approach (least-square value iteration) which requires the knowledge of the features.

At each episode, {\small\textsc{LSVI}} requires to know the design matrix $\Lambda_{k+1,h}$ and the response vector $u_{k+1,h} = \sum_{i=1}^k \phi(s_{i,h},a_{i,h})\left( r_{i,h} + V_{k+1,h+1}(s_{i,h+1})\right)$ to compute the new estimate. The matrix $\Lambda_{k+1,h}$ is a simple sum of independent rank-1 matrices, and it can be privatized using a standard local randomizer (similarly to what done in the previous section). However, the response vector requires to re-evaluate the entire sum by using the current estimate $V_{k+1,h+1}$. A solution may be to provide to the algorithm with an LDP version of the feature $\phi(s_{i,h},a)$ and of the reward $r_{i,h}$. Unfortunately, the noise required to compute these independent private terms compounds, leading to a linear regret. Note that we would have to deal with the product of private terms (e.g., $\phi \cdot r$ and $\phi \cdot V$), thus multiplying the noise level. Note that even in the case of contextual linear bandit, privatizing independently rewards and contexts leads to a linear regret bound. 

Furthermore, even if a mechanism protecting the response vector $u_{k+1,h}$ is available, the privacy analysis of Alg.~\ref{th:regret_JDP-LSVI-UCB-RarelySwitch} leverages the advanced composition of adaptive DP mechanisms to reduce the amount of noise injected to $u_{k+1,h}$, by noticing that its sensitivity is bounded by a term independent of $K$.
While composition of adaptive theorems are available for DP, we are not aware of this result for LDP. As a consequence, the variance should scale linearly with $K$ as the sensitivity of $u_{k+1,h}$ can be bounded this time only by a linear term on $K$. 
We leave the design of an LDP algorithm for linear MDPs as an open question. We believe a substantial different approach is required to solve this problem.
\end{remark}

\subsection{Proof Sketch of Theorem~\ref{th:regret_JDP-LSVI-UCB-RarelySwitch}}

Full details of the proof are reported in App.~\ref{app:regret_JDP-LSVI-UCB-RarelySwitch}. We present here a sketch of the regret and privacy analysis. As it can be easily observed, the number of different policies played in Alg.~\ref{alg:JDP-LowRank-RareSwitch.correct}, $N_K$, is at most $B$. We recall that $0 \leq b_k < N_K$ is the batch at episode $k$ and $k_b$ the first episode of the $b$-th batch.
\paragraph{Privacy analysis.}
For a sequence of $K$ users $\mathfrak{U}^K$ and for each $h \in [H]$, let $M^\Lambda_h$ and $M^u_h$ be the following mechanisms acting on $\mathfrak{U}^K$:
$M^\Lambda_h(\mathfrak{U}^K) = (\wt{\Lambda}_{b_k,h})_{k \in [K]}, M^u_h(\mathfrak{U}^K, \mathfrak{a}_{h+1}) = (\bar{u}_{b_k,h})_{k \in [K]}$
where $\mathfrak{a}_{h+1}$ corresponds to the output at the next stage $\left(M^\Lambda_{h+1}(\mathfrak{U}^K), M^u_{h+1}(\mathfrak{U}^K,\mathfrak{a}_{h+2})\right)$, $\mathfrak{a}_{H+1} := \emptyset$ and $\bar{u}_{b_k,h}$:
\begin{equation}\label{eq:proof.sketch.jdp.lsvi.ucb}
\begin{aligned}
    \bar{u}_{b_k,h} = \sum_{i=1}^{k_{b_k}-1} \phi^i_h\Big\{r^i_h + \Pi_{[0,H]}\Big[\max_{a \in A} \Big\langle\left(\mathfrak{a}^\Lambda_{k,h+1}\right)^{-1}\mathfrak{a}^u_{k,h+1},\\
    \phi(s^i_{h+1},a) \Big\rangle+ \beta\bignorm{\phi(s^i_{h+1},a)}_{(\mathfrak{a}^\Lambda_{k,h+1})^{-1}} \Big] \Big\} + \eta^{b_k}_h
\end{aligned}
\end{equation}
It turns out that for the choice of Alg.~\ref{alg:JDP-LowRank-RareSwitch.correct}, $\mathfrak{a}_{h+1} = \left(\mathfrak{a}^\Lambda_{h+1},\mathfrak{a}^u_{h+1}\right) = \left( (\wt{\Lambda}_{b_k,h+1})_k,(\wt{u}_{b_k,h+1})_k\right)$, we have $\bar{u}_{b_k,h} = \wt{u}_{b_k,h}$. We notice that $(M^\Lambda_h,M^u_h)_h$ are adaptive mechanisms and hence, they are perfectly suited to apply the advanced composition theorem.

Fix any $h \in [H]$ and notice that the choice of $\sigma_\Lambda$ ensures that adding $H^{b+1}_h$ to $\Lambda_{k+1,h}$ protects it as an $(\epsilon',\delta')$-DP mechanism, where $\epsilon' = \frac{\epsilon}{2\sqrt{8H\log\left(\frac{2}{\delta}\right)}}$, $\delta' = \frac{\delta}{4H}$. The mechanism $M^\Lambda_h(\mathfrak{U}^K) = (\wt{\Lambda}_{b_k,h})_k$ is $(\epsilon',\delta')$-DP as well since for two $k_0$-neighboring sequences of users $\mathfrak{U}^K$ and $\mathfrak{U}'^K$, the released matrices $\wt{\Lambda}_{b_k,h},\wt{\Lambda}'_{b_k,h}$ protect the same matrix $\Lambda_{k_{b_k},h}$ given that the value of $b_k$ is independent of the sequence of users.

The next step is to show that $M^u_h(\mathfrak{U}^K,\mathfrak{a}_{h+1})$ is $(\epsilon',\delta')$-DP for each possible output $\mathfrak{a}_{h+1} = (\mathfrak{a}^\Lambda_{h+1},\mathfrak{a}^u_{h+1})$\footnote{It is necessary to prove it for any $\mathfrak{a}^u_{h+1}$ to apply advanced composition}. The key observation is that the term protected by $M^u_h$:
\begin{equation}\label{eq:proof.sketch.jdp.lsvi.ucb}
\begin{aligned}
    \sum_{i=1}^{k_{b_k}-1} \phi^i_h\Big\{r^i_h + \Pi_{[0,H]}\Big[\max_{a \in A} \Big\langle\left(\mathfrak{a}^\Lambda_{k,h+1}\right)^{-1}\mathfrak{a}^u_{k,h+1},\\
    \phi(s^i_{h+1},a) \Big\rangle+ \beta\bignorm{\phi(s^i_{h+1},a)}_{(\mathfrak{a}^\Lambda_{k,h+1})^{-1}} \Big] \Big\}
\end{aligned}
\end{equation}
has a bounded sensitivity for two $k_0$-neighboring sequence of users since the only difference would be at the term $k_0$ in the sum above as $k_{b_k}$ is the same for both sequence of users (static batch scheme). It then follows that the sensitivity is at most $2(H+1)$. Then, due to the choice of $\sigma_u$, adding $\eta^b_h$ to $u^{k_b}_h$ with Gaussian entries $\mathcal{N}(0,\sigma^2_u)$ makes the mechanism $\bar{u}^k_h$ roughly $\left(\frac{\epsilon'}{\sqrt{B}},\frac{\delta'}{2B}\right)$-DP. By advanced composition, it follows that $M^u_h(\mathfrak{U}^K,\mathfrak{a}_{h+1}) = (\bar{u}_{b_k,h})_k$ is $(\epsilon',\delta')$-DP since we have at most $B$ different values for $\bar{u}_{b_k,h}$.
By advanced composition, it follows that $(M^\Lambda_h, M^u_h)$ is $\left(\epsilon \big(\sqrt{8H\log\left(\frac{2}{\delta}\right)}\big)^{-1},\delta/2H\right)$-DP for each $h \in [H]$. A last application of advanced composition leads to prove that $(M^\Lambda_h, M^u_h)_h$ is $(\epsilon,\delta)$-DP.
It follows that Alg.~\ref{alg:JDP-LowRank-RareSwitch.correct} is $(\epsilon,\delta)$-JDP as a consequence of the Billboard lemma.

\paragraph{Regret analysis.} 
The analysis is fairly similar to the one in the non-private setting~\citep[e.g.,][]{wang2021adaptivity}.
The perturbations used for $\wt{\Lambda}_{b_k,h}$ and $\wt{u}_{b_k,h}$ impact the definition of the confidence set.
Indeed, assuming that we are in the ``good'' event, where the perturbations are bounded, $\beta$ will scale with $\wt{O}(H\sqrt{d(\lambda + c_K)})$, since $\lambda + c_K$ is the new smallest eigenvalue of the perturbed design matrix $\wt{\Lambda}_{b_k,h}$.
It is also easy to prove that under the good event, the value functions are optimist and hence a very similar argument to the one in~\citep[][Thm. 5.1]{wang2021adaptivity} leads to 
    $R(K) \leq 2\sqrt{2KH^3\log\left(\frac{3}{p}\right)} + \sum_{k=1}^K\sum_{h=1}^H \min\left\{H,2\beta\bignorm{\phi^k_h}_{\wt{\Lambda}^{-1}_{k_{b_k},h}}\right\}.$ 
In contrast to~\citet{wang2021adaptivity}, we build the set $\underline{\mathcal{C}}$ of pairs $(k,h)$ such that $\bignorm{\phi(s^k_h, a^k_h)}_{\wt{\Lambda}^{-1}_{b_k,h}} > 4\bignorm{\phi(s^k_h,a^k_h)}_{\left(\Lambda_{k,h} + \left(c_K+ 2\Upsilon^J_{\frac{p}{6KH}}\right)I_{d\times d}\right)^{-1}}$. Similarly, $\abs{\underline{\mathcal{C}}} \leq \Big\lfloor dH\cdot\Big\lceil\frac{K}{B}\Big\rceil\cdot \frac{\log\left(\frac{K}{d\lambda} + 1\right)}{2\log 4 - 2} \Big\rfloor$, hence
$R(K) \leq  2\sqrt{2KH^3\log(3/p)} + H\Big\lfloor dH\cdot\Big\lceil \frac{K}{B} \Big\rceil\cdot \frac{\log\left(\frac{K}{d\lambda} + 1\right)}{2\log 4 - 2} \Big\rfloor + 8\beta H\sqrt{2dK\log\left(\frac{K}{\lambda + c_K} + 1\right)}$. Carefully choosing $B = \Big\lceil \frac{(K\epsilon)^{\frac{2}{5}}}{d^{\frac{3}{5}}H^{\frac{1}{5}}}\Big\rceil$ makes $\beta = \wt{O}\left(dH + d^{\frac{11}{10}}H^{\frac{6}{5}}K^{\frac{1}{10}}/\epsilon^{\frac{2}{5}}\right)$ and the regret $R(K) = \wt{O}\left(d^\frac{3}{2}H^2\sqrt{K} + d^{\frac{8}{5}}H^{\frac{11}{5}}K^{\frac{3}{5}}/\epsilon^{\frac{2}{5}}\right)$.

\vspace{-0.05in}
\section{CONCLUSIONS}
\vspace{-0.05in}
In this paper, we provided the first analysis of privacy-preserving exploration in MDPs with linear structure.
For linear-mixture MDPs, we presented a unified framework that allowed us to prove a $\wt{O}(\sqrt{K/\varepsilon})$ regret bound for JDP and $\wt{O}(K^{3/4}/\sqrt{\varepsilon})$ for $(\epsilon,\delta)$-LDP. In the case of linear MDPs, we presented a batched algorithm (Alg.~\ref{alg:JDP-LowRank-RareSwitch.correct}) for $(\epsilon,\delta)$-JDP exploration, whose regret bound is $\wt{O}(K^{\frac{3}{5}}/\epsilon^{\frac{2}{5}})$.

Two interesting questions raised by our paper are whether it is possible to design a JDP algorithm matching the $\wt{O}\left( \sqrt{K/\epsilon}\right)$ regret bound of linear contextual bandits or a LDP algorithm for linear MDPs.
Another direction can be to investigate the use of shuffle privacy model~\citep[e.g.,][]{CheuSUZZ19,BalleBGN19} to get guarantees akin to local privacy but better regret bounds.

\bibliography{biblio}
\bibliographystyle{plainnat}

\clearpage
\begin{appendix}
\addcontentsline{toc}{section}{Appendix} 
\part{Appendix} 
\parttoc 

\section{Related Work}
Privacy-preserving exploration has been initially studied in the context of multi-armed bandits~\citep[e.g.,][]{MishraT15,TossouD16}.
More recently, \citet{shariff2018nips} showed that the standard DP notion (central-DP) is incompatible with regret minimization in contextual bandits, by showing a linear lower-bound. To overcome this limitation, they considered the notion of \emph{joint} DP, originally introduced for mechanism design~\citep{KearnsPRU14}. 
They proposed a $(\epsilon,\delta)$-JDP algorithm for contextual linear bandit whose regret bound is $\wt{O}\Big( \frac{d^{3/4}\sqrt{K}}{\sqrt{\epsilon}} \Big)$. This show that the cost of JDP is multiplicative in contextual linear bandit rather than being additive as in central-DP exploration for MABs~\citep[e.g.,][]{TossouD16,SajedS19}.
More recently, \citet{zheng2020locally} studied the problem of local-DP exploration in contextual linear bandit --they also considered the generalized linear case-- and proved a $\wt{O}\Big(\frac{(dK)^{3/4}}{\sqrt{\epsilon}}\Big)$.\footnote{In the paper, \citet{zheng2020locally} reported a linear dependence w.r.t.\ $\epsilon$. By checking the analysis we noticed that a $\sqrt{\epsilon}$-dependence can be obtained, improving the bound when $\epsilon <1$ that is the relevant regime in privacy-preserving exploration.} While our findings are aligned with the ones in the contextual bandit literature, the design of our algorithms is much more involved than the one for contextual linear bandits, as showed by the need of batching for JDP in linear MDPs and the challenges for designing LDP algorithms in the such a setting. \citet{zheng2020locally} also conjectured that the lower bound for LDP exploration is $\Omega(K^{3/4})$ in contextual linear bandits, but we are not aware of any proof of this conjecture.

In the RL literature, privacy-preserving exploration was only studied in tabular MDPs with $S$ states, $A$ actions and horizon $H$. \citet{vietri2020privaterl} introduced an $(\epsilon,0)$-JDP algorithm with $\wt{O}\left(H^2\sqrt{SAK} + \frac{SAH^3 + S^2AH^3}{\epsilon}\right)$ showing that, similarly to MABs, the cost of privacy is only additive to the regret. On the other hand, \citet{garcelon2020local} showed that the cost of LDP is multiplicative by deriving a $\Omega\Big( \frac{H\sqrt{SAK}}{\min\{e^{\epsilon}-1,1 \}} \Big)$ lower-bound. They also introduce an $(\epsilon,0)$-LDP algorithm with $\wt{O}(H^{\frac{3}{2}}S\sqrt{AK} + \frac{H^3S^2A\sqrt{K}}{\epsilon})$ regret bound. In this paper, we focus for the first time on privacy-preserving exploration in MDPs with parametric structure. 
\section{Table of notations}\label{app:table.notations}
We provide the following table for reference. 
\begin{table}[h]
    \centering
    \begin{tabular}{lll}
    \hline
    \multicolumn{3}{c}{Notation for linear-mixture MDPs}\\
    \hline
        $C_w$ & & parameter bound (i.e., $\|w_h\|_2 \leq C_w$)\\
        $\epsilon$ & & privacy level\\
        $\delta$ & & privacy approximation\\
        $p$ & & probability (for regret bound)\\
       $\wt{\Lambda}_{k,h}$ &  & perturbed design matrix\\
       $\wt{u}_{k,h}$ &  & perturbed response vector\\
       $\wt{w}_{k,h}$ &  & perturbed estimated vector\\
       $B_{k,h}^1$ & & LDP noise for design matrix\\
       $B_{k,h}^2$ & & JDP noise for design matrix\\
       $g_{k,h}^1$ & & LDP noise for response vector\\
       $g_{k,h}^2$ & & JDP noise for response vector\\
       $\phi_v(s,a)$ & $=$ & $\int_{s' \in \mathcal{S}} \phi(s'|s,a) v(s') \mathrm{d}s'$\\
       $\wt{Q}_{k,h}$ & $=$  & $\min\Big\{H,r_h(\cdot,\cdot) + \langle\phi_{\wt{V}_{k,h+1}}(\cdot, \cdot),\wt{w}_{k,h}\rangle + \beta_k\bignorm{\phi_{\wt{V}_{k,h+1}}(\cdot,\cdot)}_{(\wt{\Lambda}_{k,h})^{-1}}\Big\}$\\
       $\wt{V}_{k,h}(\cdot)$ & $=$ & $\max_a \wt{Q}_{k,h}(\cdot,a)$\\
       $X_{k,h}$ & $=$ & $\phi_{\wt{V}_{k,h+1}}(s_h^k,a_h^k)$\\
       $\eta_h^b$ & & JDP noise for response vector\\
       $[n]$ & $=$ & $\left\{1, 2, \dots, n\right\}$\\ 
   \hline
    \end{tabular}
    \label{tab:notation}
\end{table}

\begin{table}[h]
    \centering
    \begin{tabular}{lll}
    \hline
    \multicolumn{3}{c}{Notation for linear MDPs}\\
    \hline
        $\epsilon$ & & privacy level\\
        $\delta$ & & privacy approximation\\
        $p$ & & probability (for regret bound)\\
       $b_k$ &  & batch at episode $k$\\
       $k_b$ & & starting episode of batch $b$\\
       $\wt{\Lambda}_{b,h}$ &  & perturbed design matrix\\
       $\wt{u}_{b,h}$ &  & perturbed response vector\\
       $\wt{w}_{b,h}$ &  & perturbed estimated vector\\
       $\eta_h^b$ & & JDP noise for response vector\\
       $H_{h}^b$ & & noise from tree method\\
       $[n]$ & $=$ & $\left\{1, 2, \dots, n\right\}$\\
       $B$ & & Upper-bound to the total number of batches\\
       $\Pi_{[0,H]}(x)$ & & projection of $x$ on the interval $[0,H]$--i.e $\max\left\{0, \min\left\{H, x\right\}\right\}$\\
   \hline
    \end{tabular}
    \label{tab:notation}
\end{table}

\section{Analysis of privacy-preserving UCRL-VTR}\label{app:ucrlvtr}

\subsection{Proof of Thm.~\ref{th:regret_Private_UCRL-VTR}}\label{app:regret_Private_UCRL-VTR}
We follow the proof given in ~\citet{jia2020vtr}, but adopt the strategy of ~\citet{ayoub2020vtricml} relying on the self-normalized bound for vector-valued martingales to provide high-probability confidence sets for the unknown parameters $(w_h)_h$.
\subsubsection{Confidence set of value target regression}
We follow the generic idea in~\citep[][E.4]{ayoub2020vtricml}.
Let $\mathcal{E}_{hyp}$ be the event where the hypothesis of Thm.~\ref{th:regret_Private_UCRL-VTR} hold. Hence, $\mathbb{P}\left(\mathcal{E}_{hyp}\right) \geq 1-\frac{p}{3}$

Conditioned on the event $\mathcal{E}_{hyp}$,
we will prove that $\bignorm{w_h - \wt{w}_{k,h}}_{\wt{\Lambda}_{k,h}} \leq \beta_k$ simultaneously for all $(k,h) \in [K]\times [H]$, with probability at least $1-\frac{p}{3}$.\\ \\
Fix any $h \in [H]$, 
\[
    \bignorm{w_h - \wt{w}_{k,h}}_{\wt{\Lambda}_{k,h}} = \bignorm{\wt{\Lambda}_{k,h}\left(w_h - \wt{w}_{k,h}\right)}_{\wt{\Lambda}^{-1}_{k,h}}
\]
\[
    \bignorm{w_h - \wt{w}_{k,h}}_{\wt{\Lambda}_{k,h}} = \bignorm{\sum_{1 \leq i < k} X_{i,h}X^\intercal_{i,h}\cdot w_h + B_{k-1,h}w_h + \lambda\cdot w_h - \sum_{1 \leq i < k}X_{i,h}y_{i,h} - g_{k-1,h}}_{\wt{\Lambda}^{-1}_{k,h}}
\]
Let $\eta_{k,h} = y_{k,h} - X^\intercal_{k,h}\theta^\star_h$ and observe that $\eta_{k,h}$ is $H$-subgaussian since $\abs{\eta_{k,h}} \leq H$.\\
Thus,
\[
    \bignorm{w_h - \wt{w}_{k,h}}_{\wt{\Lambda}_{k,h}} = \bignorm{\sum_{1 \leq i < k}X_{i,h}\eta_{i,h} + B_{k-1,h}w_h + \lambda\cdot w_h - g_{k-1,h} }_{\wt{\Lambda}^{-1}_{k,h}}
\]
Since $\left(\sum_{1 \leq i < k}X_{i,h}X^\intercal_{i,h}\right) + \left(\lambda + \Upsilon^{k-1}_{low}\right)\cdot I_{d \times d}\preceq \wt{\Lambda}_{k,h} \preceq \left(\sum_{1 \leq i < k}X_{i,h}X^\intercal_{i,h}\right) + \left(\lambda + \Upsilon^{k-1}_{high}\right)\cdot I_{d \times d}$, we conclude by Clm.~\ref{app:clm.order.PSD} that:
\[
    \bignorm{w_h - \wt{w}_{k,h}}_{\wt{\Lambda}_{k,h}} \leq \bignorm{\sum_{1 \leq i < k} X_{i,h}\eta_{i,h}}_{\wt{\Lambda}^{-1}_{k,h}} + \frac{\left(\lambda + \Upsilon^{k-1}_{high}\right)\cdot C_w}{\sqrt{\lambda + \Upsilon^{k-1}_{low}}} + \frac{C_{k-1}}{\sqrt{\lambda + \Upsilon^{k-1}_{low}}}
\]
\[
    \bignorm{w_h - \wt{w}_{k,h}}_{\wt{\Lambda}_{k,h}} \leq \bignorm{\sum_{1 \leq i < k} X_{i,h}\eta_{i,h}}_{\left(\left(\sum_{1 \leq i < k}X_{i,h}X^\intercal_{i,h}\right) + \lambda\cdot I_{d \times d}\right)^{-1}} + \frac{\left(\lambda + \Upsilon^{k-1}_{high}\right)\cdot C_w}{\sqrt{\lambda + \Upsilon^{k-1}_{low}}} + \frac{C_{k-1}}{\sqrt{\lambda + \Upsilon^{k-1}_{low}}}
\]
We apply the self-normalized bound inequality for vector-valued martingales ~\citep{abbasi-yadkori2011} to the martingale $\eta_{k,h}$ with respect to the filtration $\mathcal{G}_{k,h}$ generated by the random paths $\left\{(s^{k'}_{h'},a^{k'}_{h'},r^{k'}_{h'})\right\}_{(k',h') \leq (k,h)}$ and the noises $\left(g^1_{k',h'}\right)_{(k',h') \leq (k,h)}$, $\left(g^2_{k',h'}\right)_{(k',h') \leq (k,h)}$, $\left(B^1_{k',h'}\right)_{(k',h') \leq (k,h)}$, 
$\left(B^2_{k',h'}\right)_{(k',h') \leq (k,h)}$\footnote{The notation $(k',h') \leq (k,h)$ is lexicographic. It is, $k' < k$ or $k' = k$ and $h' \leq h$.}. 

That is $\mathbb{E}\left[\eta_{k,h}|\mathcal{G}_{k,h}\right] = 0$ and observe that $X_{k,h}$ and $\eta_{k+1,h}$ are $\mathcal{G}_{k,h}$-measurable.\\
We conclude that with probability at least $1 - \frac{p}{3H}$:
\[
    \bignorm{\sum_{1 \leq i < k} X_{i,h}\eta_{i,h}}_{\left(\left(\sum_{1 \leq i < k}X_{i,h}X^\intercal_{i,h}\right) + \lambda\cdot Id_d\right)^{-1}} \leq \sqrt{2H^2\log\left(\frac{3H\cdot\left(1 + KH\right)^{\frac{d}{2}}}{p}\right)}, \forall k \in [K]
\] 
Then,
\[
    \bignorm{w_h - \wt{w}_{k,h}}_{\wt{\Lambda}_{k,h}} \leq \sqrt{2H^2\log\left(\frac{3H\cdot\left(1 + KH\right)^{\frac{d}{2}}}{p}\right)} + \frac{C_w\cdot\left(\lambda + \Upsilon^{k-1}_{high}\right) + C_{k-1}}{\sqrt{\lambda + \Upsilon^{k-1}_{low}}} = \beta_{k-1} \leq \beta_k, \forall k \in [K]
\]
In the union event, holding with probability at least $1-\frac{p}{3}$ conditioned on $\mathcal{E}_{hyp}$, 
$\bignorm{w_h - \wt{w}_{k,h}}_{\wt{\Lambda}_{k,h}} \leq \beta_k$ for any $(k,h) \in [K] \times [H]$ as initially wished. \\
Thus, since $\mathbb{P}\left(\mathcal{E}_{hyp}\right) \geq 1-\frac{p}{3}$, $\bignorm{w_h - \wt{w}_{k,h}}_{\wt{\Lambda}_{k,h}} \leq \beta_k$ for any $(k,h) \in [K] \times [H]$ occurs with probability at least $1-\frac{2p}{3}$.

\subsubsection{Regret analysis}
From now onwards, the proof is exactly the same as in ~\citep[][D.1]{jia2020vtr}, the only changes are that every time we need to bound $\norm{\cdot}_{\wt{\Lambda}^{-1}_{k,h}}$ from above, we will rather bound the term $\norm{\cdot}_{\left(\lambda\cdot Id_d + \sum_{1 \leq i < k} X_{i,h}X^\intercal_{i,h}\right)^{-1}}$. Also, notice that we initialize $\lambda = H^2$ instead of $\lambda = H^2d$ as in ~\citep{jia2020vtr} since a stationary transition probability is considered in their case, where $\bignorm{X_{k,h}}^2 = \bignorm{\phi_{\wt{V}_{k,h+1}}(s^k_h,a^k_h)}^2$ is bounded by $H^2d$, while in the setting we are considering (Asm.~\ref{asm:linear.mixture}), 
$\bignorm{\phi_{\wt{V}_{k,h+1}}(s^k_h,a^k_h)}^2 \leq H^2$ since $\wt{V}_{k,h+1}\left(\cdot\right) \leq H$ by definition. \\
We conclude then that with probability at least $1-p$\footnote{Intersection of the event where the the parameters $w_h$ belong to the confidence sets and the Azuma's inequality bounding the second term of the regret.}:
\[
    R(K) \leq 4\sqrt{H^2dK\beta_K\log\left(1+HK\right)} + \sqrt{2H^3K\log\left(\frac{3}{p}\right)}
\]
\[
    R(K) = \widetilde{O}\left(\left(Hd^{\frac{1}{2}}\beta_K + H^{\frac{3}{2}}\right)K^{\frac{1}{2}}\right)
\]

\subsection{Proof of Cor.~\ref{cor:JDP-UCRL-VTR}}\label{app:Cor-JDP-UCRL-VTR}
\paragraph{Full formulation of Cor.~\ref{cor:JDP-UCRL-VTR}:} 
Fix any privacy level $\epsilon, \delta \in (0,1)$. Set $B^1_{k,h} = 0_{d\times d}$, $g^1_{k,h} = 0_d$ and $B^2_{k,h}$, $g^2_{k,h}$ as the noises associated to the $k$-th prefix of trees containing in each node symmetric matrices/vectors with all of its entries $\mathcal{N}(0,\sigma^2_B)$, where $\sigma_B = \frac{32H^2}{\epsilon}\sqrt{2HK_0\log\left(\frac{8H}{\delta}\right)\log\left(\frac{4}{\delta}\right)\log\left(\frac{16HK_0}{\delta}\right)}$ and $K_0 = \lceil\log_2(K) + 1\rceil$. More precisely, the trees are initialized in a similar way as done in Sec. 4.2 from ~\citet{shariff2018nips}. It is, we instantiate $2H$ trees, each with $K$ leaves ($2K$ nodes) and the first $H$ of them contain a symmetric matrix $Z = (Z' + Z'^\intercal)/\sqrt{2}$ in each node, where $Z' \in \mathbb{R}^{d \times d}$ and each of its entries $(Z')_{i,j}$ is drawn according to i.i.d $\mathcal{N}(0,\sigma^2_B)$. The other $H$ of them will be instantiated with gaussian vectors, each of them drawn from $\mathcal{N}(0,\sigma^2_B\cdot I_{d\times d})$.
Let $\Upsilon^J_{\frac{p}{6KH}} = \sigma_B\sqrt{K_0}\left(4\sqrt{d} + 2\log\left(\frac{6KH}{p}\right)\right)$. $B^2_{k,h}$ is the $k$-th prefix of the $h$-th tree of random matrices described above shifted by $2\Upsilon^J_{\frac{p}{6KH}}I_{d \times d}$. Similarly, $g^2_{k,h}$ is the $k$-th prefix of the $h$-th tree of random vectors described above. Finally, choose $\beta_k$ as:

\[
\beta_k = 3(C_w + 1)\sqrt{\lambda + \Upsilon^J_{\frac{p}{6KH}}} + \sqrt{2H^2\log\left(\frac{3H\cdot\left(1+KH\right)^{\frac{d}{2}}}{p}\right)} 
\]
Then, for any $p \in (0,1)$, Alg.~\ref{alg:Perturbed-UCRL-VTR} is $(\epsilon,\delta)$-JDP and with probability at least $1-p$, its regret is bounded by:
\[
    R(K) = \widetilde{O}\left(\left((C_w + 1)\cdot\frac{d^{\frac{3}{4}}H^{\frac{9}{4}}}{\sqrt{\epsilon}} + H^2d\right)\sqrt{K}\right)
\]
where $\widetilde{O}(\cdot)$ hides polylog$\left(\frac{1}{p}, \frac{1}{\delta}, H, K\right)$ factors.

\paragraph{Regret analysis:}
As explained in ~\citet{shariff2018nips, zheng2020locally}, due to Clm.~\ref{app:clm.bounded.eigen}, the eigenvalues of $B_{k,h} = B^2_{k,h}$ are in the interval $\left[\Upsilon^J_{\frac{p}{6KH}}, 3\Upsilon^J_{\frac{p}{6KH}}\right]$ with probability at least $1-\frac{p}{6KH}$ for each $(k,h) \in [K] \times [H]$ since the eigenvalues of the tree-based mechanism are in $\left[-\Upsilon^J_{\frac{p}{6KH}},\Upsilon^J_{\frac{p}{4KH}}\right]$ due that it is the sum of at most $K_0$ matrices, each of them with gaussian entries $\mathcal{N}(0, \sigma^2_B)$. By a similar argument, due to Clm.~\ref{app:clm.bound.gauss.vector} for each $(k,h) \in [K]\times [H]$, $g_{k,h} = g^2_{k,h}$ is in the interval $\left[-C^J_{\frac{p}{6KH}}, C^J_{\frac{p}{6KH}}\right]$ with probability at least $1-\frac{p}{6KH}$, where $C^J_{\frac{p}{6KH}} = \sigma_B\sqrt{K_0}\left(\sqrt{d} + 2\sqrt{\log\left(\frac{6KH}{p}\right)}\right)$. \\
In the union event, with probability at least $1-\frac{p}{3}$, the eigenvalues of $B_{k,h}$ are in the interval $\left[\Upsilon^k_{low}, \Upsilon^k_{high}\right]$ with $\Upsilon^k_{low} = \Upsilon^J_{\frac{p}{6KH}}$, $\Upsilon^k_{high} = 3\Upsilon^J_{\frac{p}{6KH}}$ for any $(k,h) \in [K] \times [K]$. Also, $\bignorm{g_{k,h}} \leq C_k$, where $C_k = C^J_{\frac{p}{6KH}}$ for any $(k,h) \in [K] \times [H]$. 

The choice of $\beta_k = 3(C_w + 1)\sqrt{\lambda + \Upsilon^J_{\frac{p}{6KH}}} + \sqrt{2H^2\log\left(\frac{3H\cdot\left(1+KH\right)^{\frac{d}{2}}}{p}\right)}$ is a constant sequence (independent of $k$) and hence, non-decreasing satisfying that $\beta_k \geq \frac{\left(\lambda + 3\Upsilon^J_{\frac{p}{6KH}}\right)C_w + C^J_{\frac{p}{6KH}}}{\sqrt{\lambda + \Upsilon^J_{\frac{p}{6KH}}}} + \sqrt{2H^2\log\left(\frac{3H\left(1+KH\right)^{\frac{d}{2}}}{p}\right)} \geq \frac{(\lambda+\Upsilon^k_{up})C_w + C_k}{\sqrt{\lambda + \Upsilon^k_{low}}} + \sqrt{2H^2\log\left(\frac{3H\left(1+KH\right)^{\frac{d}{2}}}{p}\right)}$ since $\Upsilon^J_{\frac{p}{6KH}} \geq C^J_{\frac{p}{6KH}}$. The regret bound just follows from Thm.~\ref{th:regret_Private_UCRL-VTR} by replacing $\beta_k = 3(C_w + 1)\sqrt{\lambda + \Upsilon^J_{\frac{p}{6KH}}} + \sqrt{2H^2\log\left(\frac{3H\cdot\left(1+KH\right)^{\frac{d}{2}}}{p}\right)} = \wt{O}\left((C_w + 1)d^{\frac{1}{4}}H^{\frac{5}{4}}\frac{1}{\sqrt{\epsilon}} + H\sqrt{d}\right)$.
\paragraph{Privacy analysis:} First, we show that the mechanism $(\wt{\Lambda}_{k,h}, \wt{u}_{k,h})_{(k,h) \in [K] \times [H]}$ is $(\epsilon,\delta)$-DP.\\
Notice that proving that the mechanisms $(\wt{\Lambda}_{k,h})_{k \in [K]}$ and $(\wt{u}_{k,h})_{k \in [K]}$ are $\left(\frac{\epsilon}{2\sqrt{8H\log\left(\frac{4}{\delta}\right)}}, \frac{\delta}{4H}\right)$-DP for each $h \in [H]$
imply by advanced composition (see App.~\ref{app:advanced.comp}) that $(\wt{\Lambda}_{k,h})_{(k,h) \in [K] \times [H]}$ and $(\wt{u}_{k,h})_{(k,h) \in [K] \times [H]}$ are individually $\left(\frac{\epsilon}{2}, \frac{\delta}{2}\right)$-DP, a final application of simple composition (see App.~\ref{app:simple.comp}) let us conclude the initial claim. 

Then, fix $h \in [H]$ and remark that the Frobenius norm $\bignorm{X_{i,h}X^\intercal_{i,h}}_F \leq H^2$ implying that the $L_2$-sensitivity of each of these terms is bounded by $2H^2$. We follow the same reasoning as in ~\citet{shariff2018nips}, the tree-based method ensures that the released matrix will be the private design matrix $\lambda\cdot Id + \sum_{1 \leq i < k} X_{i,h}X^\intercal_{i,h}$, plus the deterministic matrix $2\Upsilon^J_{\frac{p}{6KH}}I_{d \times d}$, plus the sum of at most $K_0$ matrices, each of them with gaussian entries $\mathcal{N}(0, \sigma^2_B)$. 

Hence, in order to ensure $\left(\frac{\epsilon}{2\sqrt{8H\log\left(\frac{4}{\delta}\right)}}, \frac{\delta}{4H}\right)$-DP, we need  $\sigma_B \geq 2H^2\cdot \frac{16\sqrt{HK_0\log\left(\frac{8H}{\delta}\right)\log\left( \frac{4}{\delta}\right)}}{\epsilon}\sqrt{2\log\left(\frac{16HK_0}{\delta}\right)} = \frac{32H^2}{\epsilon}\sqrt{2HK_0\log\left(\frac{8H}{\delta}\right)\log\left(\frac{4}{\delta}\right)\log\left(\frac{16HK_0}{\delta}\right)}$ due to the design of gaussian mechanisms (see App.~\ref{app:gaussian.mech}) and remarking that the sum of at most $K_0$ nodes in the tree of noises can be interpreted as the sum of at most $K_0$ private mechanisms, each of them ensuring $\left( \frac{\epsilon}{16\sqrt{HK_0\log\left(\frac{8H}{\delta}\right)\log\left(\frac{4}{\delta}\right)}}, \frac{\delta}{8HK_0}\right)$-DP and guaranteeing privacy of $k$-th prefix sum of the tree by advanced composition.

Exactly the same argument allows to conclude that $(\wt{u}_{k,h})_{k \in [K]}$ is $\left(\frac{\epsilon}{2\sqrt{8H\log\left(\frac{4}{\delta}\right)}}, \frac{\delta}{4H}\right)$-DP for each $h \in [H]$. Again, the $L_2$-sensitivity of $X_{i,h}y_{i,h}$ is bounded by $2H^2$. This time the released private quantities are the sum of $\sum_{1 \leq i < k} X_{i,h}y_{i,h}$, plus the sum of at most $K_0$ matrices, each of them with gaussian entries $\mathcal{N}(0, \sigma^2_B)$. Due to the design of gaussian mechanisms, we need $\sigma_B \geq 2H^2\cdot \frac{16\sqrt{HK_0\log\left(\frac{8H}{\delta}\right)\log\left( \frac{4}{\delta}\right)}}{\epsilon}\sqrt{2\log\left(\frac{16HK_0}{\delta}\right)} = \frac{32H^2}{\epsilon}\sqrt{2HK_0\log\left(\frac{8H}{\delta}\right)\log\left(\frac{4}{\delta}\right)\log\left(\frac{16HK_0}{\delta}\right)}$, which is again true due to the choice of $\sigma_B$.

As stated at the beginning, we conclude that $(\wt{\Lambda}_{k,h}, \wt{u}_{k,h})_{(k,h) \in [K] \times [H]}$ is $(\epsilon,\delta)$-DP. By the post-processing property (see App.~\ref{app:post.processing}), it follows that $(\wt{\Lambda}_{k,h}, \wt{w}_{k,h})_{(k,h) \in [K] \times [H]}$ is $(\epsilon,\delta)$-DP since $\wt{w}_{k,h} = \wt{\Lambda}^{-1}_{k,h}\cdot \wt{u}_{k,h}$. A similar argument to the one in ~\citet{vietri2020privaterl} allows to conclude that Alg.~\ref{alg:Perturbed-UCRL-VTR} is $(\epsilon,\delta)$-JDP. Notice that $\wt{Q}_{k,h}$ is a function of the mechanism $(\wt{\Lambda}_{k,h}, \wt{w}_{k,h})_{k,h}$. In particular, the policy $\pi_k$ is also a function of the same mechanism and user's private data (trajectories), therefore, by the Billboard Lemma (see App.~\ref{app:billboard.lemma}), we conclude that Alg.~\ref{alg:Perturbed-UCRL-VTR} is $(\epsilon,\delta)$-JDP.
\subsection{Proof of Cor.~\ref{cor:LDP-UCRL-VTR}}\label{app:Cor-LDP-UCRL-VTR}
\paragraph{Regret analysis:} Due to Clm.~\ref{app:clm.bounded.eigen}, the eigenvalues of $B_{k,h} = \left(\sum_{1 \leq i \leq k}B^1_{i,h}\right) + 2\Upsilon^L_{\frac{p}{4KH}}$ are in the interval $\left[\Upsilon^L_{\frac{p}{6KH}}, 3\Upsilon^L_{\frac{p}{6KH}}\right]$ with probability at least $1-\frac{p}{6KH}$ for each $(k,h) \in [K] \times [H]$ since each of the matrices $B^1_{k,h}$ has gaussian entries $\mathcal{N}(0,\sigma^2_B)$. Similarly, due to Clm.~\ref{app:clm.bound.gauss.vector}, for each $(k,h) \in [K] \times [H]$, $g_{i,h} = \left( \sum_{1 \leq i \leq k} g^1_{i,h}\right)$ is in the interval $\left[-C^L_{\frac{p}{6KH}}, C^L_{\frac{p}{6KH}}\right]$ with probability at least $1-\frac{p}{6KH}$, where $C^L_{\frac{p}{6KH}} = \sigma_B\sqrt{K}\left( \sqrt{d} + 2\sqrt{\log\left(\frac{6KH}{p}\right)}\right)$. 

In the union event, with probability at least $1-\frac{p}{3}$, the eigenvalues of $B_{k,h}$ are in the interval $\left[\Upsilon^k_{low}, \Upsilon^k_{up}\right]$ with $\Upsilon^k_{low} = \Upsilon^L_{\frac{p}{4KH}}$, $\Upsilon^k_{high} = 3\Upsilon^L_{\frac{p}{4KH}}$ $\forall (k,h) \in [K] \times [H]$. Also, $\bignorm{g_{k,h}} \leq C_k$, where $C_k = C^L_{\frac{p}{6KH}}$ $\forall (k,h) \in [K] \times [H]$. 

The choice of $\beta_k = 3(C_w + 1)\sqrt{\lambda + \Upsilon^L_{\frac{p}{6KH}}} + \sqrt{2H^2\log\left(\frac{3H\cdot\left(1+KH\right)^{\frac{d}{2}}}{p}\right)}$ is a constant sequence (independent of $k$) and hence, non-decreasing satisfying that 
\begin{align*}
    \beta_k 
    &\geq \frac{\left(\lambda + 3\Upsilon^L_{\frac{p}{6KH}}\right)C_w + C^L_{\frac{p}{6KH}}}{\sqrt{\lambda + \Upsilon^L_{\frac{p}{6KH}}}} + \sqrt{2H^2\log\left(\frac{3H\left(1+KH\right)^{\frac{d}{2}}}{p}\right)}\\
    &\geq \frac{(\lambda+\Upsilon^k_{high})C_w + C_k}{\sqrt{\lambda + \Upsilon^k_{low}}} + \sqrt{2H^2\log\left(\frac{3H\left(1+KH\right)^{\frac{d}{2}}}{p}\right)}
\end{align*}
since $\Upsilon^L_{\frac{p}{6KH}} \geq C^L_{\frac{p}{6KH}}$. The regret bound just follows from Thm.~\ref{th:regret_Private_UCRL-VTR} by replacing $\beta_k = 3(C_w + 1)\sqrt{\lambda + \Upsilon^L_{\frac{p}{6KH}}} + \sqrt{2H^2\log\left(\frac{3H\cdot\left(1+KH\right)^{\frac{d}{2}}}{p}\right)} = \wt{O}\left((C_w + 1)d^{\frac{1}{4}}H^{\frac{3}{2}}K^{\frac{1}{4}}\frac{1}{\sqrt{\epsilon}} + H\sqrt{d}\right)$.
\paragraph{Privacy analysis:} The procedure is similar to the one explained in~\citep{zheng2020locally}, since we aim to ensure that Alg.~\ref{alg:Perturbed-UCRL-VTR} is $(\epsilon,\delta)$-LDP, we need to show that the mechanism $\left(X_{k,h}X^\intercal_{k,h} + B^1_{k,h}, X_{k,h}y_{k,h} + g^1_{k,h}\right)_{h \in [H]}$ is $(\epsilon,\delta)$-LDP for any fixed $k \in [K]$. To that end, we show that each of the mechanisms $\left(X_{k,h}X^\intercal_{k,h} + B^1_{k,h}\right)$ and $(X_{k,h}y_{k,h} + g^1_{k,h})$ is $\left(\frac{\epsilon}{2H}, \frac{\delta}{2H}\right)$-LDP for each $h \in [H]$ and then conclude by simple composition of $H$ independent mechanisms. Finally, we know that gaussian mechanisms (see App.~\ref{app:gaussian.mech}) can ensure local privacy for appropriate variances depending on the sensitivity of the private quantity, in this case, given that $\bignorm{X_{k,h}X^\intercal_{k,h}}_F, \bignorm{X_{k,h}y_{k,h}}_2 \leq H^2$, it is enough with $\sigma_B \geq 2H^2\cdot\frac{2H}{\epsilon}\cdot\sqrt{2\log\left(\frac{4H}{\delta}\right)} = \frac{4H^3}{\epsilon}\sqrt{2\log\left(\frac{4H}{\delta}\right)}$. Given that $\sigma_B$ is precisely chosen in this way, the proof is complete.

\subsection{Pure DP in the UCRL-VTR Framework}\label{app:pureDP.ucrl.vtr}
It is worthwhile to notice that even if all of our results (Cor.~\ref{cor:JDP-UCRL-VTR}, Cor~\ref{cor:LDP-UCRL-VTR}, Thm.~\ref{th:regret_JDP-LSVI-UCB-RarelySwitch}, Cor.~\ref{app:cor.JDP-UCRL-VTRPlus} and Cor.~\ref{app:cor.LDP-UCRL-VTRPlus}) ensure approximate joint/local DP, they can be slighlty modified to also guarantee pure DP. \\
In this section, we will focus on the results presented for the Linear Mixture Setting (Cor.~\ref{cor:JDP-UCRL-VTR} and ~\ref{cor:LDP-UCRL-VTR}) for the ease of comprenhension. 

The regrets in both corollaries are consequences of Thm.~\ref{th:regret_Private_UCRL-VTR} and the main difference is in the value of $\beta_K$.
\begin{itemize}
    \itemsep0.5em 
    \item Cor.~\ref{cor:JDP-UCRL-VTR} (approximate JDP guarantees): $\beta_K = \wt{O}\left((C_w + 1)d^{\frac{1}{4}}H^{\frac{5}{4}}\cdot\frac{1}{\sqrt{\epsilon}} + H\sqrt{d}\right)$
    \item Cor.~\ref{cor:LDP-UCRL-VTR} (approximate LDP guarantees): $\beta_K = \wt{O}\left((C_w + 1)d^{\frac{1}{4}}H^{\frac{3}{2}}K^{\frac{1}{4}}\cdot\frac{1}{\sqrt{\epsilon}} + H\sqrt{d} \right)$
\end{itemize}
As it can be noticed, if we dismiss the second term $H\sqrt{d}$ coming from the self-normalized bound inequality, the gap between both values of $\beta_K$ is $H^{\frac{1}{4}}K^{\frac{1}{4}}$. The term $H^{\frac{1}{4}}$ is proper from the additional difficulty of ensuring $(\epsilon,\delta)$-LDP, where we rely on simple composition of gaussian mechanisms and hence, we need to prove that each of the terms $X_{k,h}X^\intercal_{k,h}, X_{k,h}y_{k,h}$ for each $h \in [H]$ is $(\epsilon/H,\delta/H)$-LDP. On the other hand, proving JDP only requires to show that $X_{k,h}X^\intercal_{k,h}, X_{k,h}y_{k,h}$ are $\approx (\epsilon/\sqrt{H}, \delta/H)$-DP since we are able to apply the advanced composition theorem (available for DP, but not for LDP). 

The $\sqrt{H}$ gap on the level of privacy for $X_{k,h}X^\intercal_{k,h}$ and $X_{k,h}y_{k,h}$ becomes a $\sqrt{H}$ gap in the variance $\sigma_B$ and given that $\beta_K$ scales with $\sqrt{\sigma_B}$, we have the $H^{\frac{1}{4}}$ gap in the results above. 

In the pure DP case, we do not apply the advanced composition theorem even for JDP guarantees since advanced composition theorem always yields approximate DP only. Hence, we need to use the simple composition theorem and as explained before, it increases the variance $\sigma_B$ of the noises injected by $\sqrt{H}$. Remark that it does not affect the reasoning to achieve LDP guarantees since we were already applying the simple composition for gaussian mechanisms.  

The composition of mechanisms alone is not enough ensure pure DP guarantees, we also need to change gaussian noises to laplace noises (see App.~\ref{app:gaussian.mech} and ~\ref{app:laplace.mech}). This change is reflected on the eigenvalues of the random matrix with gaussian (laplace) entries that can be bounded this time only by $\wt{\Theta}(d)$ instead of $\wt{\Theta}(\sqrt{d})$ (see Clms.~\ref{app:clm.bounded.eigen} and ~\ref{app:clm.bound.laplace.eigen}).

Thus, we can conclude that after changing gaussian by laplace noises and increasing the variance by $\approx \sqrt{H}$, the new regret bounds $R(K) = \wt{O}\left(Hd^{\frac{1}{2}}\beta_K + H^{\frac{3}{2}}\right)K^{\frac{1}{2}}$ would be:
\begin{itemize}
    \item Cor.~\ref{cor:JDP-UCRL-VTR} (pure JDP guarantees): $\beta_K = \wt{O}\left(\max\left\{H\sqrt{d}, (C_w+1)\sqrt{\Upsilon^J_{\frac{p}{6KHd}}} \right\}\right)$, since $\Upsilon^J_{\frac{p}{6KHd}}$ is the upper bound of the absolute value of the eigenvalues of a random matrix with laplace entries, we have $\Upsilon^J_{\frac{p}{6KHd}} = \wt{O}\left(d\sigma_B\right)$, where $\sigma_B$ is $\approx \sqrt{H}$ times its counterpart in Cor.~\ref{cor:JDP-UCRL-VTR}. It makes $\beta_K = \wt{O}\left(H\sqrt{d} + (C_w+1)\cdot d^{\frac{1}{2}}H^\frac{3}{2}/\sqrt{\epsilon}\right)$ and the corresponding regret $R(K) = \wt{O}\left(\left(H^2d + (C_w+1)H^{\frac{5}{2}}d/\sqrt{\epsilon} \right)\sqrt{K}\right)$.
    
    \item Cor.~\ref{cor:LDP-UCRL-VTR} (pure LDP guarantees): $\beta_K = \wt{O}\left(\max\left\{H\sqrt{d}, (C_w+1)\sqrt{\Upsilon^L_{\frac{p}{6KHd}}} \right\}\right)$ and we have $\Upsilon^L_{\frac{p}{6KHd}} = \wt{O}\left(d\sigma_B\sqrt{K}\right)$, where $\sigma_B$ is of the same order as the one in Cor.~\ref{cor:LDP-UCRL-VTR}, since we did not use the advanced composition theorem. Hence, $\beta_K = \wt{O}\left(H\sqrt{d} + (C_w+1)\cdot d^{\frac{1}{2}}H^{\frac{3}{2}}K^{\frac{1}{4}}/\sqrt{\epsilon}\right)$ and $R(K) = \wt{O}\left(H^2d\sqrt{K} + (C_w + 1)H^{\frac{5}{2}}dK^{\frac{3}{4}}/\sqrt{\epsilon}\right)$.
\end{itemize}

\section{Analysis of privacy-preserving UCRL-VTR+}\label{app:vtrplus.framework}

\begin{algorithm*}
\caption{Privacy-Preserving UCRL-VTR+}
\label{alg:Perturbed-UCRL-VTRPlus}
\SetKwInOut{KwIn}{Input}
\SetKwInOut{KwOut}{Output}

\KwIn{episodes $K$, horizon $H$, ambient dimension $d$, privacy parameters $\epsilon,\delta$, failure probability $p$, bound $C_w$}
Initialize $\lambda = 1$, and $\forall h \in [H]$, $\wh{\Lambda}_{1,h} = \widetilde{\Lambda}_{1,h} = \lambda\cdot I_{d \times d}$, $\wh{u}_{1,h} = \widetilde{u}_{1,h} = 0_d$, $\wh{w}_{1,h} = \widetilde{w}_{1,h} = 0_d$\\
Initialize $\wh{\beta}_k, \check{\beta}_k$ and $\widetilde{\beta}_k$ satisfying (\ref{eq:beta_ineq_UCRL-VTRPlus}) $\forall k \in [K]$, e.g. as defined in (\ref{eq:beta_JDP-UCRL-VTRPlus}) or (\ref{eq:beta_LDP-UCRL-VTRPlus}) for JDP or LDP guarantees \\
\For{$k=1, \dots, K$}{
    \tcp{User's side}
    Receive $\wh{w}_{k,h}, \widetilde{w}_{k,h}, \wh{\Lambda}_{k,h}$ and $\widetilde{\Lambda}_{k,h}$\\
    \For{$h = 1, \dots, H$}{
        Observe $s^k_h$\\
        Choose action $a^k_h = \argmax_{a \in A} \wh{Q}_{k,h}(s^k_h,a)$\\ 
        
        $\wh{Q}_{k,h}(\cdot, \cdot) = \min\left\{H, r_h(\cdot,\cdot) + \langle\phi_{\wh{V}_{k,h+1}}(\cdot,\cdot),\wh{w}_{k,h}\rangle + \wh{\beta}_k\bignorm{\phi_{\wh{V}_{k,h+1}}(\cdot,\cdot)}_{(\wh{\Lambda}_{k,h})^{-1}}\right\}$\\
        $\wh{V}_{k,h}(\cdot,\cdot) = \max_{a \in A} \wh{Q}_{k,h}(\cdot,a)$
    }
    \For{$h = 1, \dots, H$}{
        $\bar{\mathbb{V}}_{k,h}\wh{V}_{k,h+1}(s^k_h,a^k_h) = \left[\Big\langle \phi_{\wh{V}^2_{k,h+1}}(s^k_h,a^k_h), \wt{w}_{k,h}\Big\rangle\right]_{[0,H^2]} - \left[ \Big\langle \phi_{\wh{V}_{k,h+1}}(s^k_h,a^k_h), \wh{w}_{k,h}\Big\rangle\right]^2_{[0,H]}$ \\
        $E_{k,h} = \min\left\{H^2, 2H\check{\beta}_k\bignorm{\left( \wh{\Lambda}_{k,h}\right)^{-\frac{1}{2}}\cdot\phi_{\wh{V}_{k,h+1}}(s^k_h,a^k_h)}_2\right\} + \min\left\{H^2, \wt{\beta}_k\bignorm{\left( \wt{\Lambda}_{k,h}\right)^{-\frac{1}{2}}\cdot\phi_{\wt{V}_{k,h+1}}(s^k_h,a^k_h)}_2\right\}$\\
        $\bar{\sigma}_{k,h} = \left(\max\left\{H^2/d, \left[\bar{\mathbb{V}}_{k,h}\wh{V}_{k,h+1}(s^k_h,a^k_h) + E_{k,h}\right] \right\}\right)^{1/2}$
    }
    Send $\left\{\bar{\sigma}^{-2}_{k,h}\cdot\phi_{\wh{V}_{k,h+1}}(s^k_h,a^k_h)\phi^\intercal_{\wh{V}_{k,h+1}}(s^k_h,a^k_h) + B^1_{k,h}\right\}_{h \in [H]}$,$\left\{\bar{\sigma}^{-2}_{k,h}\phi_{\wh{V}_{k,h+1}}(s^k_h,a^k_h)\wh{V}_{k,h+1}(s^k_{h+1}) + f^1_{k,h} \right\}_{h \in [H]}$,\\
    $\left\{\phi_{\wh{V}^2_{k,h+1}}(s^k_h,a^k_h)\phi^\intercal_{\wh{V}^2_{k,h+1}}(s^k_h,a^k_h) + B^2_{k,h} \right\}_{h \in [H]}$,
    $\left\{\phi_{\wh{V}^2_{k,h+1}}(s^k_h,a^k_h)\wh{V}^2_{k,h+1}(s^k_{h+1}) + g^1_{k,h}\right\}_{h \in [H]}$\\
    \tcp{Server's side (i.e., algorithm)}
    Update design matrix and target $\forall h \in [H]$\\
    \For{$h = 1, \dots, H$}{
        $D^1_{k,h} = \left(\sum_{1 \leq i \leq k} B^1_{i,h}\right) + B^3_{k,h}$\\
        $D^2_{k,h} = \left(\sum_{1 \leq i \leq k} B^2_{i,h}\right) + B^4_{k,h}$\\
        $f_{k,h} = \left(\sum_{1 \leq i \leq k} f^1_{i,h}\right) + f^2_{k,h}$\\
        $g_{k,h} = \left(\sum_{1 \leq i \leq k} g^1_{i,h}\right) + g^2_{k,h}$\\
        $\wh{\Lambda}_{k+1,h} = \lambda\cdot Id + \sum_{1 \leq i \leq k} \bar{\sigma}^{-2}_{i,h}\phi_{\wh{V}_{i,h+1}}(s^i_h,a^i_h)\phi^\intercal_{\wh{V}_{i,h+1}}(s^i_h,a^i_h) + D^1_{k,h}$\\
        $\wt{\Lambda}_{k+1,h} = \lambda\cdot Id + \sum_{1 \leq i \leq k} \phi_{\wh{V}^2_{i,h+1}}(s^i_h,a^i_h)\phi^\intercal_{\wh{V}^2_{i,h+1}}(s^i_h,a^i_h) + D^2_{k,h}$\\
        $\wh{u}_{k+1,h} = \sum_{1 \leq i \leq k} \bar{\sigma}^{-2}_{i,h}\phi_{\wh{V}_{i,h+1}}(s^i_h,a^i_h)\wh{V}_{i,h+1}(s^i_{h+1}) + f_{k,h}$\\
        $\wt{u}_{k+1,h} = \sum_{1 \leq i \leq k} \phi_{\wh{V}^2_{i,h+1}}(s^i_h,a^i_h)\wh{V}^2_{i,h+1}(s^i_{h+1}) + g_{k,h}$ \\
        $\wh{w}_{k+1,h} = \left(\wh{\Lambda}_{k+1,h}\right)^{-1}\cdot\wh{u}_{k+1,h}$\\
        $\wt{w}_{k+1,h} = \left( \widetilde{\Lambda}_{k+1,h}\right)^{-1}\cdot\widetilde{u}_{k+1,h}$
    }
    
}
\end{algorithm*}

In this section, the Privacy-Preserving UCRL-VTR+ algorithm (Alg.~\ref{alg:Perturbed-UCRL-VTRPlus}) is presented, a framework inspired in UCRL-VTR+ also taking the second moment covariance matrix and response vector of the value function into account. 



\begin{theorem}[Regret bound of Alg.~\ref{alg:Perturbed-UCRL-VTRPlus}]\label{app:th.regret_Private_UCRL-VTRPlus}
Assume that the following hypothesis hold with probability at least $1-\frac{p}{6}$:
\begin{enumerate}[noitemsep,topsep=0pt,parsep=0pt,partopsep=0pt,leftmargin=.4cm]
    \item There exist two sequences of non-negative real numbers  $(\Upsilon^k_{1,low})_k$ and $(\Upsilon^k_{1,high})_k$ such that all the eigenvalues of $D^1_{k,h}$ belong to $[\Upsilon^k_{1,low}, \Upsilon^k_{1,high}]$ $\forall h \in [H]$.
    \item There exist two sequences of non-negative real numbers  $(\Upsilon^k_{2,low})_k$ and $(\Upsilon^k_{2,high})_k$ such that all the eigenvalues of $D^2_{k,h}$ belong to $[\Upsilon^k_{2,low}, \Upsilon^k_{2,high}]$ $\forall h \in [H]$.
    \item There exists a sequence $(C_{1,k})_k$ such that $\bignorm{f_{k,h}}_2 \leq C_{1,k}$ $\forall h \in [H]$.
    \item There exists a sequence $(C_{2,k})_k$ such that $\bignorm{g_{k,h}}_2 \leq C_{2,k}$ $\forall h \in [H]$.
    \item $\wh{\beta}_k, \widetilde{\beta}_k, \check{\beta}_k$ are non-decreasing sequences satisfying the following inequalities:
    \begin{equation}\label{eq:beta_ineq_UCRL-VTRPlus}
    \begin{aligned}
        \check{\beta}_k \geq \frac{C_w\left(\lambda + \Upsilon^k_{1,high}\right) + C_{1,k}}{\sqrt{\lambda + \Upsilon^k_{1,low}}} + 8d\sqrt{\log\left(1 + \frac{K}{\lambda} \right)\log\left(\frac{24k^2H}{p}\right)} + 4\sqrt{d}\log\left( \frac{24k^2H}{p}\right)
    \\
        \wh{\beta}_k \geq \frac{C_w\left(\lambda + \Upsilon^k_{1,high}\right) + C_{1,k}}{\sqrt{\lambda + \Upsilon^k_{1,low}}} + 8\sqrt{d\log\left(1 + \frac{K}{\lambda} \right)\log\left(\frac{24k^2H}{p}\right)} + 4\sqrt{d}\log\left( \frac{24k^2H}{p}\right)
    \\
        \widetilde{\beta}_k \geq \frac{C_w\left(\lambda + \Upsilon^k_{2,high}\right) + C_{2,k}}{\sqrt{\lambda + \Upsilon^k_{2,low}}} + 8\sqrt{dH^4\log\left(1+\frac{KH^4}{d\lambda}\right)\log\left(\frac{24k^2H}{p} \right)} + 4H^2\log\left(\frac{24k^2H}{p}\right)
    \end{aligned}
    \end{equation}
\end{enumerate}
Then, for any $p \in (0,1)$, with probability at least $1-p$, the regret of Alg.~\ref{alg:Perturbed-UCRL-VTRPlus} is bounded as follows:
\[
    R(K) = \widetilde{O}\left(H^{\frac{1}{2}}d^{\frac{1}{2}}\wh{\beta}_K\left(\max\left\{H^{\frac{5}{2}}d^{\frac{1}{2}}\wh{\beta}_K, H^{\frac{7}{4}}K^{\frac{1}{4}}, \frac{H^{\frac{3}{2}}K^{\frac{1}{2}}}{d^{\frac{1}{2}}}, HK^{\frac{1}{2}}, H^{\frac{1}{2}}d^{\frac{1}{4}}\sqrt{\widetilde{\beta}_K}K^{\frac{1}{4}}, H^{\frac{3}{2}}d^{\frac{1}{4}}\sqrt{\check{\beta}_K}K^{\frac{1}{4}}\right\}\right) + H^{\frac{3}{2}}K^{\frac{1}{2}}\right)
\]
where the $\widetilde{O}(\cdot)$ notation hides $\log\left(\frac{1}{p}\right), \log(K)$ and $\log(H)$ factors.
\end{theorem}

\begin{corollary}[JDP-UCRL-VTR+]\label{app:cor.JDP-UCRL-VTRPlus} Fix any privacy level $\epsilon,\delta \in (0,1)$. Setting $B^1_{k,h} = B^2_{k,h} = 0_{d\times d}$, $f^1_{k,h} = g^1_{k,h} = 0_d$ for each $(k,h) \in [K] \times [H]$, while $B^3_{k,h}, B^4_{k,h}, f^2_{k,h}$ and $g^2_{k,h}$ will come from 4 trees (initialized as done in the tree-based method). More precisely, we follow a similar approach to the one of Section 4.2 in ~\citet{shariff2018nips}. Indeed, we instantiate 4H trees, each with K leaves (2K nodes), the first $H$ of them contain symmetric matrices $Z_1 = (Z'_1 + Z'^\intercal_1)/\sqrt{2}$, where $Z'_1 \in \mathbb{R}^{d\times d}$ and each of its entries follow a gaussian distribution $\mathcal{N}(0, \sigma^2_{B,1})$ with $\sigma_{B,1} = \frac{64d}{\epsilon}\sqrt{2HK_0\log\left(\frac{16H}{\delta}\right)\log\left(\frac{8}{\delta}\right)\log\left(\frac{32HK_0}{\delta}\right)}$and $K_0 = \lceil \log_2(K) + 1 \rceil$, the next $H$ of them contain a $d$-dimensional vector following a gaussian distribution $\mathcal{N}(0, \sigma^2_{B,1}I_{d \times d})$, the next H of them contain symmetric matrices $Z_2 = (Z'_2 + Z'^\intercal_2)/\sqrt{2}$, where $Z'_2 \in \mathbb{R}^{d\times d}$ and each of its entries follow a gaussian distribution $\mathcal{N}(0, \sigma^2_{B,2})$ with $\sigma_{B,2} = \frac{64H^4}{\epsilon}\sqrt{2HK_0\log\left(\frac{16H}{\delta}\right)\log\left(\frac{8}{\delta}\right)\log\left(\frac{32HK_0}{\delta}\right)}$, finally the last $H$ trees contain a $d$-dimensional vector following a gaussian distribution $\mathcal{N}(0, \sigma^2_{B,2}I_{d \times d})$.\\
We then set $B^3_{k,h}$ as the sum of the $k$-th prefix in the $h$-th tree (of those instantiated with gaussian entries $\mathcal{N}(0, \sigma^2_{B,1})$) shifted by $2\Upsilon^{J,1}_{\frac{p}{24KH}}I_{d\times d}$, where $\Upsilon^{J,1}_{\frac{p}{24KH}} = \sigma_{B,1}\sqrt{K_0}\left(4\sqrt{d} + 2\log\left(\frac{24KH}{p}\right)\right)$. In the same way, $B^4_{k,h}$ is set to the sum of the $k$-th prefix in the $h$-th tree (of those instantiated with gaussian entries $\mathcal{N}(0, \sigma^2_{B,2})$) shifted by $2\Upsilon^{J,2}_{\frac{p}{24KH}}I_{d\times d}$, where $\Upsilon^{J,2}_{\frac{p}{24KH}} = \sigma_{B,2}\sqrt{K_0}\left(4\sqrt{d} + 2\log\left(\frac{24KH}{p}\right)\right)$. Finally, we set $f^2_{k,h}$ as the $k$-th prefix in the $h$-th tree (of those instantiated with $\mathcal{N}(0, \sigma^2_{B,1}I_{d \times d})$) and $g^2_{k,h}$ as the $k$-th prefix in the $h$-th tree (of those instantiated with $\mathcal{N}(0, \sigma^2_{B,2}I_{d \times d})$).\\
Then, choosing $\check{\beta}_k, \wh{\beta}_k, \widetilde{\beta}_k$ as follows: 
\begin{equation}\label{eq:beta_JDP-UCRL-VTRPlus}
\begin{aligned}
    \check{\beta}_k = 3(C_w+1)\sqrt{\lambda + \Upsilon^{J,1}_{\frac{p}{24KH}}} + 8d\sqrt{\log\left(1 + \frac{K}{\lambda} \right)\log\left(\frac{24k^2H}{p}\right)} + 4\sqrt{d}\log\left( \frac{24k^2H}{p}\right)\\
    \wh{\beta}_k = 3(C_w+1)\sqrt{\lambda + \Upsilon^{J,1}_{\frac{p}{24KH}}} + 8\sqrt{d\log\left(1 + \frac{K}{\lambda} \right)\log\left(\frac{24k^2H}{p}\right)} + 4\sqrt{d}\log\left( \frac{24k^2H}{p}\right)\\
    \widetilde{\beta}_k = 3(C_w+1)\sqrt{\lambda + \Upsilon^{J,2}_{\frac{p}{24KH}}} + 8\sqrt{dH^4\log\left(1+\frac{KH^4}{d\lambda}\right)\log\left(\frac{24k^2H}{p} \right)} + 4H^2\log\left(\frac{24k^2H}{p}\right)
\end{aligned}
\end{equation}
Alg.~\ref{alg:Perturbed-UCRL-VTRPlus} is $(\epsilon,\delta)$-JDP and with probability at least $1-p$, its regret is bounded as follows:
\[
    R(K) = \widetilde{O}\left((C_w+1)^2\frac{H^{\frac{7}{2}}d^{\frac{5}{2}}}{\epsilon} + \frac{(C_w+1)}{\sqrt{\epsilon}}\cdot\left(\left(H^{\frac{5}{2}}d^{\frac{5}{4}} + H^{\frac{9}{4}}d^2 \right)K^{\frac{1}{4}} + \left(H^{\frac{7}{4}}d^{\frac{5}{4}} + H^{\frac{9}{4}}d^{\frac{3}{4}} \right)K^{\frac{1}{2}} \right)\right)
\]
where the $\widetilde{O}(\cdot)$ notation hides polylog$\left(\frac{1}{p}, \frac{1}{\delta}, H, K\right)$ factors. 
\end{corollary}

\begin{corollary}[LDP-UCRL-VTR+]\label{app:cor.LDP-UCRL-VTRPlus} Fix any privacy level $\epsilon, \delta \in (0,1)$. Setting $B^3_{k,h} = 2\Upsilon^{L,1}_{\frac{p}{24KH}}\cdot I_{d \times d}, B^4_{k,h} = 2\Upsilon^{L,2}_{\frac{p}{24KH}}\cdot I_{d \times d}$ for each $(k,h) \in [K] \times [H]$ and $B^1_{k,h}, B^2_{k,h}, f^1_{k,h}$, $g^1_{k,h}$ will be gaussian mechanism (see App.~\ref{app:gaussian.mech}), where $\sigma_{B,1} = \frac{8dH}{\epsilon}\sqrt{2\log\left(\frac{8H}{\delta}\right)}, \Upsilon^{L,1}_{\frac{p}{24KH}} = \sigma_{B,1}\sqrt{K}\left(4\sqrt{d} + 2\log\left(\frac{24KH}{p}\right)\right)$, $\sigma_{B,2} = \frac{8H^5}{\epsilon}\sqrt{2\log\left(\frac{8H}{\delta}\right)}$ and $\Upsilon^{L,2}_{\frac{p}{24KH}} = \sigma_{B,2}\sqrt{K}\left(4\sqrt{d} + 2\log\left(\frac{24KH}{p}\right)\right)$. Indeed, $B^1_{k,h}$ is a $d\times d$ matrix with each of its entries drawn from $\mathcal{N}(0, \sigma^2_{B,1})$. In the same way, $B^2_{k,h}$ is a $d\times d$ matrix with each of its entries drawn from $\mathcal{N}(0, \sigma^2_{B,2})$. Similarly, $f^1_{k,h}$ and $g^1_{k,h}$ are drawn from $\mathcal{N}(0, \sigma^2_{B,1}I_{d\times d})$ and $\mathcal{N}(0, \sigma^2_{B,2}I_{d \times d})$ respectively.\\
Then, choosing $\check{\beta}_k, \wh{\beta}_k, \widetilde{\beta}_k$ as follows:
\begin{equation}\label{eq:beta_LDP-UCRL-VTRPlus}
\begin{aligned}
    \check{\beta}_k = 3(C_w+1)\sqrt{\lambda + \Upsilon^{L,1}_{\frac{p}{24KH}}} + 8d\sqrt{\log\left(1 + \frac{K}{\lambda} \right)\log\left(\frac{24k^2H}{p}\right)} + 4\sqrt{d}\log\left( \frac{24k^2H}{p}\right)\\
    \wh{\beta}_k = 3(C_w+1)\sqrt{\lambda + \Upsilon^{L,1}_{\frac{p}{24KH}}} + 8\sqrt{d\log\left(1 + \frac{K}{\lambda} \right)\log\left(\frac{24k^2H}{p}\right)} + 4\sqrt{d}\log\left( \frac{24k^2H}{p}\right)\\
    \widetilde{\beta}_k = 3(C_w+1)\sqrt{\lambda + \Upsilon^{L,2}_{\frac{p}{24KH}}} + 8\sqrt{dH^4\log\left(1+\frac{KH^4}{d\lambda}\right)\log\left(\frac{24k^2H}{p} \right)} + 4H^2\log\left(\frac{24k^2H}{p}\right)
\end{aligned}
\end{equation}
Alg.~\ref{alg:Perturbed-UCRL-VTRPlus} is $(\epsilon,\delta)$-LDP and with probability at least $1-p$, its regret is bounded as follows: 
\[
    R(K) = \widetilde{O}\left(\frac{(C_w+1)^2}{\epsilon}\cdot H^4d^{\frac{5}{2}}K^{\frac{1}{2}} + \frac{(C_w+1)}{\sqrt{\epsilon}}\left( H^{\frac{11}{4}}d^{\frac{5}{4}} + H^{\frac{5}{2}}d^2\right)K^{\frac{1}{2}} + \frac{(C_w+1)}{\sqrt{\epsilon}}\left( H^2d^{\frac{5}{4}} + H^{\frac{5}{2}}d^{\frac{3}{4}}\right)K^{\frac{3}{4}}\right)
\]
where the $\widetilde{O}(\cdot)$ notation hides polylog$\left(\frac{1}{p}, \frac{1}{\delta}, H, K\right)$ factors.
\end{corollary}

\begin{remark}
As pointed out in the case of the UCRL-VTR Framework (Alg.~\ref{alg:Perturbed-UCRL-VTR}), we can also ensure pure DP instead of approximate DP by replacing each gaussian mechanism in Alg.~\ref{alg:Perturbed-UCRL-VTRPlus} by a Laplace mehcanism and increasing the variance roughly by a factor $\sqrt{H}$. On the other hand, the most interesting result here is that the regret bounds obtained for the private version of UCRL-VTR+ are not asymptotically better than the ones of UCRL-VTR. We believe that the main reason is that in difference to the non-private versions of both algorithms, the dominant term in the size of the confidence sets comes from the noise added and it is then of the same order for both of our algorithms (no improvement in the size of the confidence set).
\end{remark}

\subsection{Privacy Preserving UCRL-VTR vs UCRL-VTR+}\label{app:privacy.preserving.ucrl.comparison}
Thms.~\ref{th:regret_Private_UCRL-VTR} and ~\ref{app:th.regret_Private_UCRL-VTRPlus} offer two ways to bound the regret of privacy-preserving algorithms in the Linear-Mixture setting. 

Let's take a step back and consider the non-private versions of Algs.~\ref{alg:Perturbed-UCRL-VTR} and ~\ref{alg:Perturbed-UCRL-VTRPlus}, i.e. all the tuning vectors/matrices equal to 0. 

It is very well known that in this case, UCRL-VTR+ (Alg.~\ref{alg:Perturbed-UCRL-VTRPlus}) manages to improve the regret of UCRL-VTR (Alg.~\ref{alg:Perturbed-UCRL-VTR}) by a factor of $\sqrt{H}$ in the regime $d \geq H$ leveraging the usage of Bernstein-type concentration inequalities instead of Hoeffding inequalities. This improvement can be perceived through our results in the following way:
\begin{itemize}
    \item Thm.~\ref{alg:Perturbed-UCRL-VTR}: $R_{UCRL-VTR}(K) \approx \wt{O}\left(Hd^{\frac{1}{2}}\beta_K\cdot K^{\frac{1}{2}}\right)$
    \item Thm.~\ref{alg:Perturbed-UCRL-VTRPlus}: $R_{UCRL-VTR+}(K) \approx \wt{O}\left( H^{\frac{1}{2}}d^{\frac{1}{2}}\wh{\beta}_K\cdot \left(H+\frac{H^{\frac{3}{2}}}{d^{\frac{1}{2}}}\right)K^{\frac{1}{2}}\right) = \wt{O}\left(H^{\frac{3}{2}}d^{\frac{1}{2}}\wh{\beta}_KK^{\frac{1}{2}}\right)$ in the regime $d \geq H$
\end{itemize}
It suffices to observe that $\beta_K = \wt{\Theta}(H\sqrt{d})$ and $\wh{\beta}_K = \wt{\Theta}(\sqrt{d})$ and hence, a simple comparison of the expressions above already reflects the $\sqrt{H}$ gap between their regrets in favor of UCRL-VTR+. 

Consider now the private case and more particularly, the results provided in Cor.~\ref{cor:JDP-UCRL-VTR}/ Cor.~\ref{app:cor.JDP-UCRL-VTRPlus} for joint DP guarantees and Cor.~\ref{cor:LDP-UCRL-VTR}/Cor.~\ref{app:cor.LDP-UCRL-VTRPlus} for local DP guarantees. Moreover, for ease of simplicity, consider the former case, since the comparison in the latter scenario is fairly similar.

Indeed, in this case, the size of the confidence sets can be decomposed in two terms, one coming from the noise added to ensure privacy, while the other one is inherent to the type of concentration inequality used to build the intervals:
\[
\beta_K = \underbrace{3(C_w + 1)\sqrt{\lambda + \Upsilon^J_{\frac{p}{6KH}}}}_{\text{1st: Due to the noise}} + \underbrace{\sqrt{2H^2\log\left(\frac{3H\cdot(1+KH)^{\frac{d}{2}}}{p}\right)}}_{\text{2nd: Hoeffding's inequality}}
\]
\[
\wh{\beta}_k = \underbrace{3(C_w+1)\sqrt{\lambda + \Upsilon^{J,1}_{\frac{p}{24KH}}}}_{\text{1st: Due to the noise}} + \underbrace{8\sqrt{d\log\left(1 + \frac{K}{\lambda} \right)\log\left(\frac{24k^2H}{p}\right)} + 4\sqrt{d}\log\left( \frac{24k^2H}{p}\right)}_{\text{2nd: Bernstein's inequality}}
\]
Notice that in the non-private case, where $\Upsilon^J_{\frac{p}{6KH}} = 0, \Upsilon^{J,1}_{\frac{p}{24KH}} = 0$, the 2nd term of both expressions are dominant with respect to the 1st one. However, in the private case, the noise produces the reverse effect of making the 1st term (in the past a low-order term) now comparable to (or greater than) the 2nd term.

In the JDP setting, 
\[
\beta_K = \wt{O}\left((C_w + 1)\cdot \sqrt{\sigma_B}\cdot d^{\frac{1}{4}} + H\sqrt{d}\right) = \wt{O}\left((C_w + 1)\cdot d^{\frac{1}{4}}H^{\frac{5}{4}}\cdot\frac{1}{\sqrt{\epsilon}} + H\sqrt{d}\right)
\] 
\[
\wh{\beta}_K = \wt{O}\left((C_w + 1)\cdot\sqrt{\sigma_{B,1}}\cdot d^{\frac{1}{4}} + \sqrt{d}\right) = \wt{O}\left((C_w + 1)\cdot d^{\frac{3}{4}}H^{\frac{1}{4}}\cdot\frac{1}{\sqrt{\epsilon}}\right)
\]
Therefore, the regrets of Algs.~\ref{alg:Perturbed-UCRL-VTR} and ~\ref{alg:Perturbed-UCRL-VTRPlus} will be:
\begin{itemize}
    \item Cor.~\ref{cor:JDP-UCRL-VTR}: $R_{UCRL-VTR}(K) = \wt{O}\left(\left((C_w+1)\cdot d^{\frac{3}{4}}H^{\frac{9}{4}}\cdot\frac{1}{\sqrt{\epsilon}} + H^2d\right)K^{\frac{1}{2}}\right)$
    \item Cor.~\ref{app:cor.JDP-UCRL-VTRPlus}: $R_{UCRL-VTR+}(K) \approx \wt{O}\left(\frac{(C_w+1)}{\sqrt{\epsilon}}\cdot\left(d^{\frac{5}{4}}H^{\frac{7}{4}} + d^{\frac{3}{4}}H^{\frac{9}{4}}\right)K^{\frac{1}{2}}\right)$
\end{itemize}
Noticing that $\frac{(C_w+1)}{\sqrt{\epsilon}}\cdot\left(d^{\frac{5}{4}}H^{\frac{7}{4}} + d^{\frac{3}{4}}H^{\frac{9}{4}}\right) \geq 2H^2d$ due to AM-GM inequality and recalling that $\epsilon < 1$. We can then observe that the regret bound of Cor.~\ref{cor:JDP-UCRL-VTR} is in this case not better than the one of Cor.~\ref{app:cor.JDP-UCRL-VTRPlus}.
\subsection{Proof of Thm.~\ref{app:th.regret_Private_UCRL-VTRPlus}}\label{app:proof.th.regret_Private-UCRL-VTRPlus}
We follow the proof given in ~\citet{zhou2021vtrplus} for the non-private version of UCRL-VTR+. 

\subsubsection{Confidence sets}
The proof proceeds as in ~\citep[][App. C.1]{zhou2021vtrplus} and relies also in Theorem 4.1 (Bernstein inequality for vector-valued martingales) from the same work. The differences on the sizes of the confidence sets, $\check{\beta}_k, \wh{\beta}_k$ and $\widetilde{\beta}_k$, are due to the perturbations $B^1_{k,h}, B^2_{k,h}, B^3_{k,h}, B^4_{k,h}, f^1_{k,h}, f^2_{k,h}, g^1_{k,h}$ and $g^2_{k,h}$ such as it was observed in the proof of Thm.~\ref{th:regret_Private_UCRL-VTR}. 

Similarly, let $\check{\mathcal{C}}_{k,h}, \wh{\mathcal{C}}_{k,h}$ and $\widetilde{\mathcal{C}}_{k,h}$ be the following confidence sets:
\[
    \check{\mathcal{C}}_{k,h} = \left\{w \in \mathbb{R}^d: \bignorm{\wh{\Lambda}^{\frac{1}{2}}_{k,h}\cdot\left(w - \wh{w}_{k,h}\right)}_2 \leq \check{\beta}_k\right\}
\]
\[
    \wh{\mathcal{C}}_{k,h} = \left\{w \in \mathbb{R}^d: \bignorm{\wh{\Lambda}^{\frac{1}{2}}_{k,h}\cdot\left(w - \wh{w}_{k,h}\right)}_2 \leq \wh{\beta}_k\right\}
\]
\[
    \widetilde{\mathcal{C}}_{k,h} = \left\{w \in \mathbb{R}^d: \bignorm{\wt{\Lambda}^{\frac{1}{2}}_{k,h}\cdot\left(w - \wt{w}_{k,h}\right)}_2 \leq \wt{\beta}_k \right\}
\]
\begin{lemma}[Adapted from lemma C.1 in ~\citet{zhou2021vtrplus}]\label{lem:lemaC1_UCRL-VTRPlus} Let $\wh{V}_{k,h+1}, \wh{w}_{k,h}, \wh{\Lambda}_{k,h}, \widetilde{w}_{k,h}, \widetilde{\Lambda}_{k,h}$ be defined in Alg.~\ref{alg:Perturbed-UCRL-VTRPlus}, then we have
\begin{equation*}
\begin{aligned}
    \abs{\mathbb{V}_h\wh{V}_{k,h+1}(s^k_h,a^k_h) - \bar{\mathbb{V}}_{k,h}\wh{V}_{k,h+1}(s^k_h, a^k_h)} \leq \\ \min\left\{H^2,\bignorm{\wt{\Lambda}^{-\frac{1}{2}}_{k,h}\cdot\phi_{\wh{V}^2_{k,h+1}}\left(s^k_h,a^k_h\right)}_2\cdot\bignorm{\wt{\Lambda}^{\frac{1}{2}}_{k,h}\cdot\left( \wt{w}_{k,h} - w_h\right)}_2\right\} + \\ \min\left\{H^2, 2H\bignorm{\wh{\Lambda}^{-\frac{1}{2}}_{k,h}\cdot\phi_{\wh{V}_{k,h+1}}(s^k_h,a^k_h)}_2\cdot\bignorm{\wh{\Lambda}^{\frac{1}{2}}_{k,h}\cdot\left(\wh{w}_{k,h} - w_h\right)}_2\right\}
\end{aligned}
\end{equation*}
\end{lemma}
\begin{proof}
The proof is the same of Lemma C.1 in ~\citet{zhou2021vtrplus}, the only changes are replacing $V_{k,h+1}, \wh{\theta}_{k,h}, \wh{\Sigma}_{k,h}, \wt{\theta}_{k,h}, \wt{\Sigma}_{k,h}$ by $\wh{V}_{k,h+1}, \wh{w}_{k,h}, \wh{\Lambda}_{k,h}, \wt{w}_{k,h}, \wt{\Lambda}_{k,h}$.
\end{proof}
\noindent Let $\mathcal{E}_{hyp}$ be the event where the assumptions of Thm.~\ref{app:th.regret_Private_UCRL-VTRPlus} hold, $\mathbb{P}\left(\mathcal{E}_{hyp}\right) \geq 1-\frac{p}{6}$.

From now onwards, all the results will be proved conditioned on the event $\mathcal{E}_{hyp}$.\\
Fix $h \in [H]$ and apply Thm. 4.1 ~\citet{zhou2021vtrplus} for the following parameters:
\begin{equation*}
\begin{aligned}
x_i = \bar{\sigma}^{-1}_{i,h}\cdot\phi_{\wh{V}_{i,h+1}}(s^i_h,a^i_h),
\eta_i = \bar{\sigma}^{-1}_{i,h}\cdot \wh{V}_{i,h+1}(s^i_{h+1}) - \bar{\sigma}^{-1}_{i,h}\cdot\Big\langle \phi_{\wh{V}_{i,h+1}}(s^i_h,a^i_h), w_h\Big\rangle \\ \mathcal{G}_i = \sigma\left(\left\{s^{k'}_{h'}, a^{k'}_{h'}, r^{k'}_{h'}\right\}_{(k',h') \leq (k,h)} \cup \left(g^1_{k',h'}, g^2_{k',h'}, f^1_{k',h'},f^2_{k',h'}, B^1_{k',h'}, B^3_{k',h'}\right)_{(k',h') \leq (k,h)}\right), \\
\mu^\star = w_h, y_i = \Big\langle \mu^\star, x_i\Big\rangle + \eta_i = \bar{\sigma}^{-1}_{i,h}\cdot \wh{V}_{i,h+1}(s^i_{h+1}), Z_i = \lambda\cdot I_{d \times d} + \sum_{i'=1}^i x_{i'}\cdot x^\intercal_{i'}
\end{aligned}
\end{equation*}
Since $\norm{x_i}_2 \leq \bar{\sigma}^{-1}_{i,h}H \leq \sqrt{d}$, $\abs{\eta_i}_2 \leq \bar{\sigma}^{-1}_{i,h}H \leq \sqrt{d}, \mathbb{E}\left[\eta_i|\mathcal{G}_i\right] = 0, \mathbb{E}\left[\eta^2_i|\mathcal{G}_i\right] \leq d$, it can be concluded that with probability at least $1-\frac{p}{6H}$, $\forall k \leq K$:
 \begin{equation}\label{eq:bernstein_bound1}
    \bignorm{\sum_{i=1}^k x_i\eta_i}_{Z^{-1}_k} \leq 8d\sqrt{\log\left(1+k/\lambda\right)\log\left(24k^2H/p\right)} + 4\sqrt{d}\log\left(\frac{24k^2H}{p}\right)
\end{equation}
Eq.~\ref{eq:bernstein_bound1} is used to prove that $w_h \in \bigcap_{k \in [K]}\check{\mathcal{C}}_{k,h}$ with probability at least $1-p/6H$.
\begin{small}
\[
    \bignorm{\wh{\Lambda}^{\frac{1}{2}}_{k,h}\cdot\left(w_h - \wh{w}_{k,h}\right)}_2 = \bignorm{\wh{\Lambda}_{k,h}\cdot\left(w_h - \wh{w}_{k,h}\right)}_{\wh{\Lambda}^{-1}_{k,h}}
\]
\[
    = \bignorm{\left(\lambda + D^1_{k-1,h} \right)w_h + \sum_{1 \leq i < k}\bar{\sigma}^{-2}_{i,h}\phi_{\wh{V}_{i,h+1}}(s^i_h,a^i_h)\phi^\intercal_{\wh{V}_{i,h+1}}(s^i_h,a^i_h)\cdot w_h - f_{k-1,h} - \sum_{1 \leq i < k} \bar{\sigma}^{-2}_{i,h}\phi_{\wh{V}_{i,h+1}}(s^i_h,a^i_h)\wh{V}_{i,h+1}(s^i_{h+1})}_{\wh{\Lambda}^{-1}_{k,h}}
\]
\[
    = \bignorm{\left(\lambda + D^1_{k-1,h}\right)w_h + \sum_{1 \leq i < k} x_i\eta_i - f_{k-1,h}}_{\wh{\Lambda}^{-1}_{k,h}} \leq \bignorm{\left(\lambda + D^1_{k-1,h}\right)w_h}_{\wh{\Lambda}^{-1}_{k,h}} + \bignorm{\sum_{1 \leq i < k} x_i\eta_i}_{\wh{\Lambda}^{-1}_{k,h}} + \bignorm{f_{k-1,h}}_{\wh{\Lambda}^{-1}_{k,h}}
\]
\[
    \leq \frac{C_w\cdot\left(\lambda + \Upsilon^{k-1}_{1,high}\right)}{\sqrt{\lambda + \Upsilon^{k-1}_{1,low}}} + \bignorm{\sum_{1 \leq i < k} x_i\eta_i}_{Z^{-1}_{k-1}} + \frac{C_{1,k-1}}{\sqrt{\lambda + \Upsilon^{k-1}_{1,low}}}
\]
Then,
\[
    \bignorm{\wh{\Lambda}^{\frac{1}{2}}_{k,h}\cdot\left(w_h - \wh{w}_{k,h}\right)}_2 \leq \frac{C_w\cdot\left(\lambda + \Upsilon^{k-1}_{1,high}\right) + C_{1,k-1}}{\sqrt{\lambda + \Upsilon^{k-1}_{1,low}}} + 8d\sqrt{\log\left(1+k/\lambda\right)\log\left(24k^2H/p\right)} + 4\sqrt{d}\log\left(\frac{24k^2H}{p}\right)
\]
\[
    \bignorm{\wh{\Lambda}^{\frac{1}{2}}_{k,h}\cdot\left(w_h - \wh{w}_{k,h}\right)}_2 \leq \check{\beta}_{k-1} \leq \check{\beta}_k
\]
\end{small}
It means that for a fixed $h \in [H]$, $w_h \in \bigcap_{k \in[K]}\check{\mathcal{C}}_{k,h}$ with probability at least $1-\frac{p}{6H}$.

The same reasoning can be followed with the proper changes to prove that $w_h \in \bigcap_{k \in [K]}\wt{C}_{k,h}$ with high probability. It suffices to notice that this time $\eta_i = \phi_{\wh{V}^2_{i,h+1}}(s^i_h,a^i_h)\cdot\wh{V}^2_{i,h+1}(s^i_{h+1}) - \Big\langle \phi_{\wh{V}^2_{i,h+1}}(s^i_h,a^i_h), w_h\Big\rangle$ is bounded by $H^2$ instead of $d$ and it can be similarly concluded that 
\[
    \bignorm{\wt{\Lambda}^{\frac{1}{2}}_{k,h}\cdot\left(w_h - \wt{w}_{k,h}\right)}_2 \leq \frac{C_w\cdot\left(\lambda + \Upsilon^{k-1}_{2,high}\right) + C_{2,k-1}}{\sqrt{\lambda + \Upsilon^{k-1}_{2,low}}} + 8\sqrt{dH^4\log\left(1+\frac{kH^4}{d\lambda}\right)\log\left(\frac{24k^2H}{p}\right)} + 4H^2\log\left(\frac{24k^2H}{p}\right)
\]
\[
    \bignorm{\wt{\Lambda}^{\frac{1}{2}}_{k,h}\cdot\left(w_h - \wt{w}_{k,h}\right)}_2 \leq \wt{\beta}_{k-1} \leq \wt{\beta}_k
\]
with probability at least $1-\frac{p}{6H}$ $\forall k \in [K]$. It follows that for a fixed $h \in [H]$, $w_h \in \bigcap_{k \in [K]}\wt{\mathcal{C}}_{k,h}$ with probability at least $1 - \frac{p}{6H}$.\\
We finally prove that for a fixed $h \in [H]$, $\left\{w_h \in \bigcap_{k \in [K]}\wh{\mathcal{C}}_{k,h}\right\}$ holds with probability at least $1-\frac{p}{2H}$, using again Thm. 4.1 in ~\citet{zhou2021vtrplus} and Lem.~\ref{lem:lemaC1_UCRL-VTRPlus}.

Considering the following parameters:
\begin{equation*}
\begin{aligned}
x_i = \bar{\sigma}^{-1}_{i,h}\cdot\phi_{\wh{V}_{i,h+1}(s^i_h,a^i_h)},
\eta_i = \bar{\sigma}^{-1}_{i,h}\cdot\mathds{1}_{\left\{w_h \in \check{\mathcal{C}}_{i,h} \cap \wt{\mathcal{C}}_{i,h}\right\}}\cdot\left[ \wh{V}_{i,h+1}(s^i_{h+1}) - \Big\langle \phi_{\wh{V}_{i,h+1}}(s^i_h,a^i_h), w_h\Big\rangle\right]\\ \mathcal{G}_i = \sigma\left(\left\{s^{k'}_{h'}, a^{k'}_{h'}, r^{k'}_{h'}\right\}_{(k',h') \leq (k,h)} \cup \left(g^1_{k',h'}, g^2_{k',h'}, f^1_{k',h'},f^2_{k',h'}, B^1_{k',h'}, B^3_{k',h'}\right)_{(k',h') \leq (k,h)}\right), \\
\mu^\star = w_h, y_i = \Big\langle \mu^\star, x_i\Big\rangle + \eta_i = \bar{\sigma}^{-1}_{i,h}\cdot\mathds{1}_{\left\{w_h \in \check{\mathcal{C}}_{i,h} \cap \wt{\mathcal{C}}_{i,h}\right\}}\cdot \wh{V}_{i,h+1}(s^i_{h+1}), Z_i = \lambda\cdot I_{d \times d} + \sum_{i'=1}^i x_{i'}\cdot x^\intercal_{i'}
\end{aligned}
\end{equation*}
and remark that $\mathbb{E}\left[\eta_i|\mathcal{G}_i\right] = 0$, $\abs{\eta_i} \leq \bar{\sigma}^{-1}_{i,h} \leq \sqrt{d}$, $\norm{x_i}_2 \leq \sqrt{d}$. As done in ~\citet{zhou2021vtrplus}, using Lem.~\ref{lem:lemaC1_UCRL-VTRPlus}, we obtain that $\mathbb{E}\left[\eta^2_i|\mathcal{G}_i\right] \leq 1$.

Then, applying Thm. 4.1 from ~\citet{zhou2021vtrplus}, with probability at least $1-\frac{p}{6H}$:
\begin{equation}\label{eq:bernstein_bound2}
    \bignorm{\sum_{i=1}^k x_i\eta_i}_{Z^{-1}_k} \leq 8\sqrt{d\log\left(1+k/\lambda\right)\log\left(24k^2H/p\right)} + 4\sqrt{d}\log\left(\frac{24k^2H}{p}\right), \forall k \in [K]
\end{equation}
On the event $\left\{w_h \in \bigcap_{k \in [K]} \check{\mathcal{C}}_{k,h}\right\} \cap \mathcal{E}_{hyp}$ and Eq.~\ref{eq:bernstein_bound2} holds, $\eta_i = \bar{\sigma}^{-1}_{i,h}\cdot\left[ \wh{V}_{i,h+1}(s^i_{h+1}) - \Big\langle \phi_{\wh{V}_{i,h+1}}(s^i_h,a^i_h), w_h\Big\rangle\right]$ and it implies that as proved previously:
\[
    \bignorm{\wh{\Lambda}^{\frac{1}{2}}_{k,h}\cdot\left(w - \wh{w}_{k,h}\right)}_2 \leq \frac{C_w\cdot\left(\lambda + \Upsilon^{k-1}_{1,high}\right)}{\sqrt{\lambda + \Upsilon^{k-1}_{1,low}}} + \bignorm{\sum_{1 \leq i < k} x_i\eta_i}_{Z^{-1}_{k-1}} + \frac{C_{1,k-1}}{\sqrt{\lambda + \Upsilon^{k-1}_{1,low}}}
\]
\[
\bignorm{\wh{\Lambda}^{\frac{1}{2}}_{k,h}\cdot\left(w - \wh{w}_{k,h}\right)}_2 \leq \frac{C_w\cdot\left(\lambda + \Upsilon^{k-1}_{1,high}\right) + C_{1,k-1}}{\sqrt{\lambda + \Upsilon^{k-1}_{1,low}}} + 8\sqrt{d\log\left(1+k/\lambda\right)\log\left(24k^2H/p\right)} + 4\sqrt{d}\log\left(\frac{24k^2H}{p}\right)
\]
\[
\bignorm{\wh{\Lambda}^{\frac{1}{2}}_{k,h}\cdot\left(w - \wh{w}_{k,h}\right)}_2 \leq \wh{\beta}_{k-1} \leq \wh{\beta}_k
\]
Hence, recalling that $\mathbb{P}\left(\bigcap_{k \in [K]}\check{C}_{k,h}|\mathcal{E}_{hyp}\right), \mathbb{P}\left(\bigcap_{k \in [K]}\wt{C}_{k,h}|\mathcal{E}_{hyp}\right) \geq 1 - \frac{p}{6H}$ and that Eq.~\ref{eq:bernstein_bound2} holds with probability $1-\frac{p}{6H}$, we conclude that $\left\{w_h \in \bigcap_{k \in [K]}\wh{\mathcal{C}}_{k,h}\right\}$ with probability at least $1-\frac{p}{2H}$, conditioned on the event $\mathcal{E}_{hyp}$. Finally, the proof is finished by remarking that in the event $\left\{w_h \in \bigcap_{k \in [K]} \wh{\mathcal{C}}_{k,h} \cap \wt{\mathcal{C}}_{k,h}\right\}$, 
\[
    \abs{\bar{\mathbb{V}}_{k,h}\wh{V}_{k,h+1}(s^k_h,a^k_h) - \mathbb{V}_h\wh{V}_{k,h+1}(s^k_h,a^k_h)} \leq E_{k,h}
\]
follows directly from Lem.~\ref{lem:lemaC1_UCRL-VTRPlus} and the definition of $E_{k,h}$.
\subsubsection{Regret analysis}
The procedure is almost exactly the same as in ~\citep[][App. C.2]{zhou2021vtrplus}, the adaptations are stated here for instructive reasons. Again, all the lemmas hold under $\mathcal{E}_{hyp}$.
\begin{lemma}[Lemma C.2 in ~\citet{zhou2021vtrplus}, Azuma-Hoeffding inequality ~\citep{azuma1967ineq}]\label{lem:azuma} Let $M > 0$ be a constant. Let $\left\{x_i\right\}^n_{i=1}$ be a martingale difference sequence with respect to a filtration $\left\{\mathcal{G}_i\right\}_i$ ($\mathbb{E}\left[x_i|\mathcal{G}_i\right] = 0$ a.s. and $x_i$ is $\mathcal{G}_{i+1}$-measurable) such that for all $i \in [n]$, $\abs{x_i} \leq M$ holds almost surely. Then, for any $0 < p < 1$, with probability at least $1-p$, we have
\[
    \sum_{i=1}^n x_i \leq M\sqrt{2n\log\left(\frac{1}{p}\right)}
\]
\end{lemma}
Let $\mathcal{E} = \left\{w_h \in \bigcap_{k \in [K]} \wh{\mathcal{C}}_{k,h} \cap \wt{\mathcal{C}}_{k,h}\right\}$ be the event where the conclusion of Lem.~\ref{lem:lemaC1_UCRL-VTRPlus} is true. Then, $\mathbb{P}\left(\mathcal{E}|\mathcal{E}_{hyp}\right) \geq 1-\frac{p}{2}$.

Similarly, the events $\mathcal{E}_1, \mathcal{E}_2$ are defined as follows:
\[
\mathcal{E}_1 = \left\{\forall h' \in [H],\sum_{k=1}^K\sum_{h=h'}^H \left[\left[\mathbb{P}_h(\wh{V}_{k,h+1}-V^{\pi_k}_{h+1})\right](s^k_h,a^k_h) - \left[\wh{V}_{k,h+1}-V^{\pi_k}_{h+1}\right](s^k_{h+1}) \right] \leq 4H\sqrt{2KH\log\left(\frac{6H}{p}\right)} \right\}
\]
\[
\mathcal{E}_2 = \left\{\sum_{k=1}^K\sum_{h=1}^H \left[\mathbb{V}_hV^{\pi_k}_{h+1}\right](s^k_h,a^k_h) \leq 3\left(HT + H^3\log\left(\frac{6}{p}\right)\right)\right\}
\]
Then, $\mathbb{P}\left(\mathcal{E}_1| \mathcal{E}_{hyp}\right) \geq 1-\frac{p}{6}$ is due to Lem.~\ref{lem:azuma} and $\mathbb{P}\left(\mathcal{E}_2\right) \geq 1-\frac{p}{6}$ is due to the next lemma.
\begin{lemma}[Lemma C.3 in ~\citep{zhou2021vtrplus}, Total variance lemma , lemma C.5 in ~\citep{jin2018Q-learning}]\label{lem:lemaC3_totalVariance} With probability at least $1-p$, we have
\[
    \sum_{k=1}^K\sum_{h=1}^H \left[\mathbb{V}_hV^{\pi_k}_{h+1}\right](s^k_h,a^k_h) \leq 3\left(HT + H^3\log\left(\frac{1}{p}\right)\right)
\]
\end{lemma}
\begin{lemma}[Optmism, Lemma C.4 in ~\citep{zhou2021vtrplus}]\label{lem:lemaC4_Optimism} Let $\wh{Q}_{k,h}$, $\wh{V}_{k,h}$ be defined in Alg.~\ref{alg:Perturbed-UCRL-VTRPlus}. Then, on the event $\mathcal{E} \cap \mathcal{E}_{hyp}$, for any $s,a,k,h$, we have that $Q^\star_h(s,a) \leq \wh{Q}_{k,h}(s,a)$, $V^\star_h(s) \leq \wh{V}_{k,h}(s)$. 
\end{lemma}
\begin{lemma}[Lemma C.5 in ~\citep{zhou2021vtrplus}]\label{lem:lemaC5_UCRL-VTRPlus} Let $\wh{V}_{k,h}$, $\bar{\sigma}_{k,h}$ be defined in Alg.~\ref{alg:Perturbed-UCRL-VTRPlus}. Then, on the event $\mathcal{E} \cap \mathcal{E}_1 \cap \mathcal{E}_{hyp}$, we have
\[
\sum_{k=1}^K\left[\wh{V}_{k,1}(s^k_1) - V^{\pi_k}_1(s^k_1)\right] \leq 2\wh{\beta}_K\sqrt{\sum_{k=1}^K\sum_{h=1}^H \bar{\sigma}^2_{k,h}}\sqrt{2Hd\log\left(1+\frac{K}{\lambda}\right)} + 4H\sqrt{2T\log\left(\frac{6H}{p}\right)}
\]
\[
\sum_{k=1}^K\sum_{h=1}^H \mathbb{P}_h\left[\wh{V}_{k,h+1} - V^{\pi_k}_{h+1}\right](s^k_h,a^k_h) \leq 2\wh{\beta}_K\sqrt{\sum_{k=1}^K\sum_{h=1}^H \bar{\sigma}^2_{k,h}}\sqrt{2dH^3\log\left(1+\frac{K}{\lambda}\right)} + 4H^2\sqrt{2T\log\left(\frac{6H}{p}\right)}
\]
\end{lemma}
\begin{lemma}[Lemma C.6 in ~\citep{zhou2021vtrplus}]\label{lem:lemaC6_UCRL-VTRPlus} Let $\wh{V}_{k,h}$, $\bar{\sigma}_{k,h}$ be defined in Alg.~\ref{alg:Perturbed-UCRL-VTRPlus}. Then, on the event $\mathcal{E} \cap \mathcal{E}_2 \cap \mathcal{E}_{hyp}$, we have
\[
\sum_{k=1}^K\sum_{h=1}^H \bar{\sigma}^2_{k,h} \leq \frac{H^2T}{d} + 3\left(HT + H^3\log\left(\frac{6}{p}\right)\right) + 2H\sum_{k=1}^K\sum_{h=1}^H \mathbb{P}_h\left[\wh{V}_{k,h+1} - V^{\pi_k}_{h+1}(s^k_h,a^k_h)\right]
\]
\[
    + 2\widetilde{\beta}_K\sqrt{T}\sqrt{2dH\log\left(1 + \frac{KH^4}{d\lambda}\right)} + 7\check{\beta}_KH^2\sqrt{T}\sqrt{2dH\log\left(1+\frac{K}{\lambda}\right)}
\]
\end{lemma}
\noindent Thus, $\mathbb{P}\left(\mathcal{E} \cap \mathcal{E}_1 \cap \mathcal{E}_2|\mathcal{E}_{hyp}\right) \geq 1-\frac{5p}{6} \implies \mathbb{P}\left(\mathcal{E} \cap \mathcal{E}_1 \cap \mathcal{E}_2 \cap \mathcal{E}_{hyp}\right) \geq 1-p$, we can bound the regret as follows due to Lems. \ref{lem:lemaC4_Optimism} and \ref{lem:lemaC5_UCRL-VTRPlus}:
\begin{equation*}
    R(K) \leq 2\wh{\beta}_{k,h}\sqrt{\sum_{k=1}^K\sum_{h=1}^H \bar{\sigma}^2_{k,h}}\sqrt{2Hd\log\left(1+\frac{KH}{d\lambda}\right)} + 4H\sqrt{2T\log\left(\frac{6H}{p}\right)}
\end{equation*}
\begin{equation}\label{eq:proof_regret_UCRL-VTRPlus}
    R(K) = \widetilde{O}\left(H^{\frac{1}{2}}d^{\frac{1}{2}}\wh{\beta}_{k,h}\sqrt{\sum_{k=1}^K\sum_{h=1}^H \bar{\sigma}^2_{k,h}} + H^{\frac{3}{2}}K^{\frac{1}{2}}\right)
\end{equation}
We focus now on bounding $\sum_{k=1}^K\sum_{h=1}^H \bar{\sigma}^2_{k,h}$ using Lems.~\ref{lem:lemaC5_UCRL-VTRPlus} and ~\ref{lem:lemaC6_UCRL-VTRPlus}:
\[
\sum_{k=1}^K\sum_{h=1}^H \bar{\sigma}^2_{k,h} \leq \frac{H^2T}{d} + 3\left(HT + H^3\log\left(\frac{6}{p}\right)\right) + 2H\sum_{k=1}^K\sum_{h=1}^H \mathbb{P}_h\left[\wh{V}_{k,h+1} - V^{\pi_k}_{h+1}(s^k_h,a^k_h)\right]
\]
\[
    + 2\widetilde{\beta}_K\sqrt{T}\sqrt{2dH\log\left(1 + \frac{KH^4}{d\lambda}\right)} + 7\check{\beta}_KH^2\sqrt{T}\sqrt{2dH\log\left(1+\frac{K}{\lambda}\right)}
\]
\[
    \leq \frac{H^2T}{d} +  3\left(HT + H^3\log\left(\frac{6}{p}\right)\right) + 
\]
\[
    2H\cdot\left(2\wh{\beta}_K\sqrt{\sum_{k=1}^K\sum_{h=1}^H \bar{\sigma}^2_{k,h}}\sqrt{2dH^3\log\left(1+\frac{K}{\lambda}\right)} + 4H^2\sqrt{2T\log\left(\frac{6H}{p}\right)}\right)
\]
\[
 + 2\widetilde{\beta}_K\sqrt{T}\sqrt{2dH\log\left(1 + \frac{KH^4}{d\lambda}\right)} + 7\check{\beta}_KH^2\sqrt{T}\sqrt{2dH\log\left(1+\frac{K}{\lambda}\right)}
\]
At this point, we recall that $x \leq a\sqrt{x} + b \implies x \leq \frac{9}{4}\left(a^2+b\right)$, we conclude that 
\[
\sum_{k=1}^K\sum_{h=1}^H \bar{\sigma}^2_{k,h} = \widetilde{O}\left(H^5d\cdot\wh{\beta}^2_K + H^{\frac{7}{2}}K^{\frac{1}{2}} + \frac{H^3K}{d} + H^2K + \wh{\beta}_KHd^{\frac{1}{2}}K^{\frac{1}{2}} + \check{\beta}_KH^3d^{\frac{1}{2}}K^{\frac{1}{2}}\right)
\]
\[
    \sqrt{\sum_{k=1}^K\sum_{h=1}^H \bar{\sigma}^2_{k,h}} = \widetilde{O}\left(\max\left\{H^{\frac{5}{2}}d^{\frac{1}{2}}\wh{\beta}_K, H^{\frac{7}{4}}K^{\frac{1}{4}}, \frac{H^{\frac{3}{2}}K^{\frac{1}{2}}}{d^{\frac{1}{2}}}, HK^{\frac{1}{2}}, H^{\frac{1}{2}}d^{\frac{1}{4}}\sqrt{\wh{\beta}_K}K^{\frac{1}{4}}, H^{\frac{3}{2}}d^{\frac{1}{4}}\sqrt{\check{\beta}_K}K^{\frac{1}{4}}\right\}\right)
\]
Replacing the inequality above in Eq.~\ref{eq:proof_regret_UCRL-VTRPlus}, we have that under the event $\mathcal{E}\cap\mathcal{E}_1\cap\mathcal{E}_2\cap\mathcal{E}_{hyp}$, holding with probability at least $1-p$:
\[
    R(K) = \widetilde{O}\left(H^{\frac{1}{2}}d^{\frac{1}{2}}\wh{\beta}_{k,h}\cdot\max\left\{H^{\frac{5}{2}}d^{\frac{1}{2}}\wh{\beta}_K, H^{\frac{7}{4}}K^{\frac{1}{4}}, \frac{H^{\frac{3}{2}}K^{\frac{1}{2}}}{d^{\frac{1}{2}}}, HK^{\frac{1}{2}}, H^{\frac{1}{2}}d^{\frac{1}{4}}\sqrt{\wh{\beta}_K}K^{\frac{1}{4}}, H^{\frac{3}{2}}d^{\frac{1}{4}}\sqrt{\check{\beta}_K}K^{\frac{1}{4}}\right\} + H^{\frac{3}{2}}K^{\frac{1}{2}}\right)
\]

\subsection{Proof of Cor.~\ref{app:cor.JDP-UCRL-VTRPlus}}\label{app:proof.cor.JDP-UCRL-VTRPlus}
\paragraph{Regret analysis:} Due to Clm.~\ref{app:clm.bounded.eigen}, the eigenvalues of $D^1_{k,h} = B^3_{k,h}$ are in the interval  $\left[\Upsilon^{J,1}_{\frac{p}{24KH}}, 3\Upsilon^{J,1}_{\frac{p}{24KH}}\right]$ with probability at least $1-\frac{p}{24KH}$ for each $(k,h) \in [K]\times[H]$. Similarly, the eigenvalues of $D^2_{k,h} = B^4_{k,h}$ are in the interval $\left[\Upsilon^{J,2}_{\frac{p}{24KH}}, 3\Upsilon^{J,2}_{\frac{p}{24KH}}\right]$ with probability at least $1-\frac{p}{24KH}$ for each $(k,h) \in [K]\times[H]$.

Due to Clm.~\ref{app:clm.bound.gauss.vector}, for each $(k,h) \in [K]\times [H]$, $f_{k,h} = f^2_{k,h}$ is in the interval $\left[-C^{J,1}_{\frac{p}{24KH}}, C^{J,1}_{\frac{p}{24KH}}\right]$ with probability at least $1-\frac{p}{24KH}$, where $C^{J,1}_{\frac{p}{24KH}} = \sigma_{B,1}\sqrt{K_0}\left(\sqrt{d} + 2\sqrt{\log\left(\frac{24KH}{p}\right)}\right)$. For the same reason, for each $(k,h) \in [K]\times [H]$, $g_{k,h} = g^2_{k,h}$ is in the interval $\left[-C^{J,1}_{\frac{p}{24KH}}, C^{J,2}_{\frac{p}{24KH}}\right]$ with probability at least $1-\frac{p}{24KH}$, where $C^{J,2}_{\frac{p}{24KH}} = \sigma_{B,2}\sqrt{K_0}\left(\sqrt{d} + 2\sqrt{\log\left(\frac{24KH}{p}\right)}\right)$. 

In the union event, holding with probability at least $1-\frac{p}{6}$, we can choose $\Upsilon^k_{1,low} = \Upsilon^{J,1}_{\frac{p}{24KH}}$, $\Upsilon^k_{1,high} = 3\Upsilon^{J,1}_{\frac{p}{24KH}}$, $\Upsilon^k_{2,low} = \Upsilon^{J,2}_{\frac{p}{24KH}}$ and $\Upsilon^k_{2,high} = 3\Upsilon^{J,2}_{\frac{p}{24KH}}$ $\forall (k,h) \in [K]\times [H]$. Also, $C_{1,k}$ and $C_{2,k}$ can be taken equal to $C^{J,1}_{\frac{p}{24KH}}$ and $C^{J,2}_{\frac{p}{24KH}}$, respectively. \\
The choices of $\check{\beta}_k, \wh{\beta}_k$ and $\wt{\beta}_k$ guarantee the following properties:
\[
    \check{\beta}_k \geq \frac{C_w\left(\lambda + \Upsilon^k_{1,high}\right) + C_{1,k}}{\sqrt{\lambda + \Upsilon^k_{1,low}}} + 8d\sqrt{\log\left(1 + \frac{K}{\lambda} \right)\log\left(\frac{24k^2H}{p}\right)} + 4\sqrt{d}\log\left( \frac{24k^2H}{p}\right)
\]
\[
    \wh{\beta}_k \geq \frac{C_w\left(\lambda + \Upsilon^k_{1,high}\right) + C_{1,k}}{\sqrt{\lambda + \Upsilon^k_{1,low}}} + 8\sqrt{d\log\left(1 + \frac{K}{\lambda} \right)\log\left(\frac{24k^2H}{p}\right)} + 4\sqrt{d}\log\left( \frac{24k^2H}{p}\right)
\]
\[
    \widetilde{\beta}_k \geq \frac{C_w\left(\lambda + \Upsilon^k_{2,high}\right) + C_{2,k}}{\sqrt{\lambda + \Upsilon^k_{2,low}}} + 8\sqrt{dH^4\log\left(1+\frac{KH^4}{d\lambda}\right)\log\left(\frac{24k^2H}{p} \right)} + 4H^2\log\left(\frac{24k^2H}{p}\right)
\]
since 
\[
3(C_w+1)\sqrt{\lambda + \Upsilon^{J,1}_{\frac{p}{24KH}}} \geq \frac{C_w\left(\lambda + 3\Upsilon^{J,1}_{\frac{p}{24KH}}\right) + C^{J,1}_{\frac{p}{24KH}}}{\sqrt{\lambda + \Upsilon^{J,1}_{\frac{p}{24KH}}}} = \frac{C_w\left(\lambda + \Upsilon^k_{1,high}\right) + C_{1,k}}{\sqrt{\lambda + \Upsilon^k_{1,low}}}
\]
\[
3(C_w+1)\sqrt{\lambda + \Upsilon^{J,2}_{\frac{p}{24KH}}} \geq \frac{C_w\left(\lambda + 3\Upsilon^{J,2}_{\frac{p}{24KH}}\right) + C^{J,2}_{\frac{p}{4KH}}}{\sqrt{\lambda + \Upsilon^{J,2}_{\frac{p}{24KH}}}} = \frac{C_w\left(\lambda + \Upsilon^k_{2,high}\right) + C_{2,k}}{\sqrt{\lambda + \Upsilon^k_{2,low}}}
\]
and $\Upsilon^{J,1}_{\frac{p}{24KH}} \geq C^{J,1}_{\frac{p}{24KH}}$, $\Upsilon^{J,2}_{\frac{p}{24KH}} \geq C^{J,2}_{\frac{p}{24KH}}$. It is also easy to notice that $\check{\beta}_k, \wh{\beta}_k$ and $\wt{\beta}_k$ are non-decreasing sequences.\\
The regret bound just follows from Thm.~\ref{app:th.regret_Private_UCRL-VTRPlus} by replacing the values of $\check{\beta}_k, \wh{\beta}_k$ and $\wt{\beta}_k$.

\paragraph{Privacy analysis:} The analysis resembles the privacy analysis of Cor.~\ref{cor:JDP-UCRL-VTR}. We also show that the mechanism $\left(\wh{u}_{k,h}, \wt{u}_{k,h}, \wh{\Lambda}_{k,h}, \wt{\Lambda}_{k,h}\right)_{(k,h) \in [K] \times [H]}$ is $(\epsilon,\delta)$-DP. In order to do it, for each $h \in [H]$, we prove that the mechanisms $(\wh{\Lambda}_{k,h})_{k \in [K]}$, $(\wh{u}_{k,h})_{k \in [K]}$, $(\widetilde{\Lambda}_{k,h})_{k \in [K]}$ and $(\widetilde{u}_{k,h})_{k \in [K]}$ are $\left(\frac{\epsilon}{4\sqrt{8H\log\left(\frac{8}{\delta}\right)}}, \frac{\delta}{8H}\right)$-DP for each $h \in [H]$ and conclude by advanced composition that $(\wh{\Lambda}_{k,h})_{(k,h) \in [K]\times [H]}, (\wh{u}_{k,h})_{(k,h) \in [K]\times [H]}, (\widetilde{\Lambda}_{k,h})_{(k,h) \in [K]\times [H]}$ and $(\widetilde{u}_{k,h})_{(k,h) \in [K]\times [H]}$ are $\left(\frac{\epsilon}{4},\frac{\delta}{4}\right)$-DP.\\
Remarking that $\bignorm{\bar{\sigma}^{-2}_{k,h}\phi_{\wh{V}_{k,h+1}}(s^k_h, a^k_h)\phi^\intercal_{\wh{V}_{k,h+1}}(s^k_h, a^k_h)}_F \leq d$ and due to the design of gaussian mechanisms, it is enough with $\sigma_{B,1} \geq 2d\cdot\frac{32\sqrt{HK_0\log\left(\frac{16H}{\delta}\right)\log\left(\frac{8}{\delta}\right)}}{\epsilon}\sqrt{2\log\left(\frac{32HK_0}{\delta}\right)}$, to prove that the mechanism is $\left(\frac{\epsilon}{4\sqrt{8H\log\left(\frac{8}{\delta}\right)}}, \frac{\delta}{8H}\right)$-DP.
The same argument holds for $\bignorm{\bar{\sigma}^{-2}_{k,h}\phi_{\wh{V}_{k,h+1}}(s^k_h,a^k_h)\wh{V}_{k,h+1}(s^k_{h+1})}_2 \leq d$ and the choice of the variance $\sigma_{B,1}$.\\
Again, since $\bignorm{\phi_{\wh{V}^2_{k,h+1}}(s^k_h,a^k_h)\phi^\intercal_{\wh{V}^2_{k,h+1}}(s^k_h,a^k_h)}_F \leq H^4$ and due to the choice of 
\[
    \sigma_{B,2} \geq 2H^4\cdot\frac{32\sqrt{HK_0\log\left(\frac{16H}{\delta}\right)\log\left(\frac{8}{\delta}\right)}}{\epsilon}\sqrt{2\log\left(\frac{32HK_0}{\delta}\right)},
\]
it follows that the mechanism is $\left(\frac{\epsilon}{4\sqrt{8H\log\left(\frac{8}{\delta}\right)}}, \frac{\delta}{8H}\right)$-DP as well. Exactly the same argument is true for $\bignorm{\phi_{\wh{V}^2_{k,h+1}}(s^k_h,a^k_h)\cdot \wh{V}^2_{k,h+1}(s^k_{h+1})}_2 \leq H^2$ and $\sigma_{B,2}$.

Using the post-processing property $\wh{w}_{k,h} = \wh{\Lambda}^{-1}_{k,h}\cdot\wh{u}_{k,h}$, $\wt{w}_{k,h} = \wt{\Lambda}^{-1}_{k,h}\cdot\wt{u}_{k,h}$, we conclude that $\left(\wh{w}_{k,h}, \wt{w}_{k,h}, \wh{\Lambda}_{k,h}, \wt{\Lambda}_{k,h}\right)_{(k,h) \in [K] \times [H]}$ is $(\epsilon,\delta)$-DP. As in Cor.~\ref{cor:JDP-UCRL-VTR}, we rely on the Billboard Lemma and remark that the $Q$-function, $\wh{Q}_k$, is a function of the mechanism $\left(\wh{\Lambda}_{k,h}, \wh{w}_{k,h}\right)_{(k,h) \in [K]\times [H]}$ to guarantee that Alg.~\ref{alg:Perturbed-UCRL-VTRPlus} is $(\epsilon,\delta)$-JDP. 

\subsection{Proof of Cor.~\ref{app:cor.LDP-UCRL-VTRPlus}}\label{app:proof.cor.LDP-UCRL-VTRPlus}
\paragraph{Regret analysis:} Due to Clm.~\ref{app:clm.bounded.eigen}, the eigenvalues of $D^1_{k,h} = \left(\sum_{1 \leq i \leq k}B^1_{i,h}\right) + 2\Upsilon^{L,1}_{\frac{p}{24KH}}$ are in the interval $\left[\Upsilon^{L,1}_{\frac{p}{24KH}}, 3\Upsilon^{L,1}_{\frac{p}{24KH}}\right]$ with probability at least $1-\frac{p}{24KH}$ for each $(k,h) \in [K]\times [H]$. For the same reason, the eigenvalues of $D^2_{k,h} = \left(\sum_{1 \leq i \leq k} B^2_{i,h}\right) + 2\Upsilon^{L,2}_{\frac{p}{24KH}}$ are in the interval $\left[\Upsilon^{L,2}_{\frac{p}{24KH}}, 3\Upsilon^{L,2}_{\frac{p}{24KH}}\right]$ with probability at least $1-\frac{p}{24KH}$.\\
Due to Clm.~\ref{app:clm.bound.gauss.vector}, for each $(k,h) \in [K]\times [H]$, $f_{k,h} = \left(\sum_{1 \leq i \leq k}f^1_{i,h}\right)$ is in the interval $\left[-C^{L,1}_{\frac{p}{24KH}}, C^{L,1}_{\frac{p}{24KH}}\right]$ with probability at least $1-\frac{p}{24KH}$, where $C^{L,1}_{\frac{p}{24KH}} = \sigma_{B,1}\sqrt{K}\left(\sqrt{d} + 2\sqrt{\log\left( \frac{24KH}{p}\right)}\right)$ and $g_{k,h} = \sum_{1 \leq i \leq k}g^1_{i,h}$ is in the interval $\left[-C^{L,2}_{\frac{p}{24KH}}, C^{L,2}_{\frac{p}{24KH}}\right]$ with probability at least $1-\frac{p}{24KH}$, where $C^{L,2}_{\frac{p}{24KH}} = \sigma_{B,2}\sqrt{K}\left(\sqrt{d} + 2\sqrt{\log\left( \frac{24KH}{p}\right)}\right)$. \\
In the union event, holding with probability at least $1-\frac{p}{6}$, choosing $\Upsilon^k_{1,low} = \Upsilon^{L,1}_{\frac{p}{24KH}}$, $\Upsilon^k_{1,high} = 3\Upsilon^{L,1}_{\frac{p}{24KH}}$, $\Upsilon^k_{2,low} = \Upsilon^{L,2}_{\frac{p}{24KH}}$ and $\Upsilon^k_{2,high} = \Upsilon^{L,2}_{\frac{p}{24KH}}$ for each $(k,h) \in [K] \times [H]$. Also, $C_{1,k}$ and $C_{2,k}$ can be chosen as $C^{L,1}_{\frac{p}{24KH}}$ and $C^{L,2}_{\frac{p}{24KH}}$.\\
The choices of $\check{\beta}_k, \wh{\beta}_k$ and $\wt{\beta}_k$ guarantee the following properties:
\[
    \check{\beta}_k \geq \frac{C_w\left(\lambda + \Upsilon^k_{1,high}\right) + C_{1,k}}{\sqrt{\lambda + \Upsilon^k_{1,low}}} + 8d\sqrt{\log\left(1 + \frac{K}{\lambda} \right)\log\left(\frac{24k^2H}{p}\right)} + 4\sqrt{d}\log\left( \frac{24k^2H}{p}\right)
\]
\[
    \wh{\beta}_k \geq \frac{C_w\left(\lambda + \Upsilon^k_{1,high}\right) + C_{1,k}}{\sqrt{\lambda + \Upsilon^k_{1,low}}} + 8\sqrt{d\log\left(1 + \frac{K}{\lambda} \right)\log\left(\frac{24k^2H}{p}\right)} + 4\sqrt{d}\log\left( \frac{24k^2H}{p}\right)
\]
\[
    \widetilde{\beta}_k \geq \frac{C_w\left(\lambda + \Upsilon^k_{2,high}\right) + C_{2,k}}{\sqrt{\lambda + \Upsilon^k_{2,low}}} + 8\sqrt{dH^4\log\left(1+\frac{KH^4}{d\lambda}\right)\log\left(\frac{24k^2H}{p} \right)} + 4H^2\log\left(\frac{24k^2H}{p}\right)
\]
since 
\[
3(C_w+1)\sqrt{\lambda + \Upsilon^{L,1}_{\frac{p}{24KH}}} \geq \frac{C_w\left(\lambda + 3\Upsilon^{L,1}_{\frac{p}{24KH}}\right) + C^{J,1}_{\frac{p}{24KH}}}{\sqrt{\lambda + \Upsilon^{L,1}_{\frac{p}{24KH}}}} = \frac{C_w\left(\lambda + \Upsilon^k_{1,high}\right) + C_{1,k}}{\sqrt{\lambda + \Upsilon^k_{1,low}}}
\]
\[
3(C_w+1)\sqrt{\lambda + \Upsilon^{L,2}_{\frac{p}{24KH}}} \geq \frac{C_w\left(\lambda + 3\Upsilon^{L,2}_{\frac{p}{24KH}}\right) + C^{J,2}_{\frac{p}{24KH}}}{\sqrt{\lambda + \Upsilon^{L,2}_{\frac{p}{24KH}}}} = \frac{C_w\left(\lambda + \Upsilon^k_{2,high}\right) + C_{2,k}}{\sqrt{\lambda + \Upsilon^k_{2,low}}}
\]
and $\Upsilon^{L,1}_{\frac{p}{24KH}} \geq C^{L,1}_{\frac{p}{24KH}}$, $\Upsilon^{L,2}_{\frac{p}{24KH}} \geq C^{L,2}_{\frac{p}{24KH}}$. It is also easy to notice that $\check{\beta}_k, \wh{\beta}_k$ and $\wt{\beta}_k$ are non-decreasing sequences.\\
The regret bound just follows from Thm.~\ref{app:th.regret_Private_UCRL-VTRPlus} by replacing the values of $\check{\beta}_k, \wh{\beta}_k$ and $\widetilde{\beta}_k$ and remarking that $\frac{(C_w + 1)^2}{\epsilon}d^{\frac{5}{2}}H^4K^{\frac{1}{2}} + \frac{(C_w + 1)}{\sqrt{\epsilon}}d^{\frac{5}{4}}H^2K^{\frac{3}{4}} \geq 2\cdot\frac{(C_w+1)^{\frac{3}{2}}}{\epsilon^{\frac{3}{4}}}d^{\frac{15}{8}}H^{\frac{11}{4}}K^{\frac{5}{8}}$.
\paragraph{Privacy analysis:} The analysis resembles to its counterpart in the previous framework (Cor.~\ref{cor:LDP-UCRL-VTR}). We only need to show that the mechanisms 
\[\left\{\bar{\sigma}^{-2}_{k,h}\cdot\phi_{\wh{V}_{k,h+1}}(s^k_h,a^k_h)\phi^\intercal_{\wh{V}_{k,h+1}}(s^k_h,a^k_h) + B^1_{k,h}\right\}_{h \in [H]},\]
\[\left\{\bar{\sigma}^{-2}_{k,h}\phi_{\wh{V}_{k,h+1}}(s^k_h,a^k_h)\wh{V}_{k,h+1}(s^k_{h+1}) + f^1_{k,h}\right\}_{h \in [H]}\]
\[\left\{\phi_{\wh{V}^2_{k,h+1}}(s^k_h,a^k_h)\phi^\intercal_{\wh{V}^2_{k,h+1}}(s^k_h,a^k_h) + B^2_{k,h}\right\}_{h \in [H]}
\]
and $ \left\{\phi_{\wh{V}^2_{k,h+1}}(s^k_h,a^k_h)\wh{V}^2_{k,h+1}(s^k_{h+1}) + g^1_{k,h}\right\}_{h \in [H]}$ are $\left(\frac{\epsilon}{4},\frac{\delta}{4}\right)$-LDP for any fixed $k \in [K]$. 

To that end, we will prove that the mechanisms 
\[\left\{\bar{\sigma}^{-2}_{k,h}\cdot\phi_{\wh{V}^2_{k,h+1}}(s^k_h,a^k_h)\phi^\intercal_{\wh{V}_{k,h+1}}(s^k_h,a^k_h) + B^1_{k,h}\right\},\]
\[\left\{\bar{\sigma}^{-2}_{k,h}\phi_{\wh{V}_{k,h+1}}(s^k_h,a^k_h)\wh{V}_{k,h+1}(s^k_{h+1}) + f^1_{k,h}\right\},\] 
\[\left\{\phi_{\wh{V}^2_{k,h+1}}(s^k_h,a^k_h)\phi^\intercal_{\wh{V}^2_{k,h+1}}(s^k_h,a^k_h) + B^2_{k,h}\right\}\]
and $ \left\{\phi_{\wh{V}^2_{k,h+1}}(s^k_h,a^k_h)\wh{V}^2_{k,h+1}(s^k_{h+1}) + g^1_{k,h}\right\}$ are $\left(\frac{\epsilon}{4H}, \frac{\delta}{4H}\right)$-LDP for a fixed $(k,h) \in [K] \times [H]$ and then conclude by simple composition of $H$ independent gaussian mechanisms. Finally, this last requirement is guaranteed by remarking that 
\begin{align*}
&\bignorm{\bar{\sigma}^{-2}_{k,h}\cdot\phi_{\wh{V}_{k,h+1}}(s^k_h,a^k_h)\phi^\intercal_{\wh{V}_{k,h+1}}(s^k_h,a^k_h)}_F \leq d, \bignorm{\bar{\sigma}^{-2}_{k,h}\phi_{\wh{V}_{k,h+1}}(s^k_h,a^k_h)\wh{V}_{k,h+1}(s^k_{h+1})}_2 \\
&\leq d, \bignorm{\phi_{\wh{V}^2_{k,h+1}}(s^k_h,a^k_h)\phi^\intercal_{\wh{V}^2_{k,h+1}}(s^k_h,a^k_h)}_F \leq H^4, \bignorm{\phi_{\wh{V}^2_{k,h+1}}(s^k_h,a^k_h)\wh{V}^2_{k,h+1}(s^k_{h+1})}_2 \leq H^4
\end{align*}
and due to the choice of $\sigma_{B,1} \geq 2d\cdot\frac{4H}{\epsilon}\sqrt{2\log\left(\frac{8H}{\delta}\right)}$, $\sigma_{B,2} \geq 2H^4\cdot\frac{4H}{\epsilon}\sqrt{2\log\left(\frac{8H}{\delta}\right)}$.
\section{Analysis of privacy-preserving batched LSVI-UCB}\label{app:linearmdp.jdp}

\subsection{Initialization of the parameters in Alg.~\ref{alg:JDP-LowRank-RareSwitch.correct}}\label{app:parameters.jdp.lsviucb.correct}
Initialize the parameters in Alg.~\ref{alg:JDP-LowRank-RareSwitch.correct} as follows:
\begin{multicols}{2}
\begin{itemize}
    \item $\lambda = d$
    \item $B = \Big\lceil \frac{(K\epsilon)^{\frac{2}{5}}}{d^{\frac{3}{5}}H^{\frac{1}{5}}}\Big\rceil$
    \item $B_0 = \lceil \log_2 B + 1 \rceil$
    \item $\sigma_\Lambda = \frac{128}{\epsilon}\sqrt{BHB_0}\cdot\log^2\left(\frac{32HB_0B}{\delta}\right)$
    \item $\sigma_u = \frac{128}{\epsilon} H\sqrt{HB}\log^2\left(\frac{32HB_0B}{\delta}\right)$
    \item $\Upsilon^J_{\frac{p}{6KH}} = \sigma_\Lambda B_0\left(4\sqrt{d} + 2\log\left(\frac{6KH}{p}\right)\right)$
    \item $c_K = d\cdot\Upsilon^J_{\frac{p}{6KH}}$
    \item $C^J_{\frac{p}{6KH}} = \sigma_u\left(\sqrt{d} + 2\sqrt{\log\left(\frac{6KHd}{p}\right)}\right)$
    \item $U_K = \max\left\{1, 2H\sqrt{\frac{dK}{\lambda + c_K}} + \frac{C^{J}_{\frac{p}{6KH}}}{\lambda + c_K}\right\}$
    \item $\chi = \frac{24^2\cdot 18\cdot K^2d\cdot U_K\cdot H}{p}$
    \item $\beta = 24H\sqrt{d(\lambda + c_K)}\log\chi$
\end{itemize}
\end{multicols}

\subsection{Proof of Thm.~\ref{th:regret_JDP-LSVI-UCB-RarelySwitch}}\label{app:regret_JDP-LSVI-UCB-RarelySwitch}
This proof contains two challenging parts. The first one is the privacy analysis, where the main difficulty is how to ensure the mechanisms satisfy the DP-property despite of the batch scheme. The second part is on the regret analysis, not only due to the perturbations to ensure privacy, but also because the bonuses are delayed with respect to the current episode.

Letting $N_K$ be the number of different policies produced in Alg.~\ref{alg:JDP-LowRank-RareSwitch.correct}. We remark that due to the design of the algorithm, $N_K \leq B$ and therefore, the number of different outputs for $\wt{\Lambda}_{i,h}$ and $\wt{u}_{i,h}$ is also bounded by $B$. Furthermore, a very fundamental observation is that the batches schedule $(k_i)_i$ is a deterministic functions of $K$ only and independent from any sequence of users. Similarly, $N_K$ depends only on $K$ and $B$.
\subsubsection{Privacy Analysis}
Let $b_k$ and $k_b$ be the batch containing episode $k$ and the first episode in the $b$-th batch, respectively. 

To prove that the mechanism $M(\mathfrak{U}^K) = \left(\widetilde{\Lambda}_{b_k,h}, \widetilde{w}_{b_k,h} \right)_{(k,h) \in [K] \times [H]}$ is $(\epsilon,\delta)$-DP, we rely on the advanced composition theorem (see App.~\ref{app:advanced.comp}). Remark that $\wt{\Lambda}_{0,h}$ and $\wt{u}_{0,h}$ are deterministic quantities ($(0,0)$-DP) and hence are also DP for any privacy level.


For a sequence of $K$ users $\mathfrak{U}^K$ and for each $h \in [H]$, let $M^\Lambda_h$ and $M^u_h$ be the following mechanisms acting on $\mathfrak{U}^K$:

$M^\Lambda_h(\mathfrak{U}^K) = (\wt{\Lambda}_{b_k,h})_{k \in [K]}, M^u_h(\mathfrak{U}^K, \mathfrak{a}_{h+1}) = (\bar{u}_{b_k,h})_{k \in [K]}$
%
where $\mathfrak{a}_{H+1} := \emptyset $ and $\mathfrak{a}_{h+1}$ corresponds to any output coming from the mechanism at the next stage, in our case, $\mathfrak{a}_{h+1}$ will be the output of the couple $(M^\Lambda_{h+1}(\mathfrak{U}^K), M^u_{h+1}(\mathfrak{U}^K, \mathfrak{a}_{h+2}))$. Also, the value of $M^u_h(\mathfrak{U}^K,\mathfrak{a}_{h+1}) = \left(\bar{u}_{b_k,h}\right)_{k \in [K]}$ depends on $\mathfrak{a}_{h+1} = \left(\mathfrak{a}^\Lambda_{h+1}, \mathfrak{a}^u_{h+1}\right)$ and the private data of the sequence of users $\mathfrak{U}^K$ and is defined as follows:
\begin{equation*}
\begin{aligned}
\bar{u}_{b_i,h} = \sum_{j=1}^{k_{b_i}-1} \phi(s^j_h,a^j_h)\Big[r_h(s^j_h,a^j_h) + \Pi_{[0,H]}\Big[\max_{a \in A} \left(\left(\mathfrak{a}^\Lambda_{h+1,i+1}\right)^{-1}\mathfrak{a}^u_{h+1,i+1}\right)^\intercal\cdot \phi(s^j_{h+1},a)+ \\ \beta\cdot\norm{\phi(s^j_{h+1},a)}_{(\mathfrak{a}^\Lambda_{h+1,i+1})^{-1}}\Big] \Big] +  \eta^{b_i}_h
\end{aligned}
\end{equation*}
A critical observation is to notice that in the scenario described in Alg.~\ref{alg:JDP-LowRank-RareSwitch.correct}, where $\mathfrak{a}_{h+1} = \left(\mathfrak{a}^\Lambda_{h+1}, \mathfrak{a}^u_{h+1}\right) = \left((\wt{\Lambda}_{b_k,h+1})_k, (\wt{u}_{b_k,h+1})_k\right)$, we have $\bar{u}_{b_i,h} = \wt{u}_{b_i,h}$ (term defined in line 20 from Alg.~\ref{alg:JDP-LowRank-RareSwitch.correct}). 

It implies that for the particular choice of $\mathfrak{a}_{h+1} = \left(\mathfrak{a}^\Lambda_{h+1}, \mathfrak{a}^u_{h+1}\right)$ in Alg.~\ref{alg:JDP-LowRank-RareSwitch.correct}, the mechanism $M^u_h\left(\mathfrak{U}^K, \left((\wt{\Lambda}_{b_k,h+1})_k, (\wt{u}_{b_k,h+1})_k\right)\right) = (\wt{u}_{b_k,h})_{k \in [K]}$. We notice that $(M^\Lambda_h,M^u_h)_h$ are adaptive mechanism and hence, they are perfectly suited to apply the advanced composition theorem.



Hence, we prove that for a fix $h \in [H]$, $M^\Lambda_h(\mathfrak{U}^K)$ and $M^u_h(\mathfrak{U}^K, \mathfrak{a}_{h+1})$ are $(\epsilon',\delta')$-DP for any $\mathfrak{a}_{h+1}$\footnote{We need to prove that $M^u_h(\mathfrak{U}^K, \mathfrak{a}_{h+1})$ is $(\epsilon',\delta')$-DP for any $\mathfrak{a}_{h+1}$ to apply the advanced composition theorem.}, where $\epsilon' = \frac{\epsilon}{2\sqrt{8H\log\left(\frac{4}{\delta}\right)}}$ and $\delta' = \frac{\delta}{4H}$.  

By simple composition (see App.~\ref{app:simple.comp}), we would have that $\left(M^\Lambda_h(\mathfrak{U}^K), M^u_h(\mathfrak{U}^K), \left((\wt{\Lambda}_{b_k,h+1})_k, (\wt{u}_{b_k,h+1})_k\right)\right)$ is $(2\epsilon',2\delta')$-DP. Therefore, a final application of the advanced composition theorem of the $H$ adaptive mechanisms $\left(M^\Lambda_h(\mathfrak{U}^K), M^u_h(\mathfrak{U}^K), \left((\wt{\Lambda}_{b_k,h+1})_k, (\wt{u}_{b_k,h+1})_k\right)\right) = \left(\wt{\Lambda}_{b_k,h},\wt{u}_{b_k,h}\right)_{k \in [K]}$ yields that the mechanism $\left(\wt{\Lambda}_{b_k,h},\wt{u}_{b_k,h}\right)_{(k,h) \in [K]\times [H]}$ is $(\epsilon,\delta)$-DP. By the post-processing property (see App.~\ref{app:post.processing}), since $\wt{w}_{b_k,h} = \wt{\Lambda}^{-1}_{b_k,h}\cdot\wt{u}_{b_k,h}$, we conclude that the mechanism $M(\mathfrak{U}^K) = \left(\widetilde{\Lambda}_{b_k,h}, \widetilde{w}_{b_k,h} \right)_{(k,h) \in [K] \times [H]}$ is also $(\epsilon,\delta)$-DP.




The goal is now to prove that indeed $M^\Lambda_h$ and $M^u_h$ are $(\epsilon',\delta')$-DP. 
\paragraph{\underline{$M^\Lambda_h$ is $(\epsilon',\delta')$-DP:}} By construction, each node of the $H$ trees contain a matrix $Z_\Lambda = \frac{Z'_\Lambda+Z'^\intercal_\Lambda}{\sqrt{2}} \in \mathbb{R}^{d \times d}$, with each entry $1 \leq i \leq j < d$ i.i.d.\ drawn from $(Z'_\Lambda)_{i,j} \sim \mathcal{N}(0,\sigma^2_\Lambda)$.
In the considered setting, the adversary observes only $B-1$ changes in $\wt{\Lambda}$ and $\wt{u}$, then each tree mechanism needs to output only $B$ partial sums. If we wish to target a certain level of privacy $(\epsilon'_1, \delta'_1)$ for each of the partial sums, we need to guarantee that the noise in each node ensures  $\left(\frac{\epsilon'_1}{\sqrt{8B_0\log\left(\frac{2}{\delta'_1}\right)}}, \frac{\delta'_1}{2B_0} \right)$-DP since the partial sums are the sum of at most $B_0$ internal nodes in the tree.

Noticing that the $L_1$-sensitivity of $\phi(s,a)\phi(s,a)^\intercal$ w.r.t.\ the user is bounded by 2 and due to the design of Gaussian mechanisms (see App.~\ref{app:gaussian.mech}), it suffices to have $\sigma_\Lambda \geq  \frac{2\sqrt{2\log\left(\frac{4B_0}{\delta'_1}\right)}}{\epsilon'_1}\cdot\sqrt{8B_0\log\left(\frac{2}{\delta'_1}\right)}$. 

Our choice of $\sigma_\Lambda$ allows to take $\epsilon'_1 = \frac{\epsilon'}{\sqrt{8B\log\left( \frac{2}{\delta'}\right)}}$ and $\delta'_1 = \frac{\delta'}{2B}$, since $\sigma_\Lambda = \frac{128}{\epsilon}\sqrt{BHB_0}\cdot\log^2\left(\frac{32HB_0B}{\delta}\right)\geq \frac{8\sqrt{B\log\left(\frac{2}{\delta'}\right)\log\left(\frac{8B_0B}{\delta'}\right)}}{\epsilon'}\cdot\sqrt{8B_0\log\left(\frac{4B}{\delta'}\right)}$. It means that the partial sums on the trees are, in fact, $(\epsilon'_1,\delta'_1)$-DP.

Proving that the mechanism $M^\Lambda_h$ is $(\epsilon',\delta')$-DP is equivalent to prove the DP-property for two $k_0$-neighboring sequences $\mathfrak{U}^K$, $\mathfrak{U}'^K$ and any set $E \in (\mathbb{R}^{d \times d})^{K}$ of possible outputs of the mechanism $M^\Lambda_h$: 
\begin{equation}\label{eq:1.privacy.jdp.lsvi.ucb}
    \mathbb{P}\left(M^\Lambda_h(\mathfrak{U}^K) \in E| \mathcal{X}\right) \leq e^{\epsilon'}\mathbb{P}\left(M^\Lambda_h(\mathfrak{U}'^K) \in E| \mathcal{X}' \right) + \delta'
\end{equation}
where $\mathcal{X} = \left\{(s_{i,h},a_{i,h},r_{i,h})\right\}_{(i,h) \in [K] \times [H]}$ are the trajectories observed by the algorithm. $\mathcal{X}'$ is the same set of trajectories with the only difference in position $k_0$.

Letting $Z_{d \times d}$ be the following set:
\begin{equation*}
\begin{aligned}
    Z_{d\times d} = \big\{(z_1,\dots, z_K) \in (\mathbb{R}^{d \times d})^K| z_{k_0} = \dots = z_{k_1-1}, z_{k_1} = \dots = z_{k_2-1}, \dots, z_{k_{N_K-1}} = \dots = z_{K}\big\}
\end{aligned}
\end{equation*}
Eq.~\ref{eq:1.privacy.jdp.lsvi.ucb} can be separated into 2 terms: 
\begin{equation}\label{eq:2.privacy.jdp.lsvi.ucb}
\begin{aligned}
\mathbb{P}\left(M^\Lambda_h(\mathfrak{U}^K) \in E \cap Z_{d\times d}| \mathcal{X}\right) + \underbrace{\mathbb{P}\left(M^\Lambda_h(\mathfrak{U}^K) \in E \cap Z_{d\times d}^\complement| \mathcal{X}\right)}_0 \leq \\ e^{\epsilon'}\mathbb{P}\left(M^\Lambda_h(\mathfrak{U}'^K) \in E \cap Z_{d\times d}| \mathcal{X}' \right) + \underbrace{e^{\epsilon'}\mathbb{P}\left(M^\Lambda_h(\mathfrak{U}'^K) \in E \cap Z^\complement_{d \times d}| \mathcal{X}' \right)}_0 + \delta'
\end{aligned}
\end{equation}
It can be noticed that the second terms on both sides of Eq.~\ref{eq:2.privacy.jdp.lsvi.ucb} are 0. The reason is simply that $M^\Lambda_h(\mathfrak{U}^K) = \left(\wt{\Lambda}_{b_1,h}, \wt{\Lambda}_{b_2,h},\dots, \wt{\Lambda}_{b_K,h}\right) \in Z_{d\times d}$ since $b_{k_0} = b_1 = b_2 = \dots = b_{k_1-1}$, $b_{k_1} = \dots = b_{k_2-1}, \dots, b_{k_{N_K-1}} = \dots = b_{K}$.
Therefore, Eq.~\ref{eq:1.privacy.jdp.lsvi.ucb} is equivalent to:
\begin{equation}\label{eq:3.privacy.jdp.lsvi.ucb}
\begin{aligned}
    \mathbb{P}\left(M^\Lambda_h(\mathfrak{U}^K) \in E \cap Z_{d\times d}| \mathcal{X}\right) \leq e^{\epsilon'}\mathbb{P}\left(M^\Lambda_h(\mathfrak{U}'^K) \in E \cap Z_{d\times d}| \mathcal{X}' \right) + \delta'
\end{aligned}
\end{equation}
Remarking that
\[
\left\{M^\Lambda_h(\mathfrak{U}^K) = \left(\wt{\Lambda}_{b_1,h}, \wt{\Lambda}_{b_2,h},\dots, \wt{\Lambda}_{b_K,h}\right) \in E \cap Z_{d\times d}\right\} = \left\{\left(\wt{\Lambda}_{0,h}, \dots, \wt{\Lambda}_{N_K-1,h}\right) \in \Pi_{d\times d}\left(E \cap Z_{d\times d}\right) \right\}
\] 
\[
\left\{M^\Lambda_h(\mathfrak{U}'^K) = \left(\wt{\Lambda}'_{b_1,h}, \wt{\Lambda}'_{b_2,h},\dots, \wt{\Lambda}'_{b_K,h}\right) \in E \cap Z_{d\times d}\right\} = \left\{\left(\wt{\Lambda}'_{0,h}, \dots, \wt{\Lambda}'_{N_K-1,h}\right) \in \Pi_{d\times d}\left(E \cap Z_{d\times d}\right) \right\}
\]
where $\Pi_{d\times d}:Z_{d\times d} \to (\mathbb{R}^{d\times d})^{N_K}$ is the projection over the different coordinates of each point --i.e. $\Pi_{d\times d}(z_1, \dots, z_K) = \left( z_{k_0}, z_{k_1},\dots, z_{k_{N_K-1}}\right)$.

Eq.~\ref{eq:3.privacy.jdp.lsvi.ucb} can be written as follows:
\begin{equation}\label{eq:4.privacy.jdp.lsvi.ucb}
\begin{aligned}
    \mathbb{P}\left(\left(\wt{\Lambda}_{0,h}, \dots, \wt{\Lambda}_{N_K-1,h}\right) \in \Pi_{d\times d}\left(E \cap Z_{d\times d}\right)| \mathcal{X}\right) \leq e^{\epsilon'}\mathbb{P}\left(\left(\wt{\Lambda}'_{0,h}, \dots, \wt{\Lambda}'_{N_K-1,h}\right) \in \Pi_{d\times d}\left(E \cap Z_{d\times d}\right)| \mathcal{X}' \right) + \delta'
\end{aligned}
\end{equation}

Due to the tree-based mechanism\footnote{In the case of $\wt{\Lambda}_{0,h}$, it is $(0,0)$-DP}, we know that each $\wt{\Lambda}_{i,h}$ is $(\epsilon'_1, \delta'_1)$-DP and by the advanced composition theorem of $N_K \leq B$ mechanisms, we know $\left(\widetilde{\Lambda}_{0,h}, \widetilde{\Lambda}_{1,h}, \dots, \widetilde{\Lambda}_{N_K-1,h}\right)$ is $(\epsilon',\delta')$-DP, since $\epsilon'_1$ and $\delta'_1$ were chosen intentionally in that way. Thus, the DP-property expressed in Eq.~\ref{eq:4.privacy.jdp.lsvi.ucb} is true. 
\emph{Note that it is possible to perform this analysis since the batch schedule is fix and not data-dependent.}

\paragraph{\underline{$M^u_h$ is $(\epsilon',\delta')$-DP:}}The next step is to show that $M^u_h(\mathfrak{U}^K,\mathfrak{a}_{h+1})$ is $(\epsilon',\delta')$-DP for any output $\mathfrak{a}_{h+1} = (\mathfrak{a}^\Lambda_{h+1}, \mathfrak{a}^u_{h+1})$ coming from future mechanisms (adaptively chosen). In Alg.~\ref{alg:JDP-LowRank-RareSwitch.correct}, we have $\mathfrak{a}^\Lambda_{h+1,k} = \wt{\Lambda}_{b_k,h+1}$ and $\mathfrak{a}^u_{h+1,k} = \wt{u}_{b_k,h+1}$.

The DP-property for the mechanism  $M^u_h(\mathfrak{U}^K, \mathfrak{a}_{h+1})$ is expressed as follows:
\begin{equation}\label{eq:5.privacy.jdp.lsvi.ucb}
\begin{aligned}
    \mathbb{P}\left(M^u_h(\mathfrak{U}^K, \mathfrak{a}_{h+1}) \in E|\mathcal{X}\right) \leq e^{\epsilon'}\mathbb{P}\left(M^u_h(\mathfrak{U}'^K,\mathfrak{a}_{h+1}) \in E|\mathcal{X}'\right) + \delta'
\end{aligned}
\end{equation}
where $\mathfrak{U}^K$ and $\mathfrak{U}'^K$ are two $k_0$-neighboring sequence of users, $E \in (\mathbb{R}^d)^K$ is any set of possible outputs of the mechanism $M^u_h$, $\mathcal{X}$, $\mathcal{X}'$ are the trajectories observed by each sequence of users (differing only at position $k_0$) and $M^u_h(\mathfrak{U}^K, \mathfrak{a}_{h+1}) = \left( \bar{u}_{b_1,h}, \dots, \bar{u}_{b_K,h}\right)$ with 
\begin{equation*}
\begin{aligned}
\bar{u}_{b_i,h} = \sum_{j=1}^{k_{b_i}-1} \phi(s^j_h,a^j_h)\Big[r_h(s^j_h,a^j_h) + \Pi_{[0,H]}\Big[\max_{a \in A} \left(\left(\mathfrak{a}^\Lambda_{h+1,i+1}\right)^{-1}\mathfrak{a}^u_{h+1,i+1}\right)^\intercal\cdot \phi(s^j_{h+1},a)+ \\ \beta\cdot\norm{\phi(s^j_{h+1},a)}_{(\mathfrak{a}^\Lambda_{h+1,i+1})^{-1}}\Big] \Big] +  \eta^{b_i}_h
\end{aligned}
\end{equation*}

Similarly to the previous case, let $Z_d$ be the following set:
\begin{equation*}
\begin{aligned}
    Z_d = \big\{(z_1,\dots, z_K) \in (\mathbb{R}^d)^K| z_{k_0} = \dots = z_{k_1-1}, z_{k_1} = \dots = z_{k_2-1}, \dots, z_{k_{N_K-1}} = \dots = z_{K}\big\}
\end{aligned}
\end{equation*}
Eq.~\ref{eq:5.privacy.jdp.lsvi.ucb} can be separated into 2 terms:
\begin{equation}\label{eq:6.privacy.jdp.lsvi.ucb}
\begin{aligned}
\mathbb{P}\left(M^u_h(\mathfrak{U}^K,\mathfrak{a}_{h+1}) \in E \cap Z_d| \mathcal{X}\right) + \underbrace{\mathbb{P}\left(M^u_h(\mathfrak{U}^K, \mathfrak{a}_{h+1}) \in E \cap Z_d^\complement| \mathcal{X}\right)}_0 \leq \\ e^{\epsilon'}\mathbb{P}\left(M^u_h(\mathfrak{U}'^K,\mathfrak{a}_{h+1}) \in E \cap Z_d| \mathcal{X}' \right) + \underbrace{e^{\epsilon'}\mathbb{P}\left(M^u_h(\mathfrak{U}'^K, \mathfrak{a}_{h+1}) \in E \cap Z^\complement_d| \mathcal{X}' \right)}_0 + \delta'
\end{aligned}
\end{equation}
Again, the second terms on both sides of Eq.~\ref{eq:6.privacy.jdp.lsvi.ucb} are 0 given that $M^u_h(\mathfrak{U}^K) = \left(\bar{u}_{b_1,h}, \bar{u}_{b_2,h}, \dots, \bar{u}_{b_K,h}\right) \in Z_d$ since $b_{k_0} = b_1 = \dots = b_{k_1-1}$, $b_{k_1} = \dots = b_{k_2-1}, \dots, b_{k_{N_K-1}} = \dots = b_K$.

Thus, Eq.~\ref{eq:5.privacy.jdp.lsvi.ucb} is equivalent to:
\begin{equation}\label{eq:7.privacy.jdp.lsvi.ucb}
\begin{aligned}
    \mathbb{P}\left(M^u_h(\mathfrak{U}^K, \mathfrak{a}_{h+1}) \in E\cap Z_d|\mathcal{X}\right) \leq e^{\epsilon'}\mathbb{P}\left(M^u_h(\mathfrak{U}'^K,\mathfrak{a}_{h+1}) \in E \cap Z_d|\mathcal{X}'\right) + \delta'
\end{aligned}
\end{equation}
Remarking that
\[
\left\{M^u_h(\mathfrak{U}^K,\mathfrak{a}_{h+1}) = \left(\bar{u}_{b_1,h}, \bar{u}_{b_2,h},\dots, \bar{u}_{b_K,h}\right) \in E \cap Z_d\right\} = \left\{\left(\wt{u}_{0,h}, \dots, \bar{u}_{N_K-1,h}\right) \in \Pi_d\left(E \cap Z_d\right) \right\}
\] 
\[
\left\{M^u_h(\mathfrak{U}'^K,\mathfrak{a}_{h+1}) = \left(\bar{u}'_{b_1,h}, \bar{u}'_{b_2,h},\dots, \bar{u}'_{b_K,h}\right) \in E \cap Z_d\right\} = \left\{\left(\bar{u}'_{0,h}, \dots, \bar{u}'_{N_K-1,h}\right) \in \Pi_d\left(E \cap Z_d\right) \right\}
\]
where $\Pi_d:Z_d \to (\mathbb{R}^d)^{N_K}$ is the projection over the different coordinates of each point --i.e. $\Pi_d(z_1, \dots, z_K) = \left( z_{k_0}, z_{k_1},\dots, z_{k_{N_K-1}}\right)$.

Eq.~\ref{eq:7.privacy.jdp.lsvi.ucb} can be written as follows:
\begin{equation}\label{eq:8.privacy.jdp.lsvi.ucb}
\begin{aligned}
    \mathbb{P}\left(\left(\bar{u}_{0,h}, \dots, \bar{u}_{N_K-1,h}\right) \in \Pi_d\left(E \cap Z_d\right)| \mathcal{X}\right) \leq e^{\epsilon'}\mathbb{P}\left(\left(\bar{u}'_{0,h}, \dots, \bar{u}'_{N_K-1,h}\right) \in \Pi_d\left(E \cap Z_d\right)| \mathcal{X}' \right) + \delta'
\end{aligned}
\end{equation}

Finally, we observe that $\bar{u}_{i,h}$ is $(\epsilon'_1,\delta'_1)$-DP due to the design of the gaussian mechanism\footnote{In the case of $\bar{u}_{0,h}$, it is $(0,0)$-DP} and the choice of $\sigma_u$. Indeed, the crucial observation is that the sensitivity of the protected quantity is bounded by $2(H+1) \leq 4H$ for two neighboring sequences of users $\mathfrak{U}^k$ and $\mathfrak{U}'^K$. 

The reason is that the protected term 
\[
\sum_{j=1}^{k_{b_i}-1} \phi(s^j_h,a^j_h)\cdot\left\{r_h(s^j_h,a^j_h) + \Pi_{[0,H]}\left[ \max_{a \in A} \left(\left(\mathfrak{a}^\Lambda_{h+1,i+1}\right)^{-1}\mathfrak{a}^u_{h+1,i+1}\right)^\intercal\cdot \phi(s^j_{h+1},a)+ \beta\bignorm{\phi(s^j_{h+1},a)}_{(\mathfrak{a}^\Lambda_{h+1,i+1})^{-1}} \right]\right\}
\]
differs only at position $j=k_0$ for the neighboring sequences of users since $\mathfrak{a}_{h+1}$ is fixed and $k_{b_i}$ does not depend on the sequence of users (static batches defined at the beginning of the algorithm). Then, it is sufficient with $\sigma_u \geq \frac{4H\sqrt{2\log\left(\frac{2}{\delta'_1}\right)}}{\epsilon'_1} = \frac{16H\sqrt{B\log\left(\frac{4B}{\delta'}\right)\log\left(\frac{2}{\delta'}\right)}}{\epsilon'} = \frac{32H\sqrt{8HB\log\left(\frac{16HB}{\delta}\right)\log\left(\frac{8H}{\delta}\right)\log\left(\frac{4}{\delta}\right)}}{\epsilon}$, which is true since we took $\sigma_u = \frac{128H\sqrt{HB}\log^2\left(\frac{32HB_0B}{\delta}\right)}{\epsilon}$.

The previous analysis implies that the mechanism $M^u_h(\mathfrak{U}^K, \mathfrak{a}_{h+1})$ is $(\epsilon',\delta')$-DP for any $\mathfrak{a}_{h+1}$ in particular for the adaptively output $\mathfrak{a}^\Lambda_{h+1,k} = \wt{\Lambda}_{b_k,h+1}$, $\mathfrak{a}^u_{h+1,k} = \wt{u}_{b_k,h+1}$ of Alg.~\ref{alg:JDP-LowRank-RareSwitch.correct}.

As explained right at the beginning of the proof, proving that $M^\Lambda_h$ and $M^u_h$ are $(\epsilon',\delta')$-DP implies that $M(\mathfrak{U}^K) = \left(\widetilde{\Lambda}_{b_k,h}, \widetilde{w}_{b_k,h} \right)_{(k,h) \in [K] \times [H]}$ is $(\epsilon,\delta)$-DP. We finish the proof by applying the Billboard's Lemma (see App.~\ref{app:billboard.lemma}) under the same reasoning as in ~\citet{vietri2020privaterl}. 

The argument to justify that Alg.~\ref{alg:JDP-LowRank-RareSwitch.correct} is $(\epsilon,\delta)$-JDP is remarking again that each function $Q^k$ is a function of the output of the mechanisms $\left(\widetilde{\Lambda}_{b_k,h}, \widetilde{w}_{b_k,h}\right)_{(k,h) \in [K]\times [H]}$. This implies that the policy $\pi$ is also a function of $\left(\widetilde{\Lambda}_{b_k,h}, \widetilde{w}_{b_k,h}\right)_{(k,h) \in [K]\times [H]}$. Finally, the actions taken by each user, $a^k_h = \pi_k(s^k_h, h)$, can be interpreted as a function of the user's private data and $\left(\widetilde{\Lambda}_{b_k,h}, \widetilde{w}_{b_k,h}\right)_{(k,h) \in [K]\times [H]}$, as stated in the Billboard Lemma. \\
We conclude that Alg.~\ref{alg:JDP-LowRank-RareSwitch.correct} is $(\epsilon,\delta)$-JDP.

\subsubsection{Regret Analysis}
This analysis is based on the work of ~\citet{wang2021adaptivity}, which is itself based on ~\citet{jin2020lsviucb}.  
 
Consider the following high-probability events: 
\[
    \mathcal{E}_1 = \left\{\forall (k,h) \in [B]\times [H]: \bignorm{\eta^{k}_h}_2 \leq C^J_{\frac{p}{6KH}} \right\}
\]
\[
    \mathcal{E}_2 = \left\{\forall (k,h) \in [B] \times [H]: \bignorm{H^k_h}_2 \leq \Upsilon^J_{\frac{p}{6KH}} \right\}
\]
Due to Clms.~\ref{app:clm.bounded.eigen} and ~\ref{app:clm.bound.gauss.vector} and remarking that $B \leq K$, we have that $\mathbb{P}(\mathcal{E}_1) \geq 1 - \frac{p}{6}$ and $\mathbb{P}(\mathcal{E}_2) \geq 1 - \frac{p}{6}$. Calling $\mathcal{E} = \mathcal{E}_1 \cap \mathcal{E}_2$ , it follows that $\mathbb{P}(\mathcal{E}) \geq 1 - \frac{p}{3}$.\\ \\

We will prove the following lemmas conditioned on the event $\mathcal{E}$.

\begin{lemma}[Lem.B.1 in ~\citep{jin2020lsviucb}]\label{app:lem.1.regret} For any fixed policy $\pi$, let $\{w^\pi_h\}_{h \in [H]}$ be the corresponding weights such that $Q^\pi_h(s,a) = \langle \phi(s,a), w^\pi_h \rangle$ for all $(s,a,h) \in \mathcal{S} \times \mathcal{A} \times [H]$. Then, we have 
\[
    \forall h \in [H], \norm{w^\pi_h} \leq 2H\sqrt{d}
\]
\end{lemma}
\begin{proof}
Again, this is the same as the proof given in ~\citep{jin2020lsviucb} and doesn't require to be conditioned on the event $\mathcal{E}$.
\end{proof}
\begin{lemma}[Lem.B.2 in ~\citep{jin2020lsviucb}]\label{app:lem.2.regret} Under the event $\mathcal{E}$, for any $(k,h) \in [K] \times [H]$, the weight $\widetilde{w}_{b_k,h}$ in Alg.~\ref{alg:JDP-LowRank-RareSwitch.correct} satisfies: 
\[
   \norm{\wt{w}_{b_k,h}} \leq U_K := \max\left\{1, 2H\sqrt{\frac{dK}{\lambda + c_K}} + \frac{C^{J}_{\frac{p}{6KH}}}{\lambda + c_K}\right\} 
\]
\end{lemma}
\begin{proof}
Let $v$ any unit vector ($\norm{v}_2 = 1$), then we know that: 
\[
    |v^\intercal \widetilde{w}_{b_k,h}| \leq \left(\sum_{i=1}^{k_{b_k}-1} |v^\intercal\cdot\wt{\Lambda}^{-1}_{b_k,h}\cdot\phi(s^i_h,a^i_h)\cdot 2H|\right) + \Big|v^\intercal\cdot\wt{\Lambda}^{-1}_{b_k,h}\cdot\eta^{b_k}_h\Big|
\]
\[
    |v^\intercal \widetilde{w}_{b_k,h}| \leq 2H\sqrt{\left[\sum_{i=1}^{k_{b_k}-1} v^\intercal\cdot\wt{\Lambda}^{-1}_{b_k,h}\cdot v \right]\cdot \left[\sum_{i=1}^{k_{b_k}-1} (\phi(s^i_h,a^i_h))^\intercal\cdot\wt{\Lambda}^{-1}_{b_k,h}\cdot\phi(s^i_h,a^i_h)\right]} + \norm{v}\cdot\norm{\wt{\Lambda}^{-1}_{b_k,h}}\cdot\norm{\eta^{b_k}_h}
\]
\[
    |v^\intercal\wt{w}_{b_k,h}| \leq 2H\sqrt{\frac{d(k_{b_k}-1)}{\lambda + c_K}} + \frac{C^{J}_{\frac{p}{6KH}}}{\lambda + c_K}
\]
\[
    \norm{\wt{w}_{b_k,h}} = \max_{\norm{v} = 1} |v^\intercal\wt{w}_{b_k,h}| \leq 2H\sqrt{\frac{dK}{\lambda + c_K}} + \frac{C^{J}_{\frac{p}{6KH}}}{\lambda + c_K} \leq U_K
\]
\end{proof}
\begin{lemma}[Lem.B.3 in ~\citep{jin2020lsviucb}]\label{app:lem.3.regret}
Under the event $\mathcal{E}$, for $\chi = \frac{24^2\cdot 18\cdot K^2d\cdot U_K\cdot H}{p}$ and $\beta = 24\cdot H\cdot \sqrt{d(\lambda + c_K)}\cdot\log\chi$, and defining the event $\mathcal{C}$: 
\[
    \mathcal{C} = \left\{\forall (k,h) \in [K] \times [H]: \bignorm{\sum_{i=1}^{k_{b_k}-1} \phi(s^i_h,a^i_h)\left[V^{k_{b_k}}_{h+1}(s^i_{h+1}) - \mathbb{P}_hV^{k_{b_k}}_{h+1}(s^i_h,a^i_h)\right]}_{\wt{\Lambda}^{-1}_{b_k,h}} <  6\cdot dH\sqrt{\log\chi}\right\}
\]
where $V^k_h(s) = \max_{a \in [A]} Q^k_h(s,a)$ and $Q^k_h(s,a) = \min\left\{ H, (\wt{w}_{b_k,h})^\intercal\cdot\phi(s,a) + \beta\cdot\norm{\phi(s,a)}_{\wt{\Lambda}^{-1}_{b_k,h}}\right\}$\\
Then, $\mathbb{P}(\mathcal{C}|\mathcal{E}) \geq 1-\frac{p}{3}$. 
\end{lemma}
\begin{proof}
Under the event $\mathcal{E}$, for all $(k,h) \in [K] \times [H]$, by the previous lemma, we have that $\bignorm{\wt{w}_{b_k,h}} \leq U_K$. Also, notice that $\wt{\Lambda}_{b_k,h} \succeq \Lambda_{k_{b_k},h} + c_KI_{d \times d} \succeq \left(\lambda + c_K\right)I_{d \times d}$.\\
Then, for any fixed $h \in [H]$, Lemma D.6 from ~\citep{jin2020lsviucb} gives us a way to upper bound the covering number $\mathcal{N}_\epsilon$ of the class of value functions, $\mathcal{V}$. \\
\[
    \log \mathcal{N}_\epsilon \leq d\log\left(1 + \frac{4U_K}{\epsilon}\right) + d^2\log \left(1+\frac{8d^{1/2}\beta^2}{(\lambda + c_K)\cdot\epsilon^2}\right)
\]
Now, applying the self-normalized bound\footnote{Here, as done in the case of Linear Mixture setting, the filtrations include the sigma-algebras from the noises. Also, the inequality is applied to the matrix $\Lambda_{k_{b_k},h} + c_KI_{d \times d}$, instead of $\wt{\Lambda}_{b_k,h}$.} and approximating the class $\mathcal{V}$ by a set of fixed functions (Lemma D.4 in ~\citet{jin2020lsviucb}), we get that with probability at least $1-\frac{p}{6H}$, the following inequality holds for any $k \geq 1$: 
\begin{equation*}
\begin{aligned}
    \bignorm{\sum_{i=1}^{k_{b_k}-1} \phi(s^i_h,a^i_h)\left[V^{k_{b_k}}_{h+1}(s^i_{h+1}) - \mathbb{P}_h V^{k_{b_k}}_{h+1}(s^i_h,a^i_h)\right]}^2_{\wt{\Lambda}^{-1}_{b_k,h}} \leq \\ \bignorm{\sum_{i=1}^{k_{b_k}-1} \phi(s^i_h,a^i_h)\left[V^{k_{b_k}}_{h+1}(s^i_{h+1}) - \mathbb{P}_h V^{k_{b_k}}_{h+1}(s^i_h,a^i_h)\right]}^2_{(\Lambda_{k_{b_k},h} + c_KI_{d \times d})^{-1}} \leq
\end{aligned}
\end{equation*}
\[
     4H^2\left[\frac{d}{2}\log\left(\frac{k + \lambda + c_K}{\lambda + c_K}\right) + d\log\left(1 + \frac{4U_k}{\epsilon_0}\right) + d^2\log\left(1+ \frac{8\sqrt{d}\beta^2}{(\lambda + c_K)\cdot \epsilon_0^2}\right) + \log\left(\frac{6H}{p}\right)\right] + \frac{8k^2\epsilon_0^2}{\lambda + c_K} 
\]

Taking $\epsilon_0 = \frac{dH}{k}\sqrt{\lambda + c_K}$ and replacing $\beta = 24\cdot H\cdot \sqrt{d(\lambda + c_K)}\log\chi$, we obtain the following inequality:
\[
   \leq \underbrace{4H^2d^2\left(2 + \log\left(1 + \frac{8\cdot 24^2\cdot K^2\cdot \log^2\chi}{\sqrt{d}\cdot (\lambda + c_K)}\right)\right)}_{G_1}
\]
\[
    + \underbrace{2H^2d\log\left(1 + \frac{K}{\lambda + c_K}\right)}_{G_2} + \underbrace{4H^2d\log\left(1 + \frac{4U_K\cdot K}{dH\sqrt{\lambda + c_K}} \right)}_{G_3}
\]
\[
    + \underbrace{4H^2\log\left(\frac{6H}{p}\right)}_{G_4}
\]
\begin{itemize}
    \item $G_1 \leq 4H^2d^2\left[2 + \log\left(1+8\cdot 24^2 K^2\log^2\chi\right) \right]$.\\
    We focus on bounding $\log\left(1+8\cdot 24^2 K^2\log^2\chi\right)$:
    \[
        1+8\cdot 24^2K^2\log^2\chi \leq 9\cdot 24^2K^2\log^2\chi
    \]
    since $\chi \geq 3$.\\
    Now,
    \[
        9\cdot 24^2K^2\log^2\chi \leq \chi^2 \iff 3\cdot 24 K\log\chi \leq \chi
    \]
    and this inequality is valid if, for instance, $\chi \geq (3\cdot 24K)(6\cdot 24K) = 18\cdot 24^2K^2$. It follows that:  
    \[
        G_1 \leq 8H^2d^2\log\chi + 4H^2d^2\cdot\log\left(\chi^2\right) = 16H^2d^2\log\chi
    \]
    \item $G_2 \leq 2H^2d^2\log(2K) \leq 2H^2d^2\log\chi$ since $\chi \geq 2K$.
    \item $G_3 = 4H^2d\log\left(1 + \frac{4U_KK}{dH\sqrt{\lambda + c_K}} \right) \leq 4H^2d^2\log\left(5U_KK\right) \leq 4H^2d^2\log\chi$ since $\chi \geq 5U_KK$
    \item $G_4 \leq 4H^2d^2\log\chi$
\end{itemize}

\[
    \bignorm{\sum_{i=1}^{k_{b_k}-1} \phi(s^i_h,a^i_h)\left[V^{k_{b_k}}_{h+1}(s^i_{h+1}) - \mathbb{P}_h V^{k_{b_k}}_{h+1}(s^i_h,a^i_h)\right]}^2_{(\widetilde{\Lambda}^{b_k}_h)^{-1}} \leq 26\cdot H^2\cdot d^2\log\chi
\]
\end{proof}
\begin{lemma}[Lem.B.4 in ~\citep{jin2020lsviucb}]\label{app:lem.4.regret}
Under the event $\mathcal{E}$, for any fixed policy $\pi$, we have that for all $(s,a,h,k) \in \mathcal{S} \times \mathcal{A} \times [H] \times [K]$:
\[
    \langle \phi(s,a), \wt{w}_{b_k,h} \rangle - Q^\pi_h(s,a) = \mathbb{P}_h(V^{k_{b_k}}_{h+1} - V^\pi_{h+1})(s,a) + \Delta^k_h(s,a)
\]
for some $\Delta^k_h(s,a)$ satisfying $|\Delta^k_h(s,a)| \leq \beta\norm{\phi(s,a)}_{\wt{\Lambda}^{-1}_{b_k,h}}$
\end{lemma}
\begin{proof}
First, we prove the statement when $b_k = 0$. 
\[
    \Delta^k_h(s,a) = \langle \phi(s,a), \wt{w}_{0,h} \rangle - Q^\pi_h(s,a) - \mathbb{P}_h(V^k_{h+1} - V^\pi_{h+1})(s,a)  
\]
\[
    \Delta^k_h(s,a) = \Big\langle \phi(s,a), -w^\pi_h - \int(V^k_{h+1} - V^\pi_{h+1})(s') d\mu_h(s')\Big\rangle  
\]
\[
    \Delta^k_h(s,a) = \Big\langle \phi(s,a), -\theta_h -\int V^k_{h+1}(s') d\mu_h(s') \Big\rangle
\]
since $V^k_{h+1}(s') \leq H$ and $\norm{\theta_h} \leq \sqrt{d}$, we conclude thanks to Asm.~\ref{asm:lowrank} that $\abs{\Delta^k_h} \leq \norm{\phi(s,a)}_{\wt{\Lambda}^{-1}_{0,h}}\cdot \sqrt{\lambda}\left(\sqrt{d} + H\sqrt{d}\right) \leq \beta\norm{\phi(s,a)}_{\wt{\Lambda}^{-1}_{0,h}}$ given that $\beta = 24H\sqrt{d(\lambda + c_K)}\log\chi\geq \sqrt{d\lambda}(H+1)$.

If $b_k \geq 1$, the difference $\widetilde{w}_{b_k,h} - w^\pi_h$ would be equal to: 
\[
    \wt{w}_{b_k,h} - w^\pi_h = \wt{\Lambda}^{-1}_{b_k,h}\cdot\left(\eta^{b_k}_h + \sum_{i=1}^{k_{b_k}-1} \phi(s^i_h,a^i_h)\left[r^i_h + V^{k_{b_k}}_{h+1}(s^i_{h+1})\right]\right) - w^\pi_h
\]
\[
    = \wt{\Lambda}^{-1}_{b_k,h}\left\{\eta^{b_k}_h + \sum_{i=1}^{k_{b_k}-1} \phi(s^i_h,a^i_h)\left[r^i_h + V^{k_{b_k}}_{h+1}(s^i_{h+1})\right] - \wt{\Lambda}_{b_k,h}\cdot w^\pi_h\right\}
\]
\[
    = -\wt{\Lambda}^{-1}_{b_k,h}\cdot \left(\left(\lambda + c_K + \Upsilon^J_{\frac{p}{6KH}}\right)I_{d \times d} + H^{b_k}_h\right)\cdot w^\pi_h + \wt{\Lambda}^{-1}_{b_k,h}\cdot \sum_{i=1}^{k_{b_k}-1} \phi(s^i_h,a^i_h)\left[V^{k_{b_k}}_{h+1}(s^i_{h+1}) - \mathbb{P}_hV^\pi_{h+1}(s^i_h,a^i_h)\right] + \wt{\Lambda}^{-1}_{b_k,h}\cdot\eta^{b_k}_h
\]
\[
    = \underbrace{-\wt{\Lambda}^{-1}_{b_k,h}\cdot \left(\left(\lambda + c_K + \Upsilon^J_{\frac{p}{6KH}} \right)I_{d \times d} + H^{b_k}_h \right)\cdot w^\pi_h}_{q_1} + \underbrace{\wt{\Lambda}^{-1}_{b_k,h}\cdot\sum_{i=1}^{k_{b_k}-1} \phi(s^i_h,a^i_h)\cdot\left[V^{k_{b_k}}_{h+1}(s^i_{h+1}) - \mathbb{P}_hV^{k_{b_k}}_{h+1}(s^i_h,a^i_h)\right]}_{q_2}
\]
\[ 
    + \underbrace{\wt{\Lambda}^{-1}_{b_k,h}\cdot \sum_{i=1}^{k_{b_k}-1} \phi(s^i_h,a^i_h)\cdot\mathbb{P}(V^{k_{b_k}}_{h+1} - V^\pi_{h+1})(s^i_h,a^i_h)}_{q_3} 
    + \underbrace{\wt{\Lambda}^{-1}_{b_k,h}\cdot\eta^{b_k}_h}_{q_4}
\]
We bound now each of the inner products $\langle \phi(s,a), q_i\rangle$ for $1 \leq i \leq 4$ and conclude by noticing that $\langle \phi(s,a), \sum_{i=1}^4 q_i\rangle = \langle \phi(s,a), \wt{w}_{b_k,h} - w^\pi_h \rangle = \langle \phi(s,a), \wt{w}_{b_k,h} \rangle - Q^\pi_h(s,a)$.\\
\begin{itemize}
    \item We use Lem.~\ref{app:lem.1.regret} to bound $\bignorm{w^\pi_h}$ and also the corresponding bounds for the eigenvalues of each matrix $\wt{\Lambda}_{b_k,h}$ and $H^{b_k}_h$:
    \[
    \Big\lvert\langle \phi(s,a), q_1\rangle\Big\rvert = \Big\lvert\langle \phi(s,a), \wt{\Lambda}^{-1}_{b_k,h}\cdot \left(\left(\lambda + c_K + \Upsilon^J_{\frac{p}{6KH}} \right)I_{d \times d} + H^{b_k}_h\right)\cdot w^\pi_h \rangle \Big\rvert \leq 
    \]
    \[
    \bignorm{\phi(s,a)}_{\wt{\Lambda}^{-1}_{b_k,h}}\cdot\bignorm{\wt{\Lambda}^{-\frac{1}{2}}_{b_k,h}\cdot \left(\left(\lambda + c_K + \Upsilon^J_{\frac{p}{6KH}}\right)I_{d \times d} + H^{b_k}_h\right)\cdot w^\pi_h} \leq
    \]
    \[
    \bignorm{\wt{\Lambda}^{-\frac{1}{2}}_{b_k,h}}\cdot \bignorm{\left(\lambda + c_K + \Upsilon^J_{\frac{p}{6KH}} \right)I_{d \times d} + H^{b_k}_h}\cdot\bignorm{w^\pi_h}\cdot\bignorm{\phi(s,a)}_{\wt{\Lambda}^{-1}_{b_k,h}}
    \]
    \[
    \implies \Big\lvert\langle \phi(s,a), q_1\rangle\Big\rvert \leq \frac{1}{\sqrt{\lambda + c_K}}\cdot \left(\lambda + c_K + 3\Upsilon^J_{\frac{p}{6KH}}\right)\cdot 2H\sqrt{d}\cdot\norm{\phi(s,a)}_{\wt{\Lambda}^{-1}_{b_k,h}} 
    \]
    \[ \Big\lvert\langle \phi(s,a), q_1\rangle\Big\rvert    \leq 2H\sqrt{d}\left(\frac{\lambda + c_K + 3\Upsilon^J_{\frac{p}{6KH}}}{\sqrt{\lambda + c_K}} \right)\cdot\norm{\phi(s,a)}_{\wt{\Lambda}^{-1}_{b_k,h}}
    \]
    \item This is a straightforward application of Lem.~\ref{app:lem.3.regret}: $\Big\lvert\langle \phi(x,a), q_2 \rangle\Big\rvert < \norm{\phi(x,a)}_{\wt{\Lambda}^{-1}_{b_k,_h}}\cdot 6\cdot dH\sqrt{\log\chi}$
    \item We use the structure of the linear MDP: 
    \[
    \langle \phi(s,a), q_3\rangle = \Big\langle \phi(s,a), \wt{\Lambda}^{-1}_{b_k,h}\cdot\sum_{i=1}^{k_{b_k}-1} \phi(s^i_h,a^i_h)\cdot\mathbb{P}_h\left(V^{k_{b_k}}_{h+1} - V^\pi_{h+1}\right)(s^i_h,a^i_h) \Big\rangle
    \]
    \[
     = \Big\langle \phi(s,a), \wt{\Lambda}^{-1}_{b_k,h}\cdot\left(\sum_{i=1}^{k_{b_k}-1}\phi(s^i_h,a^i_h)\cdot\phi^\intercal(s^i_h,a^i_h)\right)\cdot \int (V^{k_{b_k}}_{h+1} - V^\pi_{h+1})(s') d\mu_h(s') \Big\rangle
    \]
    \[
     = \underbrace{\Big\langle \phi(s,a), \int (V^{k_{b_k}}_{h+1} - V^\pi_{h+1})(s') d\mu_h(s')\Big\rangle}_{p_1}
     \]
     \[
     \underbrace{-\Big\langle \phi(s,a), \wt{\Lambda}^{-1}_{b_k,h}\cdot\left(\left(\lambda + c_K +\Upsilon^J_{\frac{p}{6KH}}\right)I_{d \times d} + H^{b_k}_h\right)\int (V^{k_{b_k}}_{h+1} - V^\pi_{h+1})(s') d\mu_h(s')\Big\rangle}_{p_2}
    \]
    Then, we notice that $p_1 = \mathbb{P}_h(V^{k_{b_k}}_{h+1} - V^\pi_{h+1})(s,a)$ and\\ $|p_2| \leq \norm{\phi(s,a)}_{\wt{\Lambda}^{-1}_{b_k,h}}\cdot\bignorm{\wt{\Lambda}^{-\frac{1}{2}}_{b_k,h}}\cdot\bignorm{\left(\lambda + c_K + \Upsilon^J_{\frac{p}{6KH}}\right)I_{d \times d} + H^{b_k}_h}\cdot 2H\sqrt{d}$, this implies \\
    $|p_2| \leq \frac{1}{\sqrt{\lambda + c_K}}\cdot\ (\lambda + c_K + 3\Upsilon^J_{\frac{p}{6KH}})\cdot 2H\sqrt{d}\cdot\bignorm{\phi(s,a)}_{\wt{\Lambda}^{-1}_{b_k,h}} =  2H\sqrt{d}\left(\frac{\lambda + c_K + 3\Upsilon^J_{\frac{p}{6KH}}}{\sqrt{\lambda + c_K}}\right)\cdot\bignorm{\phi(s,a)}_{\widetilde{\Lambda}^{-1}_{b_k,h}}$
    \item $\Big\lvert \Big\langle \phi(s,a), q_4\Big\rangle\Big\rvert= \Big\lvert\Big\langle \phi(s,a), \wt{\Lambda}^{-1}_{b_k,h}\cdot\eta^{b_k}_h \Big\rangle\Big\rvert \leq \norm{\phi(s,a)}_{\wt{\Lambda}^{-1}_{b_k,h}}\cdot\bignorm{\wt{\Lambda}^{-\frac{1}{2}}_{b_k,h}\eta^{b_k}_h} \leq \\ \norm{\phi(s,a)}_{\wt{\Lambda}^{-1}_{b_k,h}}\cdot \bignorm{\wt{\Lambda}^{-\frac{1}{2}}_{b_k,h}}\cdot\norm{\eta^{b_k}_h}$
    \[
        \implies \Big\lvert\Big\langle \phi(s,a), q_4 \Big\rangle\Big\rvert \leq \frac{1}{\sqrt{\lambda + c_K}}\cdot C^J_{\frac{p}{6KH}} \cdot\bignorm{\phi(s,a)}_{\wt{\Lambda}^{-1}_{b_k,h}}
    \]
\end{itemize}
As said before, 
\[
\Big\langle \phi(s,a), \wt{w}_{b_k,h} \Big\rangle - Q^\pi_h(s,a) = \Big\langle \phi(s,a), \sum_{i=1}^4 q_i \Big\rangle = \langle \phi(s,a), p_1\rangle  + \langle \phi(s,a), q_1 + q_2 + p_2 + q_4 \rangle 
\]
\[
\Big\langle \phi(s,a), \wt{w}_{b_k,h} \Big\rangle - Q^\pi_h(s,a) - \mathbb{P}_h(V^{k_{b_k}}_{h+1} - V^\pi_{h+1})(s,a) = \langle \phi(s,a), q_1 + q_2 + p_2 + q_4 \rangle
\]
\[
|\Delta^k_h(s,a)| \leq |\langle \phi(s,a), q_1 \rangle| +  |\langle \phi(s,a), q_2 \rangle| +  |\langle \phi(s,a), p_2 \rangle| +  |\langle \phi(s,a), q_4 \rangle|
\]
\[
\leq \left(4H\sqrt{d}\left(\frac{\lambda + c_K + 3\Upsilon^J_{\frac{p}{6KH}}}{\sqrt{\lambda + c_K}}\right) + 6dH\sqrt{\log\chi} + \frac{C^J_{\frac{p}{6KH}}}{\sqrt{\lambda + c_K}} \right)\cdot\bignorm{\phi(x,a)}_{\wt{\Lambda}^{-1}_{b_k,h}}
\]
\[
\leq \beta \cdot\bignorm{\phi(s,a)}_{\wt{\Lambda}^{-1}_{b_k,h}}
\]
The last inequality is true if and only if
\[
24H\sqrt{d(\lambda + c_K)}\cdot\log\chi = \beta \geq \left(4H\sqrt{d}\left(\frac{\lambda + c_K + 3\Upsilon^J_{\frac{p}{6KH}}}{\sqrt{\lambda + c_K}}\right) + 6dH\sqrt{\log\chi} + \frac{C^J_{\frac{p}{6KH}}}{\sqrt{\lambda + c_K}} \right)
\]
We bound each term on the right to show that this inequality is true. 
\begin{itemize}
    \item The first term $4H\sqrt{d}\left(\frac{\lambda + c_K + 3\Upsilon^J_{\frac{p}{6KH}}}{\sqrt{\lambda + c_K}}\right) = 4H\sqrt{d(\lambda + c_K)} + 12H\sqrt{d}\cdot\frac{\Upsilon^J_{\frac{p}{6KH}}}{\sqrt{\lambda + c_K}} \leq 16H\sqrt{d(\lambda + c_K)}$ since $\lambda + c_K \geq \Upsilon^J_{\frac{p}{6KH}}$
    \item The second term $6dH\sqrt{\log\chi} \leq 6H\sqrt{d(\lambda + c_K)}\log\chi$ since $\lambda + c_K \geq d$
    \item Instead of considering the third term, consider the following ratio 
    \[
    \frac{C^J_{\frac{p}{6KH}}}{\lambda + c_K} = \frac{\sigma_u\left(\sqrt{d} + 2\sqrt{\log\left(\frac{6KH}{p}\right)}\right)}{\lambda + c_K} \leq \frac{\sigma_u\left(\sqrt{d} + 2\sqrt{\log\left(\frac{6KH}{p}\right)}\right)}{\Upsilon^J_{\frac{p}{6KH}}} = \frac{\sigma_u\left(\sqrt{d} + 2\sqrt{\log\left(\frac{6KH}{p}\right)}\right)}{\sigma_\Lambda B_0\left(4\sqrt{d} + 2\log\left(\frac{6KH}{p}\right)\right)}
    \]
    \[
        \frac{C^J_{\frac{p}{6KH}}}{\lambda + c_K} \leq  \frac{\sigma_u}{\sigma_\Lambda}\cdot\frac{1}{B_0}
    \]
    \[
        \frac{C^J_{\frac{p}{6KH}}}{\lambda + c_K} \leq
        \frac{H}{\sqrt{B_0}}\cdot\frac{1}{B_0} \leq H
    \]
    Hence,
    \[
        \frac{C^J_{\frac{p}{6KH}}}{\sqrt{\lambda + c_K}} \leq H\sqrt{\lambda + c_K} \leq 2H\sqrt{d(\lambda + c_K)}\cdot\log\chi
    \]
\end{itemize}
Adding up all the upper bounds given above, we get  the desired inequality. 
\end{proof}
\begin{lemma}[Lem.B.5 in ~\citep{jin2020lsviucb}]\label{app:lem.5.regret}
Under the event $\mathcal{E} \cap \mathcal{C}$, we have that $Q^k_h(s,a) \geq Q^\star_h(s,a)$ for all $(s,a,h,k) \in \mathcal{S} \times \mathcal{A} \times [H] \times [K]$.
\end{lemma}
\begin{proof}
The proof is exactly the same as the given in the original LSVI-UCB paper ~\citet{jin2020lsviucb} and depends only on Lem.~\ref{app:lem.4.regret} with the new definition of $\beta$.
\end{proof}
\begin{lemma}[Adapted from Lem.4.2 from ~\citet{wang2021adaptivity}]\label{app:lem.6.regret}
Let $\mathcal{B}$ be the event defined as follows:
\[
    \mathcal{B} = \left\{\sum_{k=1}^K \sum_{h=1}^H \left\{\left[\mathbb{P}_h\left(V^{k_{b_k}}_{h+1} - V^{\pi_k}_{h+1} \right)\right](s^k_h, a^k_h) - \left(V^{k_{b_k}}_{h+1} - V^{\pi_k}_{h+1}\right)(s^k_{h+1})\right\} \leq 2\sqrt{2KH^3\log\left(\frac{3}{p}\right)}\right\}
\]
Then, $\mathbb{P}\left(\mathcal{B}\right) \geq 1 - \frac{p}{3}$ and under the event $\mathcal{C} \cap \mathcal{E} \cap \mathcal{B}$, the regret of Alg.~\ref{alg:JDP-LowRank-RareSwitch.correct} is bounded in the following way:
\[
    R(K) \leq 2\sqrt{2KH^3\log(3/p)} + \sum_{k=1}^K \sum_{h=1}^H \min\left\{H, 2\beta\bignorm{\phi(s^k_h, a^k_h)}_{\wt{\Lambda}^{-1}_{b_k,h}}\right\} 
\]
\end{lemma}
\begin{proof}
The proof is similar to the procedure given in Lem.4.2 from ~\citet{wang2021adaptivity}.
Under the event $\mathcal{E} \cap \mathcal{C}$, we have 
\[
    R(K) = \sum_{k=1}^K \left[V^\star_1(s^k_1) - V^{\pi_k}_1(s^k_1)\right] \leq \sum_{k=1}^K \left[V^k_1(s^k_1) - V^{\pi_k}_1(s^k_1)\right] = \sum_{k=1}^K \left[ V^{k_{b_k}}_1(s^k_1) - V^{\pi_k}_1(s^k_1)\right]
\]
We remark that $V^{k_{b_k}}_h(s^k_h) - V^{\pi_k}_h(s^k_h) = Q^{k_{b_k}}_h(s^k_h, a^k_h) - Q^{\pi_k}_h(s^k_h, a^k_h)$ and combining it with the definition of $Q^{k_{b_k}}_h$ and Lem.~\ref{app:lem.5.regret}, we get: 
\begin{align*}
    &V^{k_{b_k}}_h(s^k_h) - V^{\pi_k}_h(s^k_h) \leq \phi^\intercal(s^k_h, a^k_h) \wt{w}_{b_k,h} - \phi^\intercal(s^k_h, a^k_h) w^{\pi_k}_h + \beta\cdot\bignorm{\phi(s^k_h, a^k_h)}_{\wt{\Lambda}^{-1}_{b_k,h}}\\
    &\leq \left[\mathbb{P}_h\left(V^{k_{b_k}}_{h+1} - V^{\pi_k}_{h+1}\right)\right](s^k_h,a^k_h) + 2\beta\bignorm{\phi(s^k_h, a^k_h)}_{\wt{\Lambda}^{-1}_{b_k,h}}
\end{align*}
\begin{align*}
    V^{k_{b_k}}_h(s^k_h) - V^{\pi_k}_h(s^k_h) &\leq \min\left\{H, \left[\mathbb{P}_h\left(V^{k_{b_k}}_{h+1} - V^{\pi_k}_{h+1}\right)\right](s^k_h, a^k_h) + 2\beta\bignorm{\phi(s^k_h, a^k_h)}_{\wt{\Lambda}^{-1}_{b_k,h}}\right\}\\
    &\leq \left[\mathbb{P}_h\left(V^{k_{b_k}}_{h+1} - V^{\pi_k}_{h+1}\right)\right](s^k_h, a^k_h) + \min\left\{H, 2\beta\bignorm{\phi(s^k_h, a^k_h)}_{\wt{\Lambda}^{-1}_{b_k,h}}\right\}\\
    &= V^{k_{b_k}}_{h+1}(s^k_{h+1}) - V^{\pi_k}_{h+1}(s^k_{h+1}) + \min\left\{H, 2\beta\bignorm{\phi(s^k_h, a^k_h)}_{\widetilde{\Lambda}^{-1}_{b_k,h}}\right\}\\
    &+ \left[\mathbb{P}_h\left(V^{k_{b_k}}_{h+1} - V^{\pi_k}_{h+1}\right)\right](s^k_h, a^k_h) - \left(V^{k_{b_k}}_{h+1}(s^k_{h+1}) - V^{\pi_k}_{h+1}(s^k_{h+1}) \right)
\end{align*}
Hence, by recursively expanding the above inequality, we get 
\[
    V^{k_{b_k}}_1(s^k_1) - V^{\pi_k}_1(s^k_1) = \sum_{h=1}^H \left\{\left[\mathbb{P}_h\left(V^{k_{b_k}}_{h+1} - V^{\pi_k}_{h+1}\right) \right](s^k_h, a^k_h) - \left(V^{k_{b_k}}_{h+1}(s^k_{h+1}) - V^{\pi_k}_{h+1}(s^k_{h+1})\right) \right\}
\]
\[
    + \sum_{h=1}^H \min\left\{H, 2\beta\bignorm{\phi(s^k_h, a^k_h)}_{\wt{\Lambda}^{-1}_{b_k,h}} \right\}
\]
We can notice that $\left[\mathbb{P}_h\left(V^{k_{b_k}}_{h+1} - V^{\pi_k}_{h+1}\right) \right](s^k_h, a^k_h) - \left(V^{k_{b_k}}_{h+1}(s^k_{h+1}) - V^{\pi_k}_{h+1}(s^k_{h+1})\right)$ is a MDS with respect to the filtration $\mathcal{G}_{h,k}$ (generated by the random noises we are injecting and the trajectories observed so far), where each term is bounded in absolute value by at most $2H$, hence by Azuma's inequality:
\[
    \sum_{k=1}^K \sum_{h=1}^H \left\{\left[\mathbb{P}_h\left(V^{k_{b_k}}_{h+1} - V^{\pi_k}_{h+1} \right)\right](s^k_h, a^k_h) - \left(V^{k_{b_k}}_{h+1} - V^{\pi_k}_{h+1}\right)(s^k_{h+1})\right\} \leq 2\sqrt{2KH^3\log\left(\frac{3}{p}\right)}
\]
with probability at least $1-\frac{p}{3}$. Then, $\mathbb{P}(\mathcal{B}) \geq 1-\frac{p}{3}$.

It follows that under the event $\mathcal{C} \cap \mathcal{E} \cap \mathcal{B}$, we have 
\[
    R(K) \leq 2\sqrt{2KH^3\log(3/p)} + \sum_{k=1}^K \sum_{h=1}^H \min\left\{H, 2\beta\bignorm{\phi(s^k_h, a^k_h)}_{\wt{\Lambda}^{-1}_{b_k,h}}\right\} 
\]
\end{proof}
\begin{lemma}[Adapted from Lem.4.3 from ~\citet{wang2021adaptivity}]\label{app:lem.7.regret}
In the notation of Alg.~\ref{alg:JDP-LowRank-RareSwitch.correct}, we have that:
\[
    \sum_{k=1}^K \sum_{h=1}^H \beta\bignorm{\phi(s^k_h, a^k_h)}_{\left(\Lambda_{k,h} + \left(c_K+ 2\Upsilon^J_{\frac{p}{6KH}}\right)I_{d\times d}\right)^{-1}} \leq \beta H\sqrt{2dK\log\left(\frac{K}{\lambda + c_K} + 1\right)}
\]
\end{lemma}
\begin{proof}
The proof is the same as the one of Lem.4.3 from ~\citet{wang2021adaptivity} with the only difference that the smallest eigenvalue of $\Lambda_{k,h} + \left(c_K+ 2\Upsilon^J_{\frac{p}{6KH}}\right)I_{d\times d}$ is bounded from below by $\lambda + c_K$. 
\end{proof}
\begin{lemma}[Adapted from Lem.4.4 from ~\citet{wang2021adaptivity}]\label{app:lem.8.regret}
Define the set $\underline{\mathcal{C}}$ as follows:
\[
    \underline{\mathcal{C}} = \left\{(k,h): \frac{\bignorm{\phi(s^k_h, a^k_h)}_{\wt{\Lambda}^{-1}_{b_k,h}}}{\bignorm{\phi(s^k_h,a^k_h)}_{\left(\Lambda_{k,h} + \left(c_K+ 2\Upsilon^J_{\frac{p}{6KH}}\right)I_{d\times d}\right)^{-1}}} > 4 \right\}
\]
Under the event $\mathcal{C} \cap \mathcal{E}$,
\[
\abs{\underline{\mathcal{C}}} \leq \Big\lfloor dH\cdot\Big\lceil\frac{K}{B}\Big\rceil\cdot \frac{\log\left(\frac{K}{d\lambda} + 1\right)}{2\log 4 - 2} \Big\rfloor
\]
\end{lemma}
\begin{proof}
Let $\underline{\mathcal{C}}_h$ be the set of indices where $(k,h) \in \underline{\mathcal{C}}$. Then, we have $\abs{\underline{\mathcal{C}}} = \sum_{h=1}^H \abs{\underline{\mathcal{C}}_h}$. 

For each $k \in \underline{\mathcal{C}}_h$, we have that 
\[
    2\log 4 < 2\log\left(\frac{\bignorm{\phi(s^k_h, a^k_h)}_{\wt{\Lambda}^{-1}_{b_k,h}}}{\bignorm{\phi(s^k_h,a^k_h)}_{\left(\Lambda_{k,h} + \left(c_K+ 2\Upsilon^J_{\frac{p}{6KH}}\right)I_{d\times d}\right)^{-1}}}\right)
\]
Due to Clm.~\ref{app:clm.norm.and.determinant} and given that $\wt{\Lambda}^{-1}_{b_k,h} \succeq \left(\Lambda_{k,h} + \left(c_K+ 2\Upsilon^J_{\frac{p}{6KH}}\right)I_{d\times d}\right)^{-1}$, we obtain:
\[
    2\log 4 < 2\log\left(\sqrt{\frac{\det(\Lambda_{k,h}+\left(c_K + 2\Upsilon^J_{\frac{p}{6KH}}\right)I_{d\times d})}{\det(\Lambda_{k_{b_k}}+c_KI_{d\times d})}}\right)
\]
since $k_{b_k} \leq k < k_{b_k+1}$, we know that $\Lambda_{k_{b_k+1}} \succeq \Lambda_k$ and due to Clm.~\ref{app:clm.order.determinant}:
\begin{equation}\label{eq:1.regret.proof.jdp.lsvi.ucb}
    2\log 4 < \log\left(\frac{\det\left(\Lambda_{k_{b_k+1},h} + \left(c_K + 2\Upsilon^J_{\frac{p}{6KH}}\right)I_{d \times d}\right)}{\det\left(\Lambda_{k_{b_k}} + \left(c_K + 2\Upsilon^J_{\frac{p}{6KH}}\right)I_{d \times d}\right)}\cdot\frac{\det\left(\Lambda_{k_{b_k}} + \left(c_K + 2\Upsilon^J_{\frac{p}{6KH}}\right)I_{d \times d}\right)}{\det\left(\Lambda_{k_{b_k}} + c_KI_{d\times d}\right)}\right)
\end{equation}
Letting $\lambda_1, \lambda_2, \dots, \lambda_d \geq 0$ be the eigenvalues of $\Lambda_{k_{b_k}}$, we get that:
\begin{equation}\label{eq:2.regret.proof.jdp.lsvi.ucb}
    \log\left(\frac{\det\left(\Lambda_{k_{b_k}} + \left(c_K + 2\Upsilon^J_{\frac{p}{6KH}}\right)I_{d \times d}\right)}{\det\left(\Lambda_{k_{b_k}} + c_KI_{d\times d}\right)}\right) = \log\left(\frac{\prod_{i=1}^d \lambda_i + c_K + 2\Upsilon^J_{\frac{p}{6KH}}}{\prod_{i=1}^d \lambda_i + c_K}\right) \leq \log\left(\prod_{i=1}^d 1 + \frac{2}{d}\right) \leq 2
\end{equation}
where the first inequality is true since $c_K \geq d\Upsilon^J_{\frac{p}{6KH}}$.

Replacing Eq.~\ref{eq:2.regret.proof.jdp.lsvi.ucb} in Eq.~\ref{eq:1.regret.proof.jdp.lsvi.ucb}, it follows that:
\begin{equation}\label{eq:3.regret.proof.jdp.lsvi.ucb}
    2\log 4 - 2 < \log\left(\frac{\det\left(\Lambda_{k_{b_k+1},h} + \left(c_K + 2\Upsilon^J_{\frac{p}{6KH}}\right)I_{d \times d}\right)}{\det\left(\Lambda_{k_{b_k},h} + \left(c_K + 2\Upsilon^J_{\frac{p}{6KH}}\right)I_{d \times d}\right)}\right)
\end{equation}
Let $\hat{\mathcal{C}}_h$ be the following set:
\[
    \hat{\mathcal{C}}_h = \left\{0 \leq b \leq N_K-1: \log\left(\frac{\det\left(\Lambda_{k_{b+1},h} + \left(c_K + 2\Upsilon^J_{\frac{p}{6KH}}\right)I_{d \times d}\right)}{\det\left(\Lambda_{k_b,h} + \left(c_K + 2\Upsilon^J_{\frac{p}{6KH}}\right)I_{d \times d}\right)}\right) > 2\log 4 - 2\right\}
\]
Eq.~\ref{eq:3.regret.proof.jdp.lsvi.ucb} ensures that $\abs{\underline{\mathcal{C}}_h} \leq \lceil \frac{K}{B}\rceil\cdot\abs{\hat{\mathcal{C}}_h}$. It just remains to bound $\abs{\hat{\mathcal{C}}_h}$:
\[
    (2\log 4 - 2)\abs{\hat{\mathcal{C}}_h} \leq \sum_{b \in \hat{\mathcal{C}}_h} \log\det\left(\Lambda_{k_{b+1},h} + \left(c_K + 2\Upsilon^J_{\frac{p}{6KH}}\right)I_{d\times d}\right) - \log\det\left(\Lambda_{k_b,h} + \left(c_K + 2\Upsilon^J_{\frac{p}{6KH}}\right)I_{d\times d}\right) 
\]
\[
    (2\log 4 - 2)\abs{\hat{\mathcal{C}}_h} \leq  \sum_{b=0}^{N_K-1} \log\det\left(\Lambda_{k_{b+1},h} + \left(c_K + 2\Upsilon^J_{\frac{p}{6KH}}\right)I_{d\times d}\right) - \log\det\left(\Lambda_{k_b,h} + \left(c_K + 2\Upsilon^J_{\frac{p}{6KH}}\right)I_{d\times d}\right) 
\]
By AM-GM inequality, we know that $\det(\Lambda_{K+1,h} + \left(c_K + 2\Upsilon^J_{\frac{p}{6KH}}\right)I_{d \times d}) = \det(\Lambda_{k_{N_K},h} + \left(c_K + 2\Upsilon^J_{\frac{p}{6KH}}\right)I_{d\times d}) \leq \left(\frac{\Tr\left(\Lambda_{k_{N_K},h} + \left(c_K + 2\Upsilon^J_{\frac{p}{6KH}}\right)I_{d\times d}\right)}{d}\right)^d = \left(\frac{\Tr\left(\Lambda_{K+1,h} + \left(c_K + 2\Upsilon^J_{\frac{p}{6KH}}\right)I_{d\times d}\right)}{d}\right)^d \leq \left(\frac{K}{d} + \lambda + c_K + 2\Upsilon^J_{\frac{p}{6KH}}\right)^d$, hence:
\[
    (2\log 4 - 2)\abs{\hat{\mathcal{C}}_h} \leq d\log\left(\frac{K}{d} + \lambda + c_K + 2\Upsilon^J_{\frac{p}{6KH}}\right) - d\log\left(\lambda + c_K + 2\Upsilon^J_{\frac{p}{6KH}}\right)
\]
\[
    (2\log 4 - 2)\abs{\hat{\mathcal{C}}_h} \leq d\log\left( \frac{K}{d\lambda} + 1\right)
\]
\[
    \abs{\hat{\mathcal{C}}_h} \leq \frac{d\log\left( \frac{K}{d\lambda} + 1\right)}{2\log 4 - 2}
\]
Hence,
\[
    \abs{\underline{\mathcal{C}}} = \sum_{h=1}^H \abs{\underline{\mathcal{C}}_h} \leq \sum_{h=1}^H\Big\lceil\frac{K}{B}\Big\rceil\cdot\abs{\hat{\mathcal{C}}_h} \leq \frac{dH\log\left(\frac{K}{d\lambda} + 1\right)}{2\log 4 - 2}\Big\lceil\frac{K}{B}\Big\rceil
\]
\[
   \abs{\underline{\mathcal{C}}} \leq \Big\lfloor dH\cdot\Big\lceil \frac{K}{B} \Big\rceil\cdot \frac{\log\left(\frac{K}{d\lambda} + 1\right)}{2\log 4 - 2} \Big\rfloor
\]
\end{proof}
We finish the regret analysis as in Thm.4.1 from ~\citet{wang2021adaptivity}. Under the event $\mathcal{C} \cap \mathcal{E} \cap \mathcal{B}$, we know that Lem.~\ref{app:lem.6.regret} states that 
\[
    R(K) \leq 2\sqrt{2KH^3\log(3/p)} + \sum_{k=1}^K \sum_{h=1}^H \min\left\{H, 2\beta\bignorm{\phi(s^k_h, a^k_h)}_{\wt{\Lambda}^{-1}_{b_k,h}}\right\} 
\]
Then,
\[
    R(K) \leq 2\sqrt{2KH^3\log(3/p)} + \sum_{(k,h) \in \underline{\mathcal{C}}} \min\left\{H, 2\beta\bignorm{\phi(s^k_h, a^k_h)}_{\wt{\Lambda}^{-1}_{b_k,h}}\right\} +  \sum_{(k,h) \notin \underline{\mathcal{C}}} \min\left\{H, 2\beta\bignorm{\phi(s^k_h, a^k_h)}_{\wt{\Lambda}^{-1}_{b_k,h}}\right\}
\]
By Lem.~\ref{app:lem.8.regret}:
\[
    R(K) \leq 2\sqrt{2KH^3\log(3/p)} + H\abs{\underline{\mathcal{C}}} + 8\beta\bignorm{\phi(s^k_h,a^k_h)}_{\left(\Lambda_{k,h} + \left(c_K + 2\Upsilon^J_{\frac{p}{6KH}}\right)I_{d \times d}\right)^{-1}}
\]
\[
    R(K) \leq 2\sqrt{2KH^3\log(3/p)} + H\Big\lfloor dH\cdot\Big\lceil \frac{K}{B} \Big\rceil\cdot \frac{\log\left(\frac{K}{d\lambda} + 1\right)}{2\log 4 - 2} \Big\rfloor + 8\beta\bignorm{\phi(s^k_h,a^k_h)}_{\left(\Lambda_{k,h} + \left(c_K + 2\Upsilon^J_{\frac{p}{6KH}}\right)I_{d \times d}\right)^{-1}}
\]
Lem.~\ref{app:lem.7.regret} provides an upper bound for the bonus term on the RHS:
\begin{equation}\label{eq:app.regret.jdp.lsvi.ucb}
    R(K) \leq  2\sqrt{2KH^3\log(3/p)} + H\Big\lfloor dH\cdot\Big\lceil \frac{K}{B} \Big\rceil\cdot \frac{\log\left(\frac{K}{d\lambda} + 1\right)}{2\log 4 - 2} \Big\rfloor + 8\beta H\sqrt{2dK\log\left(\frac{K}{\lambda + c_K} + 1\right)}
\end{equation}
Finally, $B = \Big\lceil \frac{(K\epsilon)^{\frac{2}{5}}}{d^{\frac{3}{5}}H^{\frac{1}{5}}}\Big\rceil$, implies $\beta = \wt{O}\left(dH + \frac{d^{\frac{11}{10}}H^{\frac{6}{5}}K^{\frac{1}{10}}}{\epsilon^{\frac{2}{5}}}\right)$ and recalling that $R(K) \leq KH$, we conclude:
\[
    R(K) = \wt{O}\left(\min\left\{d^{\frac{3}{2}}H^{2}\sqrt{K} + \frac{d^{\frac{8}{5}}H^{\frac{11}{5}}K^{\frac{3}{5}}}{\epsilon^{\frac{2}{5}}}, KH\right\}\right)
\]
under the event $\mathcal{C} \cap \mathcal{E} \cap \mathcal{B}$, which holds with probability at least $1-p$ since $\mathbb{P}\left(B\right), \mathbb{P}\left(\mathcal{E}\right),\mathbb{P}\left(\mathcal{C}|\mathcal{E}\right) \geq 1 - p/3$. The  $\wt{O}\left(\cdot\right)$ notation hides $polylog\left( \frac{1}{p}, \frac{1}{\delta}, d, H, K\right)$ factors.
\begin{remark}[Non-batching algorithm regret bound]\label{app:rem.nonbatched.lsvi.ucb} It is easy to remark that the non-batching algorithm can be easily obtained by setting $B = K$ in Alg.~\ref{alg:JDP-LowRank-RareSwitch.correct}. In Eq.~\ref{eq:app.regret.jdp.lsvi.ucb}, the second term would be independent of $K$ while the third one is $\wt{O}\left(d^{\frac{1}{2}}HK^{\frac{1}{2}}\beta\right) = \wt{O}\left(dH + \frac{d^{\frac{3}{4}}H^{\frac{5}{4}}K^{\frac{1}{4}}}{\sqrt{\epsilon}}\right)$. Thus, making the whole regret of the order of $\wt{O}\left(d^{\frac{3}{2}}H^2\sqrt{K} + d^{\frac{5}{4}}H^{\frac{9}{4}}K^{\frac{3}{4}}\cdot\frac{1}{\sqrt{\epsilon}}\right)$, which is worse than our regret (dismissing $d$ and $H$ factors) in the regime of interest $K\epsilon \geq 1$.

\end{remark}
\begin{remark}[Dynamic batching algorithm]\label{app:rem.dynamic.batch.schedule}
Alg.~\ref{alg:JDP-LowRank-RareSwitch.correct} is inspired in LSVI-UCB-Batch from ~\citet{wang2021adaptivity} and makes use of a static batching scheme. Ideally, we would have preferred to adapt LSVI-UCB-RareSwtich from the same article which uses only a logarithmic number of batches $O\left(dH\log K\right)$. We successfully managed to bound the regret of this adapted algorithm by $\wt{O}\left(\sqrt{K/\epsilon}\right)$, but the privacy analysis fails due to the fact that $k_0$-neighboring sequences of users $\mathfrak{U}^K$ and $\mathfrak{U}'^K$ may have completely different batches schedules since they are built depending on a dynamic determinant rule. In our notation, we would not be able to conclude Eqs.~\ref{eq:4.privacy.jdp.lsvi.ucb} and ~\ref{eq:8.privacy.jdp.lsvi.ucb} since the matrices/vectors released $\wt{\Lambda}_{b,h}$ and $\wt{u}_{b,h}$ will not correspond to the same protected quantities $\Lambda_{k_b,h}$, $u_{k_b,h}$ since $k_b$ is not anymore necessarily the same for both sequences $\mathfrak{U}^K$, $\mathfrak{U}'^K$.
\end{remark}
\subsection{Pure joint DP in the linear setting}\label{app:pureJDP.linear.setting}
To ensure pure joint DP, it is necessary to replace Gaussian noises by Laplace noises (see App.~\ref{app:laplace.mech}) and increase the variances $\sigma_\Lambda$ and $\sigma_u$ so as to let us conclude the same results using simple composition only instead of advanced compositions. In the notation of Alg.~\ref{alg:JDP-LowRank-RareSwitch.correct}:
\[
    \sigma_\Lambda = \wt{O}\left(\frac{1}{\epsilon}\cdot BH\right)
\]
\[
    \sigma_u = \wt{O}\left(\frac{1}{\epsilon}\cdot H^2B\right)
\]
Also, $\Upsilon^J_{\frac{p}{6KH}}$ and $C^J_{\frac{p}{6KH}}$ will become $\Upsilon^J_{\frac{p}{6KHd}} = \sigma_\Lambda B_0\left(2d + 2\sqrt{d\log\left(\frac{6KHd}{p}\right)} + \log\left(\frac{6KHd}{p}\right)\right)$ and $C^J_{\frac{p}{6KHd}} = \sigma_u\sqrt{d}\log\left(\frac{6KHd}{p}\right)$ due that we will make use of Clms.~\ref{app:clm.bound.laplace.eigen} and ~\ref{app:clm.bound.laplace.vector} instead of Clms.~\ref{app:clm.bounded.eigen} and ~\ref{app:clm.bound.gauss.vector}. Following the same regret analysis presented in App.~\ref{app:regret_JDP-LSVI-UCB-RarelySwitch}, we can also obtain Eq.~\ref{eq:app.regret.jdp.lsvi.ucb}. However, this time $\beta = \wt{O}\left(dH + H\sqrt{dc_K}\right) = \wt{O}\left(dH + \frac{(dH)^{\frac{3}{2}}\sqrt{B}}{\sqrt{\epsilon}}\right)$. Making the optimal choice of $B = \Big\lceil\frac{(K\epsilon)^{\frac{1}{3}}}{d^{\frac{2}{3}}H^{\frac{1}{3}}}\Big\rceil$, which yields a regret 
\[
    R(K) = \wt{O}\left(\min\left\{d^{\frac{3}{2}}H^2\sqrt{K} + \frac{d^{\frac{5}{3}}H^{\frac{7}{3}}K^{\frac{2}{3}}}{\epsilon^{\frac{1}{3}}}, KH\right\}\right)
\]
Notice that this result is slightly worse than the one for approximate joint DP. The reason is that since Alg.~\ref{alg:JDP-LowRank-RareSwitch.correct} relies on a batch schedule and we are applying a simple composition instead of advanced composition, we get a strictly polynomially bigger number of batches in this case with respect to the previous one, hence increasing the regret by a polynomial factor of $K$ and $\epsilon$, namely a $(K\epsilon)^{\frac{1}{15}}$ gap between pure and approximate joint DP.

\section{Differential Privacy Tools}

\subsection{Post-processing property ~\citep{dwork2014algorithmic}}\label{app:post.processing}
Let $\mathcal{M}: \mathcal{D} \to \mathcal{R}$ be a $(\epsilon,\delta)$-DP mechanism. Let $f: \mathcal{R} \to \mathcal{R}'$ be an arbitrary randomized mapping. Then, $f \circ \mathcal{M}: \mathcal{D} \to \mathcal{R}'$ is $(\epsilon,\delta)$-DP.

\subsection{Laplace Mechanism ~\citep{dwork2014algorithmic}}\label{app:laplace.mech}
For any function $f: \mathcal{C} \to \mathbb{R}^k$, let $\Delta_1(f) = \min_{x,y \in \mathcal{C}} \bignorm{f(x)-f(y)}_1$ be the $L_1$ sensitivity of $f$. The Laplace mechanism is defined as:
\[
    M_L(x,f(\cdot),\epsilon) = f(x) + (Y_1,Y_2,\dots,Y_k)
\]
where $Y_i$ are i.i.d random variables drawn from Lap$\left(\frac{\Delta_1(f)}{\epsilon}\right)$.\\
Then, $M_L(\cdot, f, \epsilon)$ is $(\epsilon,0)$-LDP.

\subsection{Gaussian Mechanism ~\citep{dwork2014algorithmic}}\label{app:gaussian.mech}
For any function $f: \mathcal{C} \to \mathbb{R}^k$, let $\Delta_2(f) = \min_{x,y \in \mathcal{C}} \bignorm{f(x) - f(y)}_2$ be the $L_2$ sensitivity of $f$. The Gaussian mechanism is defined as:
\[
    M_G(x, f(\cdot), \epsilon,\delta) = f(x) + (Y_1, Y_2, \dots, Y_k)
\]
where $Y_i$ are i.i.d random variables drawn from $\mathcal{N}(0, \sigma^2)$, where $\sigma = \frac{\Delta_2(f)\cdot\sqrt{2\log\left(\frac{2}{\delta}\right)}}{\epsilon}$.\\
Then, for any $\epsilon \in (0,1)$, $M_G(\cdot, f, \epsilon,\delta)$ is $(\epsilon,\delta)$-LDP.

\subsection{Tree-based method ~\citep{shariff2018nips, dwork2010DPcontinual,chan2010continual-release}}\label{app:tree.method}
Let $T$ be a complete binary tree with its leaf nodes being $l_1, l_2, \dots, l_n$. Each internal node $x \in T$ stores the sum of all the leaf nodes in the subtree rooted at $x$. Each partial sum $\sum_{j = 1}^i l_j$ uses at most $m = \lceil \log_2(n) + 1\rceil$ nodes of $T$ to its computation. Hence, if each node preserves $(\epsilon_0,\delta_0)$-DP, by adanvced composition, the entire algorithm is $\left( m\epsilon^2_0 + \epsilon_0\sqrt{2m\log\left(\frac{1}{\delta'}\right)}, m\delta_0 + \delta'\right)$-DP. Alternatively, to ensure that the tree is $(\epsilon,\delta)$-DP, it suffices to set $\epsilon_0 = \frac{\epsilon}{\sqrt{8m\log\left(\frac{2}{\delta}\right)}}, \delta_0 = \frac{\delta}{2m}$ (with $\delta' = \frac{\delta}{2}$).
\subsection{Simple Composition ~\citep{dwork2014algorithmic, dwork2010Composition, kairouz2015Composition}}\label{app:simple.comp}
Let $\mathcal{M}_i: \mathcal{D} \to R_i$ be an $(\epsilon_i,\delta_i)$-DP mechanism $\forall i \in [K]$. Letting $\mathcal{M}: \mathcal{D} \to \prod_{i=1}^k R_i$ be the mechanism defined as $\mathcal{M}(x) = (\mathcal{M}_1(x), \dots, \mathcal{M}_k(x))$, then $\mathcal{M}$ is $\left(\sum_{i=1}^k \epsilon_i, \sum_{i=1}^k \delta_i\right)$-DP.
\subsection{Advanced composition ~\citep{dwork2014algorithmic, dwork2010Composition, kairouz2015Composition}}\label{app:advanced.comp}
For all $\epsilon,\delta,\delta' > 0$, the class of $(\epsilon,\delta)$-DP mechanisms satisfies $(\epsilon',k\delta+\delta')$-DP under $k$-fold composition for 
\[
    \epsilon' = \sqrt{2k\log\left(\frac{1}{\delta}\right)}\epsilon + k\epsilon(e^{\epsilon}-1)
\]
In particular, given target privacy parameters $0 < \epsilon < 1$ and $\delta > 0$, the mechanism $\mathcal{M} = (\mathcal{M}_1, \dots, \mathcal{M}_k)$ is $(\epsilon,\delta)$-DP, where $(\mathcal{M}_i)_i$ are (potentially) adaptive $(\epsilon',\delta')$-DP mechanisms and $\epsilon' = \frac{\epsilon}{\sqrt{8k\log\left(2/\delta\right)}}$ and $\delta' = \frac{\delta}{2k}$.

\subsection{Joint DP through DP, Billboard Lemma ~\citep{hsu2016Billboard}}\label{app:billboard.lemma}
Suppose $\mathcal{M}: \mathcal{D} \to \mathcal{R}$ is $(\epsilon,\delta)$-DP. Consider any set of functions $f_i:\mathcal{D}_i \times \mathcal{R} \to \mathcal{R}'$, where $\mathcal{D}_i$ is the portion of the database containing user $i$'s data. The composition $\left\{f_i(\Pi_i D, \mathcal{M}(D))\right\}$ is $(\epsilon,\delta)$-JDP, where $\Pi_i: \mathcal{D} \to \mathcal{D}_i$ is the projection of user $i$'s data.

\section{Auxiliary Results}
\begin{claim}[Consequence of concentration inequalities in ~\citep{tao2011claim}]\label{app:clm.bounded.eigen}Let $M \in \mathbb{R}^{d\times d}$ be a symmetric matrix where each of its entries $M_{i,j} = M_{j,i} \sim \mathcal{N}(0,1)$ for any $1 \leq i \leq j\leq d$. Then, for any $\alpha > 0$, $\mathbb{P}\left(\norm{M} \geq 4\sqrt{d} + 2\log\left( \frac{1}{\alpha}\right)\right) \leq \alpha$, where $\norm{M}$ is the operator norm of a matrix associated to the norm $\norm{\cdot}_2$.
\end{claim}

\begin{claim}[Corollary to Lemma 1, p. 1325 in ~\citep{laurent2005claim}]\label{app:clm.bound.gauss.vector} 
If $U \sim \chi^2(d)$ and $\alpha \in (0,1)$:
\[
    \mathbb{P}\left(U \geq d + 2\sqrt{d\log\left(\frac{1}{\alpha}\right)} + 2\log\left(\frac{1}{\alpha}\right)\right) \leq \alpha,
\]
\[
    \mathbb{P}\left(U \leq d - 2\sqrt{d\log\left(\frac{1}{\alpha}\right)} \right) \leq \alpha
\]
As a consequence of the first inequality, we also have that for any vector $v \in \mathbb{R}^d$ drawn from a $d$-dimensional gaussian distribution $\mathcal{N}(0, I_{d \times d})$. Then, $\mathbb{P}\left(\norm{v}_2 > \sqrt{d} + 2 \sqrt{\log\left(\frac{1}{\alpha}\right)}\right) \leq \alpha$.
\end{claim}

\begin{claim}\label{app:clm.bound.laplace.vector}
Let $v \in \mathbb{R}^d$ be a vector such that $(v_i)_i$ are i.i.d random variables drawn from $Lap(0,1)$. Then, for any $\alpha > 0$, $\mathbb{P}\left(\norm{v}_2 > \sqrt{d}\cdot\log\left(\frac{d}{\alpha}\right)\right) \leq \alpha$
\end{claim}
\begin{proof}
This is an straightforward application of the definition of Laplace distributions. Indeed, for each $i \in [d]$:
\[
    \mathbb{P}\left(|v_i| > \log\left(\frac{d}{\alpha}\right)\right) \leq \frac{\alpha}{d}
\] 
Hence, in the union event, holding with probability at least $1-\alpha$, $\norm{v}_2 \leq \sqrt{d}\log\left(\frac{d}{\alpha}\right)$.
\end{proof}
\begin{claim}\label{app:clm.bound.laplace.eigen}
Let $M = \frac{Z' + Z'^\intercal}{2}\in \mathbb{R}^{d \times d}$ be a symmetric matrix where each of the entries $Z'_{i,j} = Z'_{j,i} \sim Lap(0,1)$ for any $1 \leq i \leq j \leq d$. Then, for any $\alpha > 0$, $\mathbb{P}\left(\norm{M} \geq 2d + 2\sqrt{d\log\left(\frac{d}{\alpha}\right)} + \log\left(\frac{d}{\alpha}\right)\right) \leq \alpha$ 
\end{claim}
\begin{proof}
Consider the following events for each $i \in [d]$:
\[
    \mathcal{C}_i = \left\{\abs{z_{i,i}} + \sum_{\substack{{1 \leq j \leq d} \\ j \neq i}} \abs{z_{i,j}} + \abs{z_{j,i}} \leq \frac{1}{2}\cdot \left(4d-2 + 2\sqrt{(4d - 2)\log\left(\frac{d}{\alpha}\right)} + 2\log\left(\frac{d}{\alpha}\right)\right)\right\}
\]
Then, since $z_{i,i}$ and $z_{i,j}, z_{j,i}$ are i.i.d drawn from $Lap(0,1)$, we have that $\abs{z_{i,i}} + \sum_{\substack{{1 \leq j \leq d} \\ j \neq i}} \abs{z_{i,j}} + \abs{z_{j,i}} \sim \frac{1}{2}\cdot\chi^2(4d-2)$ and by Clm.~\ref{app:clm.bound.gauss.vector}, we deduce that $\mathbb{P}\left(\mathcal{C}_i\right) \geq 1 - \frac{\alpha}{d} \implies \mathbb{P}\left(\bigcap_{1 \leq i \leq d} \mathcal{C}_i\right) \geq 1 - \alpha$. 

Hence, under the event $\bigcap_{1 \leq i \leq d} \mathcal{C}_i$, let $\lambda$ be the largest eigenvalue (in absolute value) of the matrix $M$ and $v = (v_1, \dots, v_d) \neq 0$ its associated eigenvector. 
Also, let $r = \argmax_{1 \leq i \leq d}\left\{\abs{v_i}\right\}$. 

Given that $\lambda\cdot v = Mv \implies \lambda\cdot v_r = \left(\frac{z_{r,1} + z_{1,r}}{2}\right)v_1 + \left(\frac{z_{r,2} + z_{2,r}}{2}\right)v_2 + \dots + z_{r,r}v_r + \dots + \left(\frac{z_{r,d} + z_{d,r}}{2}\right)v_d$.\\
Therefore, by the triangle inequality, $\abs{\lambda}\abs{v_r} \leq \abs{v_r}\left(\abs{\frac{z_{1,d} + z_{d,1}}{2}} + \dots + \abs{z_{r,r}} + \dots + \abs{\frac{z_{r,d} + z_{d,r}}{2}}\right)$. It implies that
\[
    \abs{\lambda} \leq \abs{z_{r,r}} + \sum_{\substack{{1 \leq j \leq d} \\ j \neq r}} \abs{z_{r,j}} + \abs{z_{j,r}} \leq \frac{1}{2}\left(4d-2 + 2\sqrt{(4d - 2)\log\left(\frac{d}{\alpha}\right)} + 2\log\left(\frac{d}{\alpha}\right)\right) < 2d + 2\sqrt{d\log\left(\frac{d}{\alpha}\right)} + \log\left(\frac{d}{\alpha}\right)
\]
\[
    \norm{M}_2 = \abs{\lambda} < 2d + 2\sqrt{d\log\left(\frac{d}{\alpha}\right)} + \log\left(\frac{d}{\alpha}\right)
\]
\end{proof}

\begin{claim}
[Theorem 7.8 in ~\citep{zhang2011claim}]\label{app:clm.order.PSD} For two positive definite matrices (PSD) $A, B \in \mathbb{R}^{d\times d}$, we write $A \succeq B$ to denote that $A - B$ is PSD. \\
Then, if $A \succeq B \succeq 0$:
\begin{itemize}
    \item rank$(A) \geq$ rank$(B)$ 
    \item $\det(A) \geq \det(B)$
    \item $B^{-1} \succeq A^{-1}$ if $A$ and $B$ are non-singular.
\end{itemize}
\end{claim}

\begin{claim}[Lemma 12, ~\citep{abbasi-yadkori2011}]\label{app:clm.norm.and.determinant}
Suppose $\boldsymbol{A}, \boldsymbol{B} \in \mathbb{R}^{d \times d}$ are two PSD matrices such that $\boldsymbol{A} \succeq \boldsymbol{B}$. Then, for any $\boldsymbol{x} \in \mathbb{R}^d$, we have $\bignorm{\boldsymbol{x}}_{\boldsymbol{A}} \leq \bignorm{\boldsymbol{x}}_{\boldsymbol{B}}\cdot \sqrt{\frac{\det\left( \boldsymbol{A}\right)}{\det\left(\boldsymbol{B}\right)}}$
\end{claim}

\begin{claim}\label{app:clm.order.determinant}
Suppose $\boldsymbol{A}, \boldsymbol{B} \in \mathbb{R}^{d \times d}$ are two PSD matrices such that $\boldsymbol{A} \succeq \boldsymbol{B}$. Then, $\det\left(\boldsymbol{A}\right) \geq \det\left(\boldsymbol{B}\right)$.
\end{claim}
\begin{proof}
Since $\boldsymbol{A} \succeq \boldsymbol{B}$, we know by Clm.~\ref{app:clm.order.PSD} that $\boldsymbol{B}^{-1} \succeq \boldsymbol{A}^{-1} \implies x^\intercal\boldsymbol{B}^{-1}x \geq x^\intercal\boldsymbol{A}^{-1}x, \forall x \in \mathbb{R}^d$.
Integrating both sides w.r.t $\mathbb{R}^d$, we have:
\[
    \int_{\mathbb{R}^d} \exp\left(-\frac{1}{2}x^\intercal\boldsymbol{B}^{-1}x\right)dx \leq \int_{\mathbb{R}^d} \exp\left(-\frac{1}{2}x^\intercal\boldsymbol{A}^{-1}x\right)dx
\]
\[
    \sqrt{(2\pi)^d\det(\boldsymbol{B})} \leq \sqrt{(2\pi)^d\det(\boldsymbol{A})} \implies \det(\boldsymbol{B}) \leq \det(\boldsymbol{A})
\]
\end{proof}

\end{appendix}
\end{document}